\documentclass[conference,onecolumn]{IEEEtran}
\usepackage{amssymb}
\usepackage{bm}
\usepackage{amsmath}
\usepackage{amsthm}
\usepackage[]{algorithm2e}
\newtheorem{theorem}{Theorem}
\newtheorem{lemma}{Lemma}
\newtheorem{definition}{Definition}
\newtheorem{proposition}{Proposition}
\usepackage{subfig}
\usepackage[pdftex]{graphicx}
\DeclareGraphicsExtensions{.pdf,.jpeg,.png}

\begin{document}
\title{Rational Kernels: A survey\\{\large Literature Survey Report}}

\author{\IEEEauthorblockN{Abhishek Ghose}
\IEEEauthorblockA{Department of Computer Science and Engineering, IIT Madras\\E-mail: abhishek.ghose.82@gmail.com}
}
%\IEEEauthorblockA{Roll No CS15D004, Program: PhD\\Advisor: Dr. Balaraman Ravindran\\Department of Computer Science and Engineering, IIT Madras}
%}

\maketitle

%\begin{abstract}
%The abstract goes here.
%\end{abstract}

\IEEEpeerreviewmaketitle

\section{Introduction}

%Performing machine learning with sequences often can be cumbersome since many algorithms admit input data only as fixed-length feature vectors. 

Many kinds of data are naturally amenable to being treated as sequences. An example is text data, where a text may be seen as a sequence of words. Another example is clickstream data, where a data instance is a sequence of clicks made by a visitor to a website. This is also common for data originating in the domains of speech processing and computational biology. Using such data with statistical learning techniques can often prove to be cumbersome since most of them only allow fixed-length feature vectors as input. In casting the data to fixed-length feature vectors to suit these techniques, we lose the convenience, and possibly information, a good sequence-based representation can offer.

The framework of \emph{rational kernels} partly addresses this problem by providing an elegant representation for sequences, for algorithms that use \emph{kernel functions}.% (ex. Support Vector Machines (SVM)).

In this report, we take a comprehensive look at this framework, its various extensions and applications. We start with an overview of the `core' ideas in Section \ref{sec:ratk_theory_algo}, where we look at the characterization of rational kernels, their properties and efficient ways to compute them. 

%The remaining sections look at extensions or applications. 

In section \ref{sec:relation_to_graph} we discuss the equivalence of \emph{graph kernels} and rational kernels under certain circumstances. %Sections \ref{sec:ratk_classify_pt} and \ref{sec:identify_metabolic} look at applications. 
Section \ref{sec:ratk_classify_pt} notes an interesting application - that of being able to classify sequences as to belonging to a formal language, using \emph{Support Vector Machines (SVMs)} with a specific rational kernel. The theory around this rather surprising relationship between the membership test for a formal language and rational kernels is discussed in considerable detail. In section \ref{sec:identify_metabolic}, we discuss yet another application, this time in the domain of computational biology: that of identifying metabolic pathways.

In a world with ever-increasing quantities of available data, techniques to perform large-scale training of algorithms are immensely valuable; we focus on this objective in Section \ref{sec:large_scale}. Rational kernels represent a \emph{family} of kernels (as we shall see in Section \ref{sec:ratk_theory_algo}), and thus,  \emph{learning} an appropriate rational kernel instead of picking one, suggests a convenient way to use them; we explore this idea in our concluding section, Section \ref{sec:learning_kernels}. 

Rational kernels are not as popular as the many other learning techniques in use today; however, we hope that this summary effectively shows that not only is their theory well-developed, but also that various practical aspects have been carefully studied over time.

\section{Rational Kernels: Theory and Algorithms}
\label{sec:ratk_theory_algo}

%Many types of data are naturally amenable to be treated as sequences. Examples may be found in the domain of text mining, speech processing, computational biology and lately, clickstream analysis. Using such data with most statistical learning techniques can often be tedious since they only allow input as fixed-length feature vectors. This incompatibility is typically resolved by casting the data into fixed-length vectors - losing the convenience, and possibly information, a good sequence-based representation can offer.

%\emph{Rational kernels} address this problem for the family of learning techniques that can use $\emph{kernels}$.

%The fundamental idea behind \emph{rational kernels} is having 
%A good way to think about the fundamental idea behind \emph{rational kernels} is to 
%Let's start by  briefly looking at regular expressions.

The strength of rational kernels lies in the fact that they rely on a powerful, compact and general representation - \emph{weighted transducers}, a richer version of the more popular automata - that makes it simple to reason about sequences for a variety of problems. %These act as an underlying structure common to a variety of problems, and help guide observations and inferences. Also, 
Much like any other effective abstraction, they eliminate the need to study sequence data individually in a lot of cases; they ``shift the burden'' to the abstraction: if we can show the data is representable by certain transducers, results applicable to the abstraction are readily applied to the data.

%Simply put, the fundamental idea behind rational kernels is similar to that of \emph{regular expressions} (\emph{regexps}, in short) - we have a powerful and general tool to represent and talk about sequences. In the case of regexps, this is the automata-based representation; in the case of rational kernels, its \emph{transducers} - a richer version of the automata. Both are effective abstractions and eliminate the need to deal with a lot of sequence data individually - if we can show them to be regexps, or in this case, representable by certain transducers, results applicable to these abstractions can be readily applied to the data at hand. Rational kernels allow for an additional flexibility by letting us choose a system of ``numbers'' and basic operations on them. This makes the framework more general.

Rational kernels were introduced in \cite{Cortes02rationalkernels}, and the theory was further developed in \cite{Cortes04rationalkernels:}, \cite{Cortes2003}, \cite{1198859}, \cite{Cortes03weightedautomata}. In this section, we borrow our arguments and notation primarily from \cite{Cortes04rationalkernels:}, since it provides a comprehensive overview of the theory and algorithms. We begin by presenting definitions and notations that are needed to understand rational kernels.

\begin{definition}
A system $(K,\odot,e)$ is a monoid if it is closed under $\odot: a\odot b \in \mathbb{K}$ for all $a,b \in \mathbb{K}$; $\odot$ is associative: $(a \odot b) \odot c = a \odot (b \odot c)$ for all $a,b,c \in \mathbb{K}$; and $e$ is an identity for $\odot: a \odot e = e \odot a = a$, for all $a \in \mathbb{K}$. When additionally $\odot$ is commutative: $a \odot b = b\odot a$ for all $a,b \in \mathbb{K}$, then $(\mathbb{K},\odot,e)$ is said to be a commutative monoid.
\end{definition}

\begin{definition} A system $(\mathbb{K},\oplus,\otimes,\overline{0},\overline{1})$ is a semiring if: $(\mathbb{K},\oplus,\overline{0})$ is a commutative monoid with identity element $\overline{0}$; $(\mathbb{K},\otimes,\overline{1})$ is a monoid with identity element $\overline{1}$; $\otimes$ distributes over $\oplus$; and $\overline{0}$ is an annihilator for $\otimes$: for all $a \in \mathbb{K}, a \otimes \overline{0} = \overline{0} \otimes a = \overline{0}$.
\end{definition}

\begin{table}[!t]
\centering
\begin{tabular}{l|c|c|c|c|c}
%\hline
Semiring &Field&$\oplus$&$\otimes$&$\overline{0}$&$\overline{1}$\\
\hline \hline
Boolean & $\{0,1\}$ & $\vee$ & $\wedge$ & $0$ & $1$\\
\hline
Probability & $\mathbb{R}_+$ & $+$ & $\times$ & $0$ & $1$\\
\hline
Log & $\mathbb{R} \cup \{-\infty, +\infty\}$  & $\oplus_\text{log}$ & $+$ & $+\infty$ & $0$\\
\hline
Tropical & $\mathbb{R} \cup \{-\infty, +\infty\}$ & $\min$ & $+$ & $+\infty$ & $0$\\
\hline
\end{tabular}
\caption{Examples of semirings. $\oplus_\text{log}$ is defined by $x\oplus_\text{log} y = -\log(e^{-x} + e^{-y})$.}
\label{tab:semirings}
\end{table}

Table \ref{tab:semirings} lists some semirings. We would mostly use the probability and tropical semirings. We are now ready to define the \emph{weighted finite-state transducer} - the representation underlying rational kernels.

\begin{definition}
A weighted finite-state transducer $T$ over a semiring $\mathbb{K}$ is an 8-tuple $T = (\Sigma,\Delta,Q,I,F,E,\lambda,\rho)$ where $\Sigma$ is the finite input alphabet of the transducer; $\Delta$ is the finite output alphabet; Q is a finite set of states; $I \subseteq Q$ the set of initial states; $F \subseteq Q$ the set of final states; $E \subseteq Q \times (\Sigma \cup \{\epsilon\})\times(\Delta \cup \{\epsilon\}) \times \mathbb{K} \times Q$ a finite set of transitions; $\lambda: I \rightarrow \mathbb{K}$ the initial weight function; and $\rho: F \rightarrow \mathbb{K}$ the final weight function mapping $F$ to $\mathbb{K}$.

Weighted automata can be formally defined in a similar way by simply omitting the input or output labels.
\end{definition}

\begin{figure*}[!t]
\centering
\includegraphics[width=2.0 in]{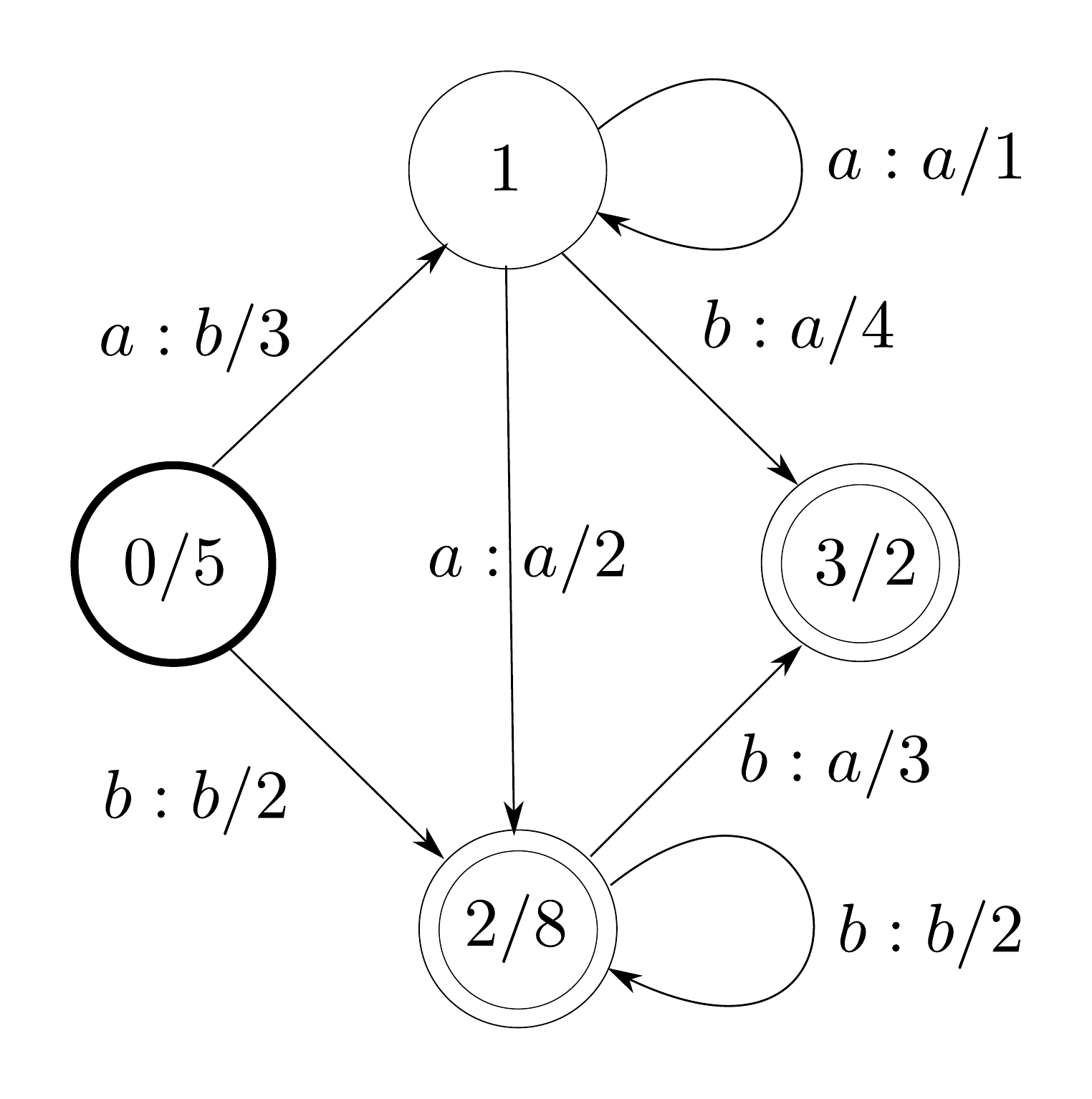}
\caption{Weighted finite-state transducer defined over the probability semiring.} 
\label{fig:awt}
\end{figure*}

Figure \ref{fig:awt} shows an example of such a weighted transducer, defined over the probability semiring. The states are denoted by nodes. The initial state, $0$, is shown in bold. The final states, $2$ and $3$, are shown in two concentric circles. Every transition has a label in the format \emph{input symbol:output symbol/weight}. The final state is labeled in the format \emph{state number/weight}.

Some relevant terminology:
\begin{itemize}
\item Given a transition $e \in E$, we denote by $p[e]$ its origin or previous state, $n[e]$ its
destination state or next state, and $w[e]$ its weight. A path $\pi = e_1 e_2 ... e_k$ is an element
of $E^*$, with consecutive transitions, such that $n[e_{i - 1}] = p[e_i], i = 2, . . . , k$. 

We extend $n$ and $p$ to paths by setting: $n[\pi] = n[e_k]$ and $p[\pi] = p[e_1]$.

\item For a path $\pi$, its \emph{input label} and \emph{output label} are the concatenation of the input symbols and the output symbols respectively on the path, e.g., for path $\pi_1$: (0-1, 1-3) in the transducer shown in Figure \ref{fig:awt}, the input label is $ab$ and the output label is $ba$.

\item We denote by $P(q, q')$ the set of paths
from $q$ to $q'$ and by $P(q, x, y, q')$ the set of paths from $q$ to $q'$ with input label $x \in \Sigma^*$
and output label $y \in \Delta^*$. These definitions can be extended to subsets $R,R' \subset Q$, by:
$P(R, x, y, R') = \cup_{q \in R, q' \in R'} P(q, x, y, q')$.

\item An \emph{accepting path} is a path starting at an initial state and ending at a final state. $\pi_1$ is an accepting path, while $\pi_2$: (0-1, 1-1) is not. 

$P(I, x, y, F)$ is the set of all accepting paths with input label $x$ and output label $y$.

\item The \emph{weight of a path}, denoted  by $w[\pi]$, is the product of the weights of all transitions on the path. For $\pi= e_1 e_2 ... e_k$, $w[\pi]=w[e_1]\otimes w[e_2]\otimes ... \otimes  w[e_k]$. For our example path, $w(\pi_1$) = $3 \cdot 4  = 12$. 

\item A weighted transducer $T$ associates a weight $T(x,y)$ to a pair of sequences $(x,y) \in \Sigma^* \times \Delta^*$. This weight is obtained by summing the weight of all accepting paths in $T$, whose input label is $x$ and output label is $y$, multiplied by the weight of the final state. This can be seen as the mapping, $T: \Sigma^* \times \Delta^* \rightarrow \mathbb{R}$: 
\begin{equation*}
T(x,y)=\bigoplus\limits_{\pi \in P(I,x,y,F)} \lambda(p[\pi]) \otimes w[\pi] \otimes \rho(n[\pi])
\end{equation*}
We define $T(x,y)=\overline{0}$, if $P(I,x,y,F)=\phi$.\\

Here's how we calculate $T(aab,baa)$, for the transducer in Figure \ref{fig:awt}: \\
Paths:
\begin{enumerate}
\item $w((0-1, 1-1, 1-3)) = 12 $
\item $w((0-1, 1-2, 2-3)) = 18 $  %\vspace{2mm}
\end{enumerate}
\begin{equation*}
T(aab,baa)= 5 \cdot 12 \cdot 2+5 \cdot 18 \cdot 2= 300
\end{equation*}
\item A transducer $T$ is \emph{regulated} if $T(x,y)$ for any $x \in \Sigma^*, y \in \Delta^*$ is well-defined and in $\mathbb{K}$. In particular, when $T$ does not have any $\epsilon$-cycle, that is a cycle labeled with $\epsilon$ (both input and output labels), it is regulated. We will only consider regulated transducers here.
\item The inverse of a transducer $T$, denoted by $T^{-1}$ is the original transducer with the input and output symbols swapped on all transitions i.e. if $E$ and $E^{-1}$ are set of transitions for $T$ and $T^{-1}$ respectively, then $(p, a, b, w, q) \in E \iff (p, b, a, w, q) \in E^{-1}, \text{ where } p,q \in Q, a \in \Sigma, b \in \Delta, w \in \mathbb{K}$.
\end{itemize}
\vspace{2mm}
Regulated weighted transducers are closed under the \emph{rational} operations: $\oplus$-sum, $\otimes$-product and Kleene-closure which are defined as follows for all transducers $T_1$ and $T_2$ and $(x,y) \in \Sigma^* \times \Delta^*$:
\begin{enumerate}
\item $[\![T_1 \oplus T_2]\!](x,y) = [\![T_1]\!](x,y) \oplus [\![T_2]\!](x,y)$
\item $[\![T_1 \otimes T_2]\!](x,y) = \bigoplus\limits_{x=x_1x_2,y=y_1y_2} [\![T_1]\!](x_1,y_1) \otimes [\![T_2]\!](x_2,y_2)$
\item $[\![T^*]\!](x,y) = \bigoplus\limits_{n=0}^\infty T^n(x,y)$
\end{enumerate}
where $T^n$ stands for the $(n - 1)$-$\otimes$-product of $T$ with itself.

\emph{Composition} is a fundamental operation on weighted transducers that can be used in many applications to create complex weighted transducers from simpler ones. Let $T_1 = (\Sigma,\Delta,Q_1,I_1,F_1,E_1,\lambda_1,\rho_1)$ and $T_2 = (\Delta,\Omega,Q_2,I_2,F_2,E_2,\lambda_2,\rho_2)$ be two weighted transducers defined over a commutative semiring $\mathbb{K}$ such that $\Delta$, the output alphabet of $T_1$, coincides with the input alphabet of $T_2$. Then, the result of the composition of $T_1$ and $T_2$ is a weighted transducer $T_1 \circ T_2$ which, when it is regulated, is defined for all $x,y$ by:
\begin{equation}
\label{eqn:transducer_composition}
[\![T_1 \circ T_2]\!](x,y) = \bigoplus\limits_{z\in \Delta^*} [\![T_1]\!](x,z) \otimes[\![T_2]\!](z,y)
\end{equation}
The definition of composition extends naturally to weighted automata since a weighted automaton can be viewed as a weighted transducer with identical input and output labels for each transition. The corresponding transducer associates $[\![A]\!](x)$ to a pair $(x,x)$, and $0$ to all other pairs. Thus, the composition of a weighted automaton $A_1 = (\Delta,Q_1,I_1,F_1,E_1,\lambda_1,\rho_1)$ and a weighted transducer $T_2 = (\delta,\Omega,Q_2,I_2,F_2,E_2,\lambda_2,\rho_2)$ is simply defined for all $x,y \in \Delta^* \times \Omega^*$ by:
\begin{equation}
\label{eqn:automata_transducer_composition}
[\![A_1 \circ T_2]\!](x,y) = \bigoplus\limits_{x\in \Delta^*}[\![A_1]\!](x)\otimes[\![T_2]\!](x,y)
\end{equation}
when these sums are well-defined and in $\mathbb{K}$. \emph{Intersection} of two weighted automata is the special case of composition where both operands are weighted automata, or equivalently weighted transducers with identical input and output labels for each transition.

Although the expression for composition given in eqns (\ref{eqn:transducer_composition}), (\ref{eqn:automata_transducer_composition}), suffice for most discussions, we will see how to construct the transducer corresponding to a composition in Section \ref{sec:computing_kernel_values}.

We are now ready to define \emph{rational kernels}.

\begin{definition}
A kernel $K$ over $\Sigma^* \times \Delta^*$ is said to be rational if there exist a weighted transducer $T = (\Sigma,\Delta,Q,I,F,E,\lambda,\rho)$ over the semiring $K$ and a function $\Psi: K \rightarrow \mathbb{R}$ such that for all $x \in \Sigma^*$ and $y \in \Delta^*$: 
\begin{equation}
K(x,y) = \Psi([\![T]\!](x,y))
\end{equation}
$K$ is then said to be defined by the pair $(\Psi,T)$.
\end{definition}

Rational kernels can be naturally extended to kernels over weighted automata. Let $A$ be a weighted automaton defined over the semiring $\mathbb{K}$ and the alphabet $\Sigma$ and $B$ a weighted automaton defined over the semiring $\mathbb{K}$  and the alphabet $\Delta$, $K(A,B)$ is defined by:
\begin{equation}
\label{eqn:kernel_automata_transducer}
K(A,B) = \Psi\bigg( \bigoplus\limits_{ (x,y) \in \Sigma^* \times \Delta^*} [\![A]\!](x) \otimes [\![T]\!](x,y)\otimes[\![B]\!](y) \bigg)
\end{equation}

Rational kernels subsume a number of well known similarity measures since transducers happen to be an extremely versatile tool for representation. We list some of these transducers next. These also demonstrate the potentially broad scope of use of rational kernels.
\begin{enumerate}
\item Bigram counter: As the name suggests, this transducer counts the number of occurrences of a bigram. The count is an overlapping count and the bigram to be matched is denoted by the output sequence. For ex, with $x=abbb$, $T(x, bb)=2$.

\begin{figure*}[!t]
\centering
\includegraphics[width=2.5 in]{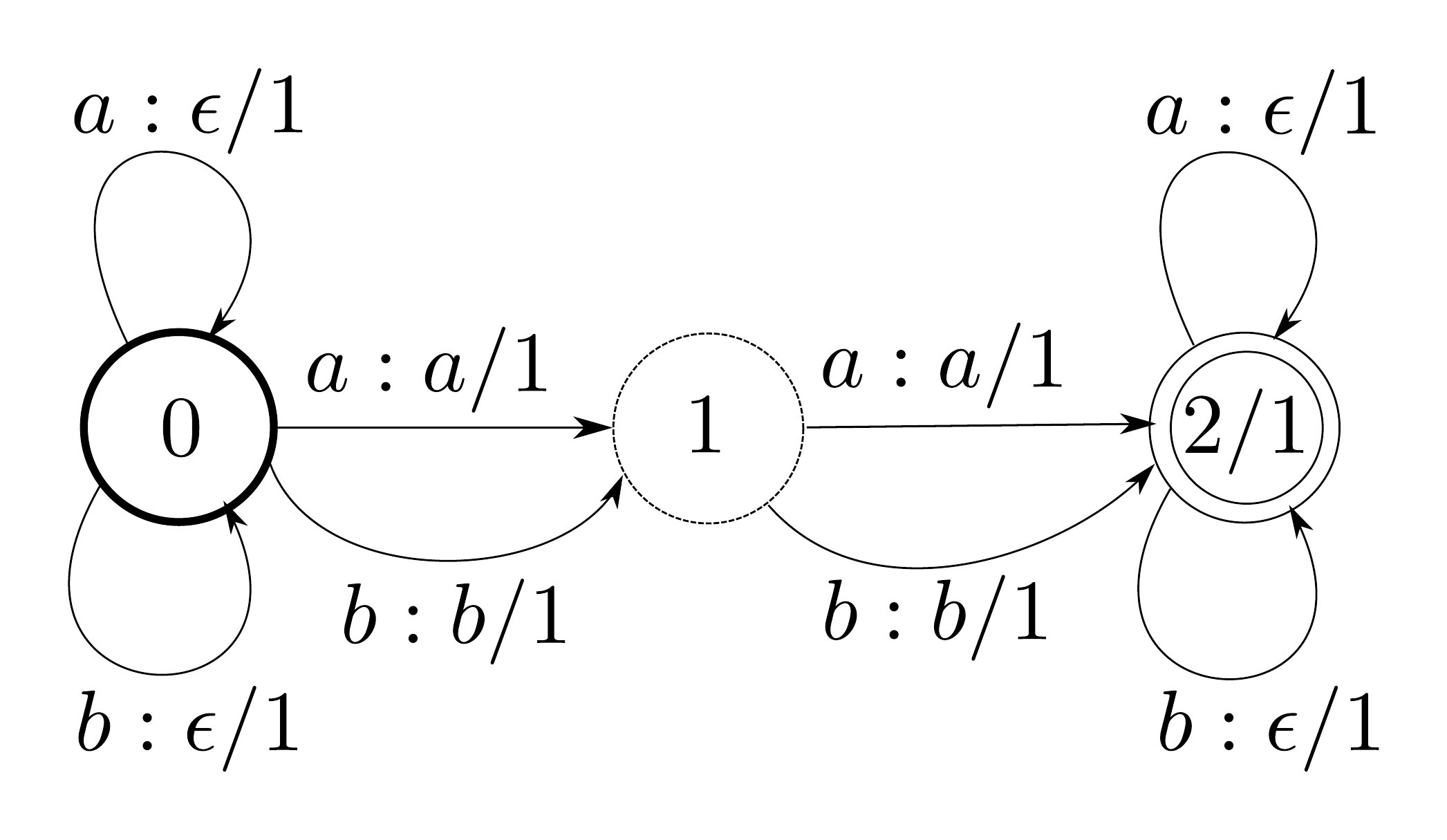}
\caption{Transducer that counts the number of bigrams from $\Sigma= \{a,b\}$.} 
\label{fig:bigram_counter}
\end{figure*}

Fig \ref{fig:bigram_counter} shows the transducer for the probability semiring. The key structure is that for an input sequence to end up in the final state, it has to pass through state $1$. In doing so, it must match the output sequence bigram (output symbols on transitions not entering/exiting state $1$ are $\epsilon$). Each such accepting path has a weight of $1$, and there are as many accepting paths as overlapping occurrences of the bigram in the input sequence.

Let $c(x,z)$ represent the overlapping count of occurrences of $z$ in $x$. Then, $T(x,z)=c(x,z)$. Consider the composition, $T \circ T^{-1}(x,y)=\sum_{|z|=2}c(x,z)c(y,z)$. Thus, $T \circ T^{-1}$ evaluates the similarity of two sequences based on co-occurring bigrams.

\item Gappy bigram transducer: The gappy bigram transducer also counts occurrences of bigrams, with the key difference that the bigrams need not be contiguous i.e. they may have ``gaps''. The more spread out a bigram is, higher the penalty associated with it. This penalty is denoted by $\lambda \in (0,1)$. For ex, for $x=abcd,\; T(x, ab) = 1, \text{ whereas } T(x, ac) = \lambda \text{ and } T(x,ad)=\lambda^2$ (this version differs slightly, not in a significant way, from the original discussed in \cite{Lodhi:2002:TCU:944790.944799}).

\begin{figure*}[!t]
\centering
\includegraphics[width=2.5 in]{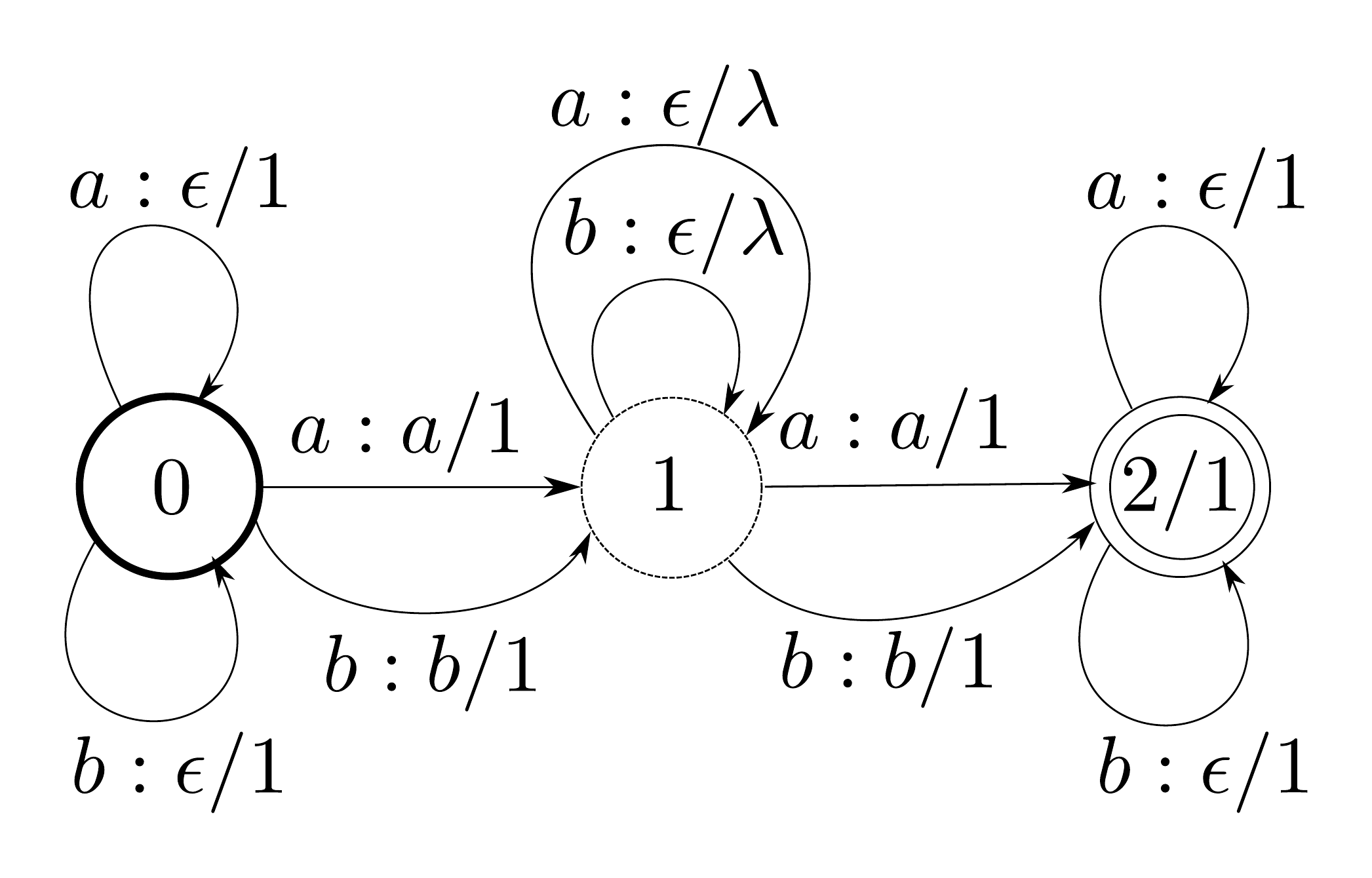}
\caption{Gappy bigram transducer with decay $\lambda$. Output sequence is the bigram.} 
\label{fig:gappy_bigram_simple}
\end{figure*}

\begin{figure*}[!t]
\centering
\includegraphics[width=6 in]{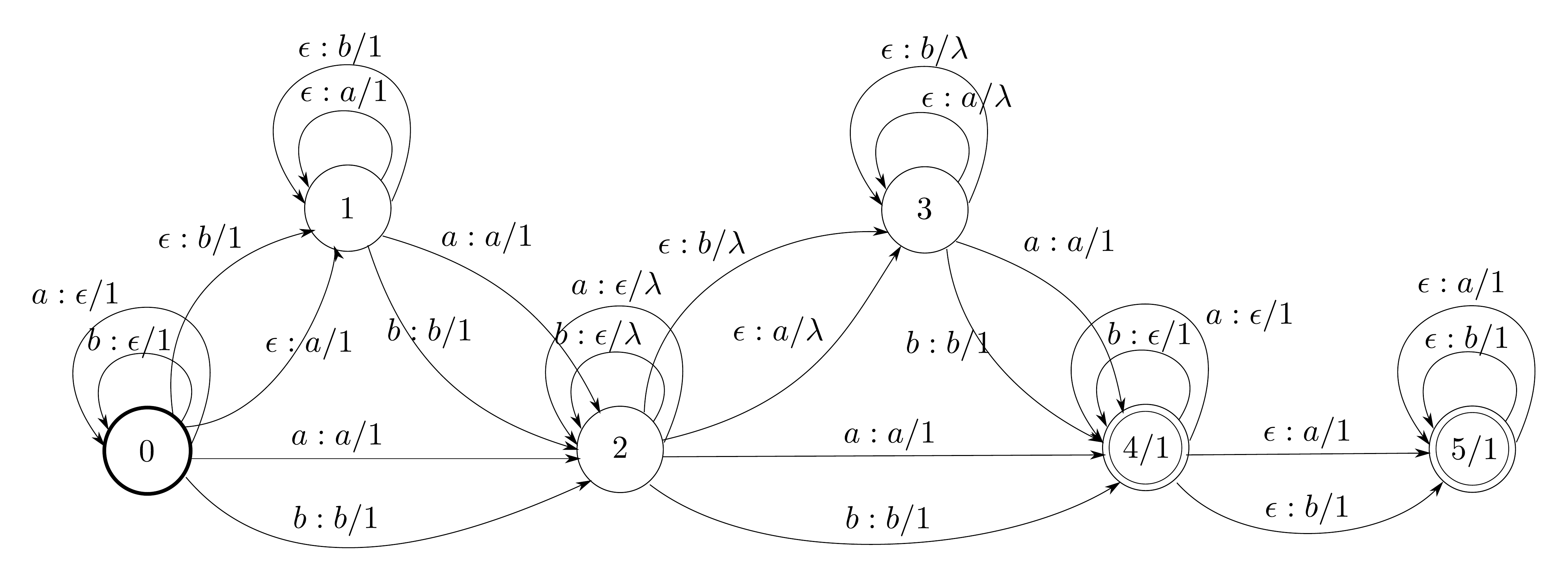}
\caption{Gappy bigram transducer with decay $\lambda$ for arbitrary sequences.} 
\label{fig:gappy_bigram}
\end{figure*}

Fig \ref{fig:gappy_bigram_simple} (from \cite{Mohri:2012:FML:2371238}) shows the transducer constructed on the lines of the bigram counter. For each self-transition on state $1$, a multiplicative penalty of $\lambda$ is accumulated. Fig \ref{fig:gappy_bigram} (from \cite{Cortes02rationalkernels}) shows yet another gappy bigram counter which can accept arbitrary sequences as its second argument i.e. it doesn't require composition to compare sequences. Both transducers are defined  over the probability semiring.

\item Mismatch string kernels: These kernels are used in computational biology for discriminative protein classification (\cite{2055}). Let $\Sigma$ be a finite alphabet, typically that of amino acids for protein sequences. For any two sequences $z_1,z_2 \in \Sigma^*$ of same length $(|z_1| = |z_2|)$, we denote by $d(z_1,z_2)$ the total number of mismatching symbols between these sequences. For all $m \in \mathbb{N}$, we define the bounded distance $d_m$ between two sequences of same length by:
\[
    d_m(z_1, z_2)= 
\begin{cases}
    1 \text{ if } d(z_1, z_2) \leq m\\
    0 \text{ otherwise. }
\end{cases}
\]
We define the projection $\Phi_{(k,m)}$ of a string $u, |u|=k$, to the $\Sigma^k$ space, indexed by all $k$-length strings from $\Sigma$, in the following way:
\begin{equation*}
\Phi_{(k,m)}(u) = (d_m(u,z))_{z \in \Sigma^k} 
\end{equation*}
The intuition here is to identify $u$ in the projected space by all $k$-length strings that are no more than $m$ mismatches away. We extend this definition to a string $x$ of arbitrary length, by first defining the set of $k$-length \emph{factors} of the string as:
\begin{equation*}
F_k(x) = \{z: \text{substring of } x, |z|=k\}
\end{equation*}
and then defining the projection as:
\begin{equation*}
\Phi_{(k,m)}(x) =  \sum\limits_{u \in F_k(x)} \Phi_{(k,m)}(u)
\end{equation*}
Note that the factor set only includes \emph{contiguous} substrings. For ex, $F_2(abc) = \{ab, bc\}$. Specifically, $ac \notin F_2(abc)$. 

For any $k,m \in \mathbb{N}$ with $m \leq k$, the $(k,m)$-mismatch kernel $K_{(k,m)} : \Sigma^* \times \Sigma^* \rightarrow \mathbb{R}$ is defined over sequences $x,y \in \Sigma^*$ by:
\begin{equation}
%K_{(k,m)}(x,y) = \sum\limits_{z_1\in F_k(x),z_2\in F_k(y), z \in \Sigma^k} d_m (z_1,z) d_m(z,z_2)
K_{(k,m)}(x,y) = \langle \Phi_{(k,m)}(x), \Phi_{(k,m)}(y)\rangle 
\end{equation}
This may be equivalently written as:
\begin{equation}
K_{(k,m)}(x,y) = \sum\limits_{z_1\in F_k(x),z_2\in F_k(y), z \in \Sigma^k} d_m (z_1,z) d_m(z,z_2)
\end{equation}
Although we don't prove it here (see \cite{Cortes04rationalkernels:}), $K_{(k,m)}(x,y) =T_{(k,m)} \circ T_{(k,m)}^{-1} (x,y)$ where $T_{(k,m)}(x,y) = \sum_{u \in F_k(x)}d_m(y, u)$. Fig \ref{fig:mismatch_kernel} shows such a transducer, $T_{(3,2)}$ defined on the probability semiring. Every accepting path has exactly three edges that don't have an $\epsilon$ output symbol. A mismatch is signified by ``traveling down'' from a level, and for every such descent, the distance from the accessible final states, in terms of non-$\epsilon$ edges, decreases by one. Since we can go down only two levels, it is ensured that we can have only up to two mismatches, while there are exactly three non-$\epsilon$ edges in the path.
\begin{figure*}[!t]
\centering
\includegraphics[width=3 in]{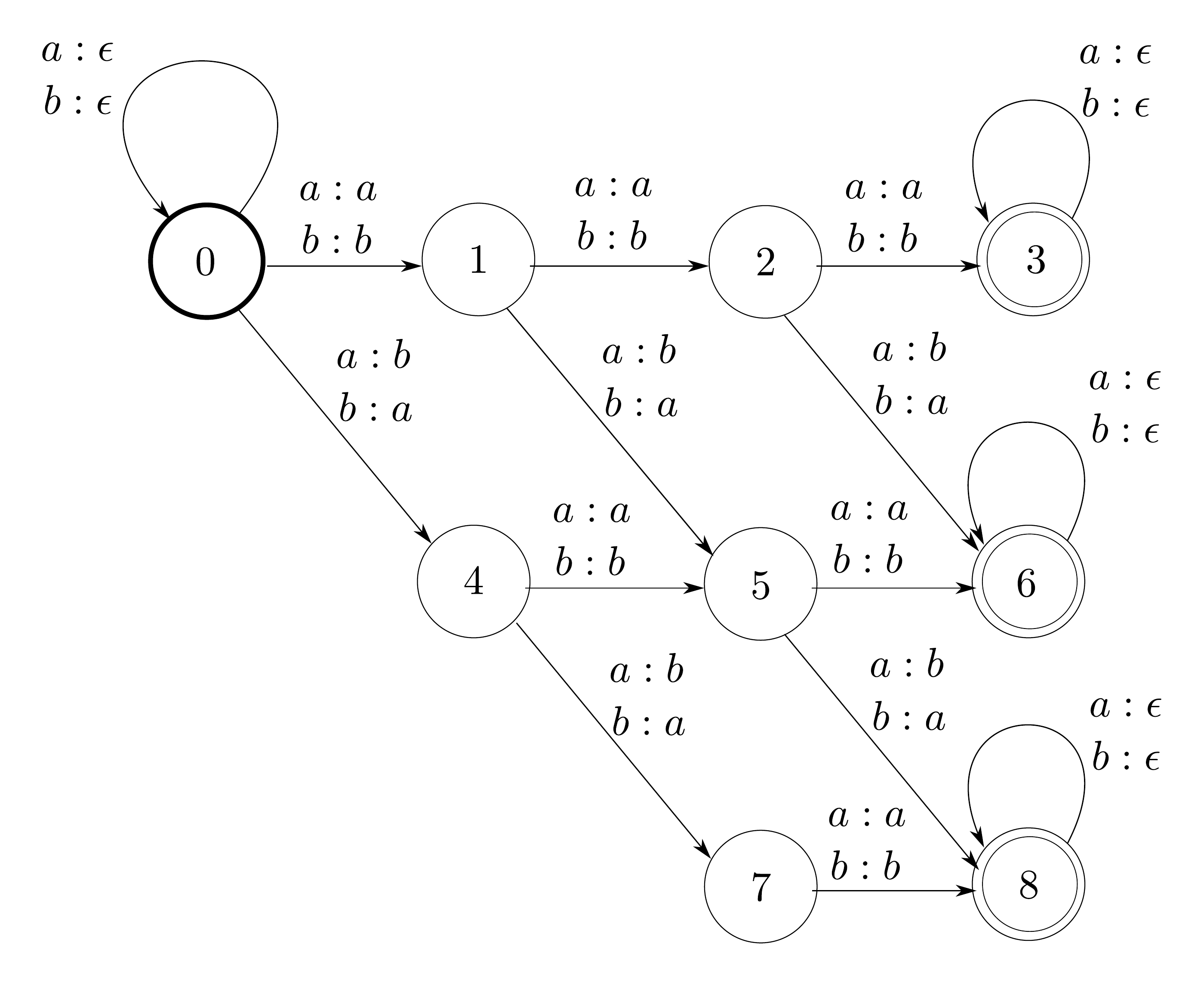}
\caption{$T_{3,2}$ corresponding to the mismatch kernel $K_{3,2} = T_{3,2} \circ T_{3,2}^{-1}$, over the probability semiring. All transitions and final states have a weight of $1$.} 
\label{fig:mismatch_kernel}
\end{figure*}

\end{enumerate}

We don't provide any more examples in the interest of brevity but we would encounter some more transducers representing similarity/distance measures in later sections: the \emph{edit-distance} transducer in Section \ref{sec: properties_rational_kernels}, and a transducer corresponding to the \emph{subsequence kernel} in Section \ref{sec:subsequence_kernels_PT_test}.

\subsection{Properties}
\label{sec: properties_rational_kernels}
In this section we look at various properties of rational kernels. Of special interest is the property that the kernel be \emph{positive definite symmetric (PDS)}. A kernel is eventually utilized by learning techniques like \emph{SVMs}, \emph{kernel Principal Component Analysis}, \emph{kernel ridge regression} etc, where this property ensures an optimal solution. We also look at the related property of kernel being \emph{negative definite symmetric (NDS)}; such kernels can be used to construct PDS kernels. 
\begin{definition}
A kernel $K: X \times X \rightarrow \mathbb{R}$ is said to be positive definite symmetric (PDS) if for any $\{x_1, ... , x_m\} \subseteq X$, the matrix $\boldsymbol{K} = [K(x_i, x_j)]_{ij} \in \mathbb{R}^{m\times m}$ is symmetric and positive semidefinite. (from  \cite{Mohri:2012:FML:2371238})
\end{definition}
Although we have made a distinction between the kernel \emph{function} $K$ and the kernel \emph{matrix} $\boldsymbol{K}$ above, we would loosely use the function notation $K$ for both, leaving it to the context to convey the intended meaning. Also note that while we need the kernel function $K$ to be positive definite, we need $\boldsymbol{K}$ to be positive \textbf{semi}-definite. This is also known as \emph{Mercer's condition}.

We begin by looking at \emph{closure} properties of PDS rational kernels. These properties provide us with tools to piece together a PDS rational kernel from other PDS rational kernels. Closure here implies that not only do we need to prove that the resulting kernel is PDS, but we also must show that it is rational.

First, we summarize the closure properties of general PDS kernels, since these would be referred to later.
\begin{theorem} Let $X$ and $Y$ be two non-empty sets.
\label{theorem:pds_kernels}
\begin{enumerate}
\item Closure under sum: Let $K_1,K_2: X \times X \rightarrow \mathbb{R}$ be PDS kernels, then $K_1 + K_2: X \times X \rightarrow \mathbb{R}$ is a PDS kernel.
\item Closure under product: Let $K_1,K_2: X \times X \rightarrow \mathbb{R}$ be PDS kernels, then $K_1 \cdot K_2: X \times X \rightarrow \mathbb{R}$ is a PDS kernel.
\item Closure under tensor product: Let $K_1: X \times X \rightarrow \mathbb{R}$ and $K_2: Y \times Y \rightarrow \mathbb{R}$ be PDS kernels, then their tensor product $K_1 \odot K_2: (X \times Y) \times (X \times Y) \rightarrow \mathbb{R}$, where $K_1 \odot K_2((x_1,y_1),(x_2,y_2)) = K_1 (x_1,x_2) \cdot K_2(y_1,y_2)$ is a PDS kernel.
\item Closure under pointwise limit: Let $K_n: X \times X \rightarrow \mathbb{R}$ be a PDS kernel for all $n \in \mathbb{N}$ and assume that $\lim\limits_{n\rightarrow \infty} K_n(x_1,x_2)$ exists for all $x_1,x_2 \in X$, then $K$ defined by $K(x_1,x_2) = \lim\limits_{n \rightarrow \infty} K_n(x_1,x_2)$ is a PDS kernel.
\end{enumerate}
\end{theorem}

We now prove the following closure properties of PDS rational kernels:
\begin{theorem} \textbf{Closure Properties of PDS rational kernels}

Let $\Sigma$ be a non-empty alphabet and $\Psi:K \rightarrow \mathbb{R}$ be a function used to define rational kernels e.g. $K(x,y)=\Psi([\![T]\!](x,y))$. Then the following properties for PDS rational kernels hold:
\begin{enumerate}
\item Closure under $\oplus$-sum: Assume that $\Psi: (K,\oplus,\overline{0}) \rightarrow (\mathbb{R},+,0)$ is a monoid morphism. Let $K_{T_1}, K_{T_2}: \Sigma^* \times \Sigma^*\rightarrow \mathbb{R}$ be PDS rational kernels, then $K_{T_1 \oplus T_2}: \Sigma^* \times \Sigma^* \rightarrow \mathbb{R}$ is a PDS rational kernel and $K_{T_1 \oplus T_2} =  K_{T_1} +K_{T_2}$.
\item Closure under $\otimes$-product: Assume that $\Psi: (K,\oplus,\otimes,0,1) \rightarrow (\mathbb{R},+,\times,0,1)$ is a semiring morphism. Let $K_{T_1},K_{T_2}: \Sigma^* \times \Sigma^* \rightarrow \mathbb{R}$ be PDS rational kernels, then $K_{T_1 \otimes T_2}: \Sigma^* \times \Sigma^* \rightarrow \mathbb{R}$ is a PDS rational kernel.
\item Closure under Kleene-closure: Assume that $\Psi: (K,\oplus,\otimes,0,1) \rightarrow (\mathbb{R},+,\times,0,1)$ is a continuous semiring morphism. Let $K_T: \Sigma^* \times \Sigma^* \rightarrow \mathbb{R}$ be a PDS rational kernel, then $K_{T^*}: \Sigma^* \times \Sigma^* \rightarrow \mathbb{R}$ is a PDS rational kernel.
\end{enumerate}
\begin{proof}

\begin{enumerate}
\item A \emph{monoid morphism}  is a function $\Psi: (K,\oplus,\overline{0}) \rightarrow (\mathbb{R},+,0)$ that satisfies $\Psi(x \oplus y) = \Psi(x)+ \Psi(y)$ for all $x,y \in K$, and $\Psi(\overline{0}) = 0$. Thus,
\begin{equation} 
\Psi([\![T_1]\!](x,y) \oplus [\![T_2]\!](x,y)) = \Psi([\![T_1]\!](x,y))+ \Psi([\![T_2]\!](x,y))
\end{equation}
Since PDS kernels are closed under sum (Property 1, Theorem \ref{theorem:pds_kernels}), the above $\oplus$-sum is a PDS kernel. The corresponding transducer is simply the two transducers $T_1, T_2$ considered together (with no new connections between them), so that $I_{T_1 \oplus T_2} = I_{T_1} \cup I_{T_2}$ and $F_{T_1 \oplus T_2}= F_{T_1} \cup F_{T_2}$, where $I_X$ and $F_X$ are the set of initial and final states of transducer $X$ respectively. Thus, the $\oplus$-sum is also defines rational kernel.
\item A \emph{semiring morphism} $\Psi$ is a function $\Psi : (K,\oplus,\otimes,\overline{0},\overline{1}) \rightarrow (\mathbb{R},+,\times,0,1)$ that is a monoid morphism and additionally satisfies $\Psi(x \otimes y) =
\Psi(x) \cdot \Psi(y)$ for all $x,y \in K$, and $\Psi(\overline{1}) = 1$. We have,
\begin{align}
\Psi([\![T_1 \otimes T_2]\!](x,y)) &= \Psi \bigg( \bigoplus_{x_1 x_2=x,y_1 y_2=y} [\![T_1]\!](x_1,y_1) \otimes [\![T_2]\!](x_2,y_2) \bigg) \nonumber \\
&=\sum \limits_{x_1 x_2=x,y_1 y_2=y} \Psi([\![T_1]\!](x_1,y_1))·\cdot \Psi([\![T_2]\!](x_2,y_2)) \nonumber \\
& = \sum\limits_{x_1 x_2=x,y_1 y_2=y} K_{T_1} \odot K_{T_2}((x_1,x_2),(y_1,y_2)).
\end{align}
By Theorem \ref{theorem:pds_kernels}, since $K_{T_1}$ and $K_{T_2}$ are PDS kernels, their tensor product $K_{T_1} \odot K_{T_2}$ is a PDS kernel\footnote{note that we cannot use the closure property of the sum here, since each term in the sum computes the kernel value for a different $(x_1, x_2),(y_1, y_2)$ pair.} and there exists a Hilbert space $H \subseteq \mathbb{R}^{\Sigma^*}$ and a mapping $u \rightarrow \phi_u$ such that $K_{T_1} \odot K_{T_2}(u,v) = \langle \phi_u,\phi_v \rangle$.
\begin{align}
\Psi([\![T_1 \otimes T_2]\!](x,y)) &= \sum \limits_{x_1 x_2=x,y_1 y_2=y} \langle \phi_{(x_1,x_2)},\phi_{(y_1,y_2)}\rangle \nonumber \\
&=\bigg \langle \sum \limits_{x_1 x_2=x} \phi_{(x_1,x_2)} , \sum \limits_{y_1 y_2=y} \phi_{(y_1,y_2)}   \bigg \rangle
\end{align}
Since a dot product is positive definite, $K_{T_1\otimes T_2}$ is a PDS kernel. The corresponding transducer can be thought of as $T_1$ and $T_2$ lined up next to each other, so that the final states of $T_1$ lead into the initial states of $T_2$, via newly introduced transitions. The weights from the final states of $T_1$ become the weights of these new transitions. Fig \ref{fig:transducer_prod} illustrates this.
\begin{figure*}[!t]
\centering
\includegraphics[width=3 in]{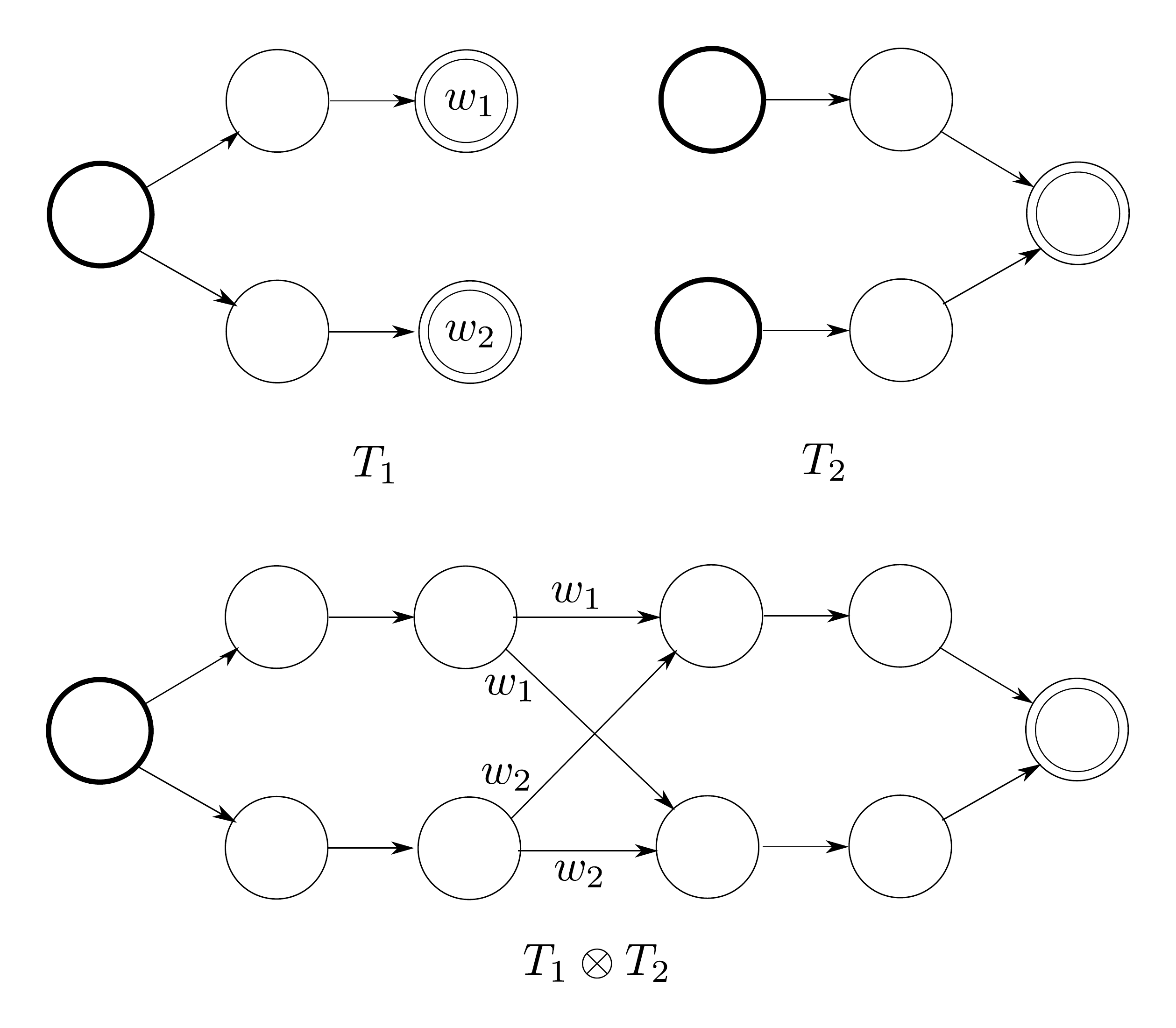}
\caption{$T_1 \otimes T_2$ is shown. The final states of $T_1$ are connected to the initial states of $T_2$. Weights $w_1, w_2$ are moved out onto the new edges.} 
\label{fig:transducer_prod}
\end{figure*}

\item The closure under Kleene-closure is a direct consequence of the closure under $\oplus$-sum and $\otimes$-product of PDS rational kernels and the closure under pointwise limit of PDS kernels.
\end{enumerate}

\end{proof}
\end{theorem}

Unfortunately, many transducers do not naturally define a kernel that is PDS (for ex, the bigram counter) - which renders them of little practical value. However, as the following proposition shows, this shortcoming is readily addressed with a simple construction.
\begin{proposition}
Let $T = (\Sigma,\Delta,Q,I,F,E,\lambda,\rho)$ be a weighted finite-state transducer defined over the semiring $(K,\oplus,\otimes,\overline{0},\overline{1})$. Assume that the weighted transducer $T \circ T^{-1}$ is regulated, then $(\Psi,T \circ T^{-1})$ defines a PDS rational kernel over $\Sigma^* \times \Sigma^*$.
\begin{proof}
Let $S$ denote the transducer $T \circ T^{-1}$, and $K$ denote the rational kernel $S$ defines. We have
\begin{equation*}
K(x,y) = \Psi([\![S]\!](x,y)) = \Psi \bigg( \bigoplus_{z \in \Delta^*}[\![T]\!](x,z)\otimes[\![T]\!](y,z) \bigg)
\end{equation*}
for all $x,y \in \Sigma^*$. Since $\Psi$ is a continuous semiring morphism, the above is equivalent to:
\begin{equation*}
K(x,y) = \sum\limits_{z \in \Delta^*} \Psi ([\![T]\!](x,z))\cdot \Psi ([\![T]\!](y,z))
\end{equation*}
For all $n \in \mathbb{N}$ and $x,y \in \Sigma^*$, define $K_n(x,y)$ by
\begin{equation*}
K_n(x,y) = \sum\limits_{|z|\leq n}\Psi([\![T]\!](x,z))\cdot \Psi([\![T]\!](y,z))
\end{equation*}
For any $l \geq 1$ and any $x_1,...,x_l \in \Sigma^*$, define the matrix $M_n$ by $M_n =(K_n(x_i,x_j)), i \leq l, j\leq l$. Let $z_1,z_2,...,z_m$ be an arbitrary ordering of the strings of length less than or equal to $n$. We define the matrix $A$ by
\begin{equation}
A = (\Psi([\![T]\!](x_i,z_j)))_{i \leq l; j\leq m}
\end{equation}
We observe that $M_n = AA^T$. %Also, $A$ is clearly rectangular since $x_i$ are not limited by length. 
Since, $(AA^T)^T = (A^T)^TA^T= AA^T$, $M_n$ is symmetric. $M_n$ is also positive semidefinite since $p^T(AA^T)p =(A^Tp)^T(A^Tp) \geq 0$, for a non-zero column-vector $p$. This proves $K_n(x,y)$ is a PDS kernel.

Since $K$ is a pointwise limit of $K_n$ i.e. $K(x,y) = \lim_{n\rightarrow \infty} K_n(x,y)$, by Property 4, Theorem \ref{theorem:pds_kernels}, $K$ is a PDS kernel too.
\end{proof}
\end{proposition}

Interestingly, it can be shown that the converse of the above proposition is also true. A PDS rational kernel defined by a transducer $S$ can \emph{always} be decomposed into an equivalent composition of transducers $T \circ T^{-1}$. This is particularly useful in light of the fact that the scope of some proofs are increased due to this equivalence (for example, see eqn (\ref{eqn:large_rational_derivation_1}), Section \ref{sec:large_scale} and eqn (\ref{eqn:count_kernel}), Section \ref{sec:learning_kernels}).

\begin{proposition} Let $S = (\Sigma,\Sigma,Q,I,F,E,\lambda,\rho)$ be an acyclic weighted finite-state transducer over $(K,\oplus,\otimes,\overline{0},\overline{1})$ such that $(\Psi,S)$ defines a PDS rational kernel on $\Sigma^* \times \Sigma^*$. Then there exists a weighted transducer $T$ over the probability semiring such that $(Id_\mathbb{R},T \circ T^{-1})$ defines the same rational kernel. $Id_\mathbb{R}$ denotes the identity function over $\mathbb{R}$.
\begin{proof}
Since $S$ is symmetric, if it accepts the pair $(x,y), x,y \in \Sigma^*$, it also accepts $(y, x)$. Hence, the set $X$, defined as the set of accepted input sequences, contains all sequences $S$ accepts either as input \emph{or} output. Let $\{x_1, x_2, ..., x_n\}$ be an arbitrary numbering of elements in $X$. Define the matrix $M$ as:
\begin{equation}
M = (\Psi([\![S]\!](x_i, x_j)))_{1 \leq i \leq n, 1 \leq j \leq n}
\end{equation}
Since $S$ defines a PDS kernel, $M$ is symmetric and positive semidefinite. The \emph{Cholesky decomposition} extends to this case\footnote{the standard case is when a matrix is positive definite.} (see \cite{doi:10.1137/1.9781611971811.ch8}) and $M = RR^T$, where $R=(R_{ij})$ is an upper triangular matrix with non-zero diagonal elements. Let $Y =\{y_1, ... ,y_n\}$ be an arbitrary subset of $n$ distinct strings of $\Sigma^*$. Define the weighted transducer $T$ over the $X \times Y$ by
\begin{equation}
[\![T]\!](x_i, y_j) = R_{ij}
\end{equation}
By definition of the composition operator, $T \circ T^{-1}(x_i,x_j) =(RR^T)_{i,j} = (M)_{i,j} = \Psi([\![S]\!](x_i, x_j)) $ for all $i, j, 1 \leq i, j \leq n$. Thus, $T \circ T^{-1} = \Psi(S)$.

Note that since $M$ is symmetric and positive \emph{semidefinite}, a unique $R$ may not always exist. A symmetric and positive \emph{definite} $M$ results in a unique $R$. Also note that although we know what the transition function for $T$ looks like, this proof doesn't provide us with a simple/minimal transducer representation.
\end{proof}
\end{proposition}

We now move away from the property of PDS and look at the related property of a kernel being \emph{negative definite symmetric (NDS)}. Intuitively, such kernels represent \emph{distances}, compared to PDS kernels representing similarities. The essential idea is, for certain kinds of data it is easier to think in terms of the distance between two points. NDS kernels help concretize such measures. We still have to adhere to the requirement for SVMs etc that the input kernel be PDS; results around NDS kernels help here too: they show how NDS kernels may be used to construct PDS kernels. We visit this idea only briefly, to study a very commonly used distance measure for sequences, the \emph{edit-distance}.

NDS kernels were first studied  in \cite{Berg84}; many of the results are discussed in \cite{NIPS2000_1862} (where the term \emph{conditionally positive definite} is used instead of NDS). 

\begin{definition}
\label{defn:nds}
Let $X$ be a non-empty set. A function $K: X \times X \rightarrow \mathbb{R}$ is said to be a negative definite symmetric kernel (NDS kernel) if it is symmetric $(K(x,y) = K(y,x)$ for all $x,y \in X)$ and
\begin{equation}
\sum\limits_{i, j=1}^n c_i c_jK(x_i,x_j) \leq 0
\end{equation}
for all $n \geq 1, \{x_1,...,x_n\} \subseteq X$ and $\{c_1,...,c_n\} \subseteq R$ with $\sum_{i=1}^{n} c_i = 0$
\end{definition}
Clearly, if $K$ is a PDS kernel then $-K$ is a NDS kernel; however the converse does not hold in general. 

We had mentioned that NDS kernels may be used to construct PDS kernels.The following theorem shows a couple of ways to do this.
\begin{theorem}
\label{theorem:nds_construction}
Let $X$ be a non-empty set, $x_o \in X$, and let $K: X \times X \rightarrow \mathbb{R}$ be a symmetric kernel. Then,
\begin{enumerate}
\item $K$ is negative definite iff $\exp({-tK})$ is positive definite for all $t > 0$. (ref \cite{Berg84})
\item Let $K'$ be the function defined by
\begin{equation}
K'(x,y) = K(x,x_0)+ K(y,x_0)- K(x,y) - K(x_0,x_0).
\end{equation}
Then $K$ is negative definite iff $K'$ is positive definite. (ref \cite{NIPS2000_1862})
\end{enumerate}
\end{theorem}

NDS kernels are also closed under certain operations (\cite{Berg84}). These are listed below.
\begin{theorem}
\label{theorem:nds_closure}
Let $X$ be a non-empty set.
\begin{enumerate}
\item Closure under sum: Let $K_1,K_2: X \times X \rightarrow \mathbb{R}$ be NDS kernels, then $K_1+ K_2: X \times X \rightarrow \mathbb{R}$ is a NDS kernel.
\item Closure under log and exponentiation: Let $K: X \times X \rightarrow \mathbb{R}$ be a NDS kernel with $K \geq 0$, and $\alpha$ a real number with $0 < \alpha < 1$, then $log(1+ K),K_\alpha: X \times X \rightarrow \mathbb{R}$ are NDS kernels.
\item Closure under pointwise limit: Let $K_n: X \times X \rightarrow \mathbb{R}$ be a NDS kernel for all $n \in N$, then $K$ defined by $K(x,y) = \lim_{n \rightarrow \infty} K_n(x,y)$ is a NDS kernel.
\end{enumerate}
\end{theorem}

We now look at the \emph{edit-distance} rational kernel. The edit-distance is a commonly used distance measure between two sequences. Given sequences $A$ and $B$, and the operations \emph{insert}, \emph{delete} and \emph{replace} (each applied to a symbol), and costs corresponding to each of these operations (assumed to be 1 here), the edit-distance is the minimum total cost of a series of operations needed to transform $A$ to $B$. For ex if $A=abb$ and $B=bb$, then two possible ways of transforming $A$ to $B$ are:
\begin{enumerate}
\item Replace the first symbol in $A$ with $b$, delete the last symbol in $A$. Total cost = 2.
\item Delete the first symbol in $A$. Total cost = 1.
\end{enumerate}
A total cost of $1$ is the minimum across all possible transformations. Hence, the edit-distance between $A$ and $B$ is $1$, and is denoted by $d_e(A,B)=1$.

Fig \ref{fig:transducer_string_edit} shows rational kernels that evaluate edit-distance over the tropical and probability semirings. The standard notion of edit-distance aligns with the former. In this case, transitions with the same input and output character contribute a weight of $0$, while for all other cases a weight of $1$ is added (recall that $\otimes$ is defined as ``$+$'' for the tropical semiring). The weight of an accepting path measures the total cost of the transformation corresponding to the path. The $\oplus$ operator, which is defined as ``$\min$'' here, picks the minimum such cost giving us the edit-distance.

\begin{figure*}[!t]
\centering
\subfloat[Edit-distance over the tropical semiring.]{\includegraphics[width=0.85 in]{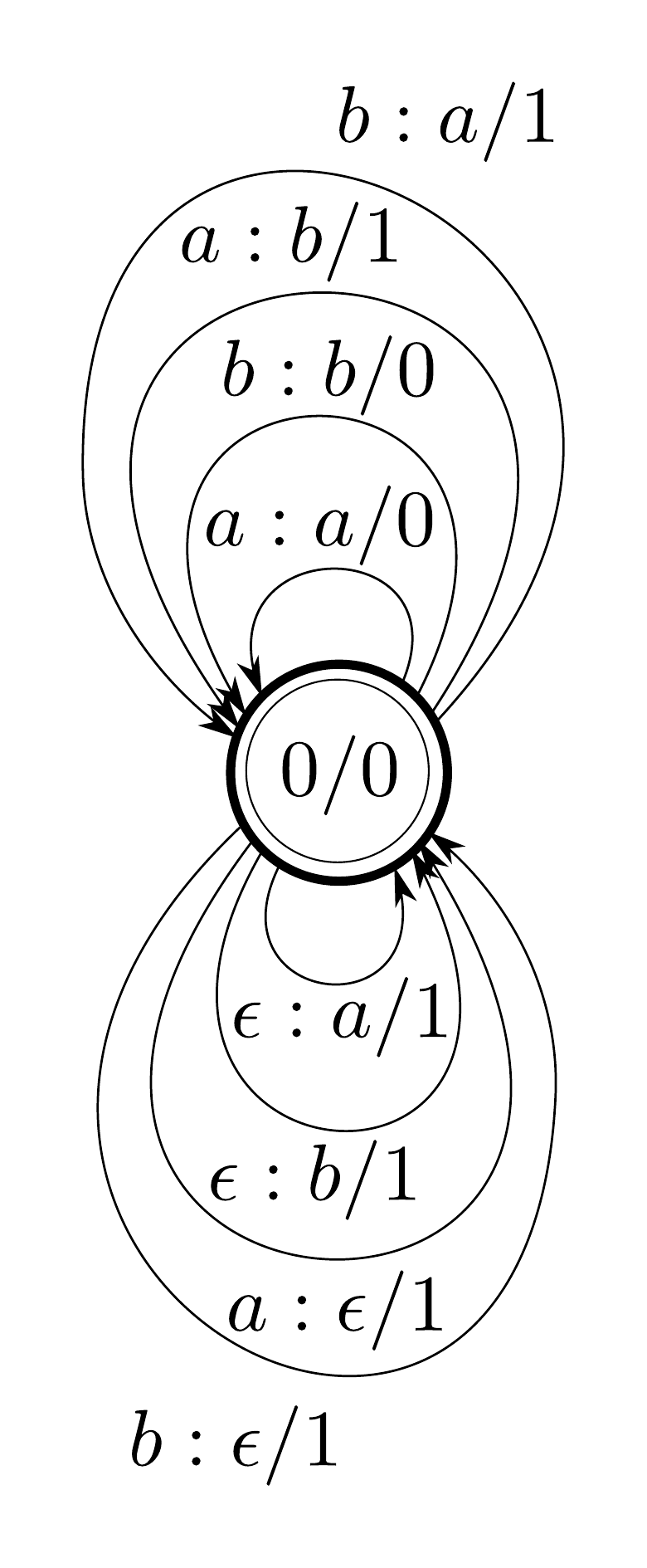}%
\label{fig_edit_dist_tropical}}
\hfil
\subfloat[Edit-distance over the probability semiring.]{\includegraphics[width=3in]{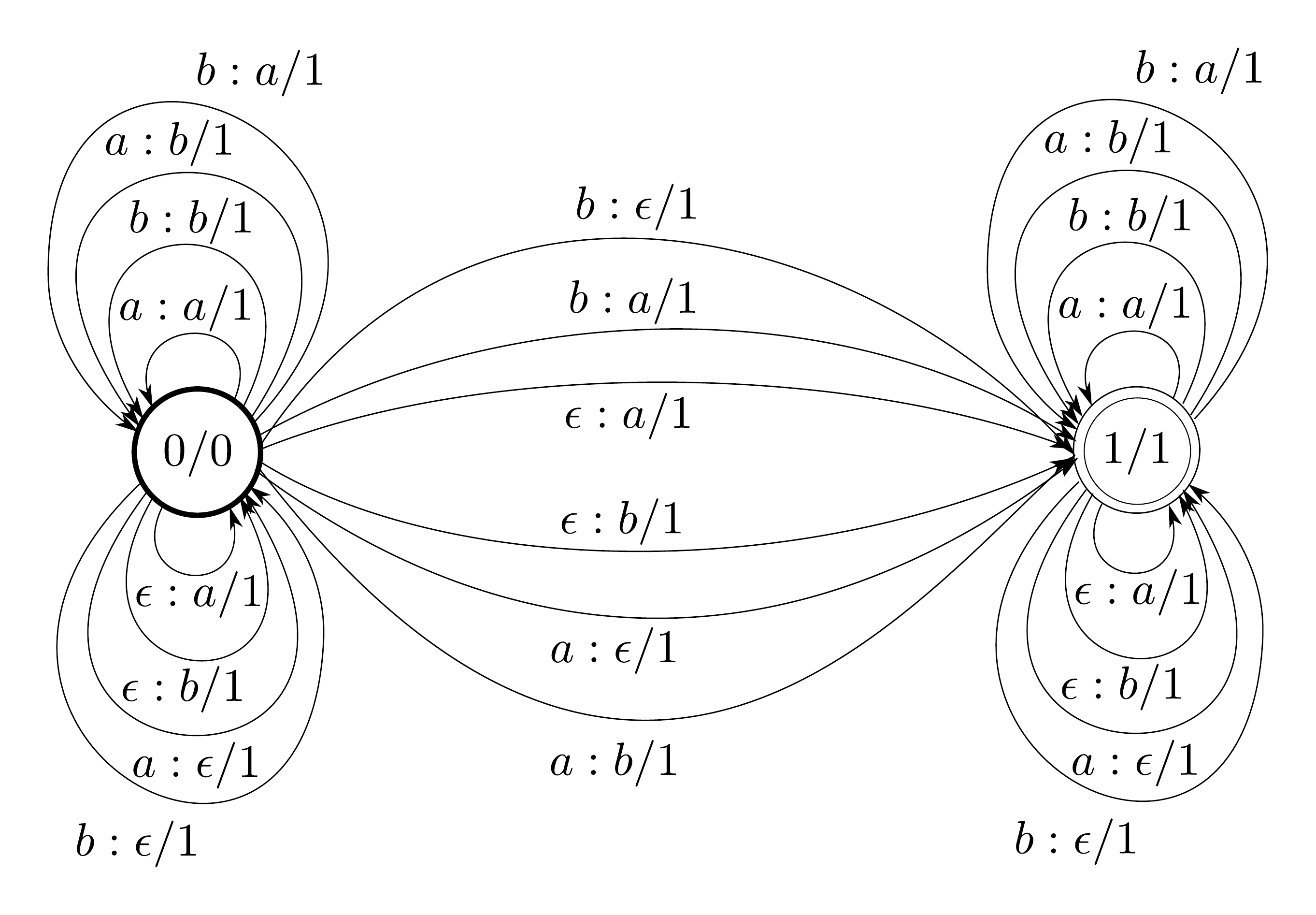}%
\label{fig_edit_dist_probability}}
\caption{Transducers representing edit-distance.} 
\label{fig:transducer_string_edit}
\end{figure*}

Given the popularity of this measure, and the fact that it is representable as a rational kernel, we are interested in knowing its properties. The following proposition looks at them.

\begin{proposition}
Let $\Sigma$ be a non-empty finite alphabet and let $d_e$ be the edit-distance over $\Sigma$, then $d_e$ is a symmetric rational kernel. Furthermore, (1): $d_e$ is not a PDS kernel, and (2): $d_e$ is a NDS kernel iff $\lvert \Sigma \rvert = 1$.
\begin{proof}
We have already seen the transducers corresponding to edit-distance in Fig \ref{fig:transducer_string_edit}.

Let $a \in \Sigma$. Consider the matrix $(d_e(x_i,x_j))_{1 \leq i, j \leq 2}$ with $x_1 = \epsilon$ and $x_2 = a$:
\[\begin{bmatrix}
    0       & 1 \\
    1       & 0
\end{bmatrix}
\]
The above matrix has the eigenvalues $(1,-1)$. For a kernel to be PDS, a kernel matrix determined by \emph{any} number of samples must be positive semidefinite, or equivalently, must have non-negative eigenvalues. Hence, edit-distance is not PDS.

It turns out that the edit-distance defines a NDS kernel only when $\lvert \Sigma \rvert=1$. In this case, the edit-distance is simply the difference in the length of the sequences. Thus, for sequences $x,y \in \Sigma^*$, we have $d_e(x,y)=\lvert \lvert x \rvert - \lvert y \rvert \rvert$. Note the symbol overloading here: $\lvert x \rvert, \lvert y \rvert$ denote the lengths of $x,y$ respectively, whereas the outer ``$\lvert \rvert$'' denotes absolute value of a number.

Let $\sum_{i}c_i=0$. Consider\footnote{\cite{Cortes04rationalkernels:} uses a different proof}:
\begin{align*}
\sum_{i,j}c_i c_j d_e(x_i, x_j)^2&=\sum_{i,j}c_i c_j \lvert\lvert x_i \rvert -\lvert x_j \rvert \rvert^2\\
&= \sum_{i,j}c_i c_j (\lvert x_i \rvert -\lvert x_j \rvert)(\lvert x_i \rvert -\lvert x_j \rvert) \\
&=\sum_{i,j}c_i c_j (\lvert x_i \rvert^2 + \lvert x_j \rvert^2 - 2 \lvert x_i \rvert \lvert x_j \rvert) \\
&= \sum_{i,j}c_i c_j \lvert x_i \rvert^2 + \sum_{i,j}c_i c_j \lvert x_j \rvert^2 - 2 \sum_{i,j}c_i c_j \lvert x_i \rvert \lvert x_j \rvert \\
&= \sum_{j} c_j \sum_{i} c_i \lvert x_i \rvert^2  + \sum_{i} c_i \sum_{j} c_j \lvert x_j \rvert^2 - 2 \sum_{i,j}c_i c_j \lvert x_i \rvert \lvert x_j \rvert \\
&=- 2 \sum_{i,j}c_i c_j \lvert x_i \rvert \lvert x_j \rvert  = -2 \bigg(\sum_{i} c_i \lvert x_i \rvert \bigg)\bigg(\sum_{j} c_j \lvert x_j \rvert \bigg)\\
&= -2 \bigg(\sum_{j} c_j \lvert x_j \rvert \bigg)^2 \leq 0
\end{align*}
By definition \ref{defn:nds}, we have $d_e^2$ is NDS, when $|\Sigma|=1$. By property 2, theorem \ref{theorem:nds_closure}, $(d_e^2)^\frac{1}{2} = d_e$ is also NDS.

Lets consider the case $|\Sigma|>1$. Assume $a,b \in \Sigma$. Let $x_1,...,x_{2^n}$ be any ordering of the strings of length $n$ consisting only of the symbols $a,b$. Define the matrix $M_n$ by:
\begin{equation*}
M_n= (\exp(-d_e(x_i,x_j)))_{1\leq i, j,\leq2^n}
\end{equation*}
Fig \ref{fig:edit_distance_eigenvalues} shows a plot of the smallest eigenvalue of $M_n$ as a function of $n$. We see that the smallest value becomes negative for $n \geq 6$. Since $M_n$ can be seen as a sample kernel matrix for any $|\Sigma|>1$, we conclude that $\exp(-d_e)$ is not PDS. By property 1, theorem \ref{theorem:nds_construction}, it follows that $d_e$ is not NDS when $|\Sigma|>1$.
\end{proof}

\end{proposition}
\begin{figure*}[!t]
\centering
\includegraphics[width=3.0 in]{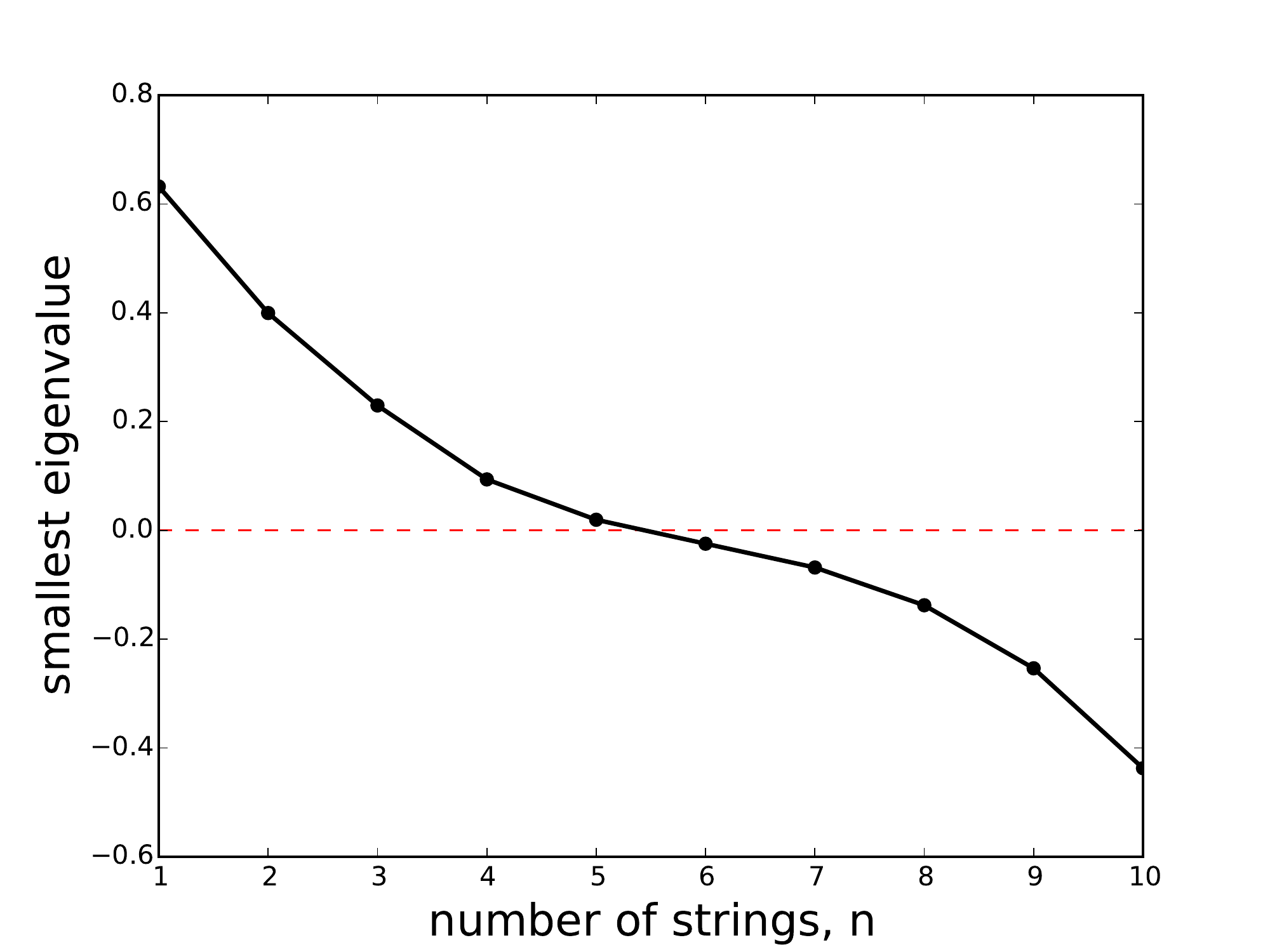}
\caption{Smallest eigenvalue of the matrix $M_n = (\exp(-d_e(x_i,x_j)))_{1 \leq i, j, \leq 2^n}$ as a function of $n$} 
\label{fig:edit_distance_eigenvalues}
\end{figure*}

\subsection{Computing Kernel Values}
\label{sec:computing_kernel_values}

To use rational kernels, in practice we need to concern ourselves with two kinds of computation:
\begin{enumerate}
\item Calculating the composition of two transducers $T_1, T_2$ - this is a common requirement given that our objective is to procure PDS kernels. States in the composition $T_1 \circ T_2$ are a subset of the cartesian product of the set of states from $T_1$ and $T_2$. Ignoring transitions with $\epsilon$ inputs or outputs, the transitions in $T_1 \circ T_2$ are identified by:
\begin{equation}
(q_1,a,b,w_1,q_2) \in E_1 \text{ and } (q'_1 ,b,c,w_2,q'_2) \in E_2\implies ((q_1,q'_1 ),a,c,w_1 \otimes w_2,(q_2,q'_2)) \in E
\end{equation}
where $E_1, E_2, E$ are the set of transitions for $T_1, T_2 \text{ and } T_1\circ T_2$ respectively. Fig \ref{fig:transducer_composition} shows an example of composition.

In the worst case, all transitions from state $q_1$ match all transitions from $q'_1$ making the space and time complexity of the algorithm quadratic: $O((|Q_1|+ |E_1|)(|Q_2|+ |E_2|))$.

Dealing with transitions with the $\epsilon$ either as the input label or output label requires special care since it might lead to transitions with $\epsilon$ as \emph{both} input and output symbols in $T_1 \circ T_2$. This makes the final transducer non-regulated. 

See \cite{citeulike:212964}, \cite{Mohri:2000:DPW:325997.326000} for details around implementing composition. \cite{DBLP:journals/ijfcs/AllauzenM09} looks at fast composition algorithms when more than two transducers are involved - known as \emph{N-way composition}. 
\item Calculating the kernel value itself, $K(x,y)$, for any sequences $x,y$ - A generalized form of the \emph{shortest distance} algorithm that applies to %$k$-closed\footnote{$\forall a \in \mathbb{K}, \bigoplus\limits^{k+1}_{n=0} a^n =\bigoplus\limits^{k}_{n=0} a^n$} 
semirings (\cite{Mohri:2002:SFA:639508.639512}), is used here. For a transducer $M$, the shortest distance from a state $q$ to a set of final states is defined as:
\begin{equation}
d[q] = \bigoplus\limits_{\pi\in P(q,F)} w[\pi] \otimes \rho[n[\pi]]
\end{equation}
This notion of shortest distance matches our standard notion of shortest distance when we consider the tropical semiring:
\begin{equation*}
d[q] = \min\limits_{\pi\in P(q,F)} w[\pi] \cdot \rho[n[\pi]]
\end{equation*}
In this case, there exists a path $\pi^*$ where $d[q] = w[\pi^*] \cdot \rho[n[\pi^*]]$, and hence the term \emph{shortest path} makes sense. In the general case however, there may not be a path corresponding to $d[q]$, and we use the term shortest distance. When $M$ is acyclic\footnote{although the transducers we use have self-loops, this property still applies to them since they are regulated, and hence path lengths are bounded by $\max(|x|,|y|)$ for $M(x,y)$}, the complexity of the algorithm is linear: $O(|Q|+ (T_\oplus + T_\otimes)|E|)$, where $T_\oplus$ denotes the maximum time to compute $\oplus$ and $T_\otimes$ the time to compute $\otimes$.

Let $K$ be a rational kernel with $T$ as its associated weighted transducer. Let $A$, $B$ be automata that only accept the sequences $x,y \in \Sigma^*$ respectively. $K(x,y)$ can be computed by:
\begin{enumerate} 
\item Construct the composition $N=A \circ T \circ B$. By eqn (\ref{eqn:automata_transducer_composition}):
\begin{equation*}
N(s,t) =\bigoplus\limits_{s,t \in \Sigma^*}[\![A]\!](s) \otimes [\![T]\!](s, t)  \otimes  [\![B]\!](t)
\end{equation*}
Since, $A(s)=0 \text{ if } s \neq x$ and $B(t)=0 \text{ if } t \neq y$, the \emph{only} accepting paths $N$ has, have the input sequence as $x$ and the output sequence as $y$.
\item Compute $w[N]$, the shortest distance from the initial states of $N$ to its final
states using the generalized shortest distance algorithm.
\item Compute $\Psi(w[N])$. 
\end{enumerate}
The worst case complexity for constructing $N$ in Step a) is $O(|A||T||B|)$, and worst complexity for computing the shortest distance in Step b) is $O(|Q_N|+ (T_\oplus + T_\otimes)|E_N|)$, where $Q_N$ and $E_N$ are the number of states and edges in $N$ respectively. The latter term may be approximated with the size of $N$, and hence the worst case complexity for Step b) may be written as $O(|A||T||B|)$ (since the space complexity of $N$ is $O(|A||T||B|)$). If $\Phi$ is the worst case complexity for computing $\Psi(x)$, then the overall worst case complexity is $O(|A||T||B| + \Phi)$. This simplifies to $O(|A||T||B|)$ if $\Phi$ is a constant, which is true in most cases.
\end{enumerate}

\begin{figure*}[!t]
\centering
\subfloat[$T_1$]{\includegraphics[width=2.75 in]{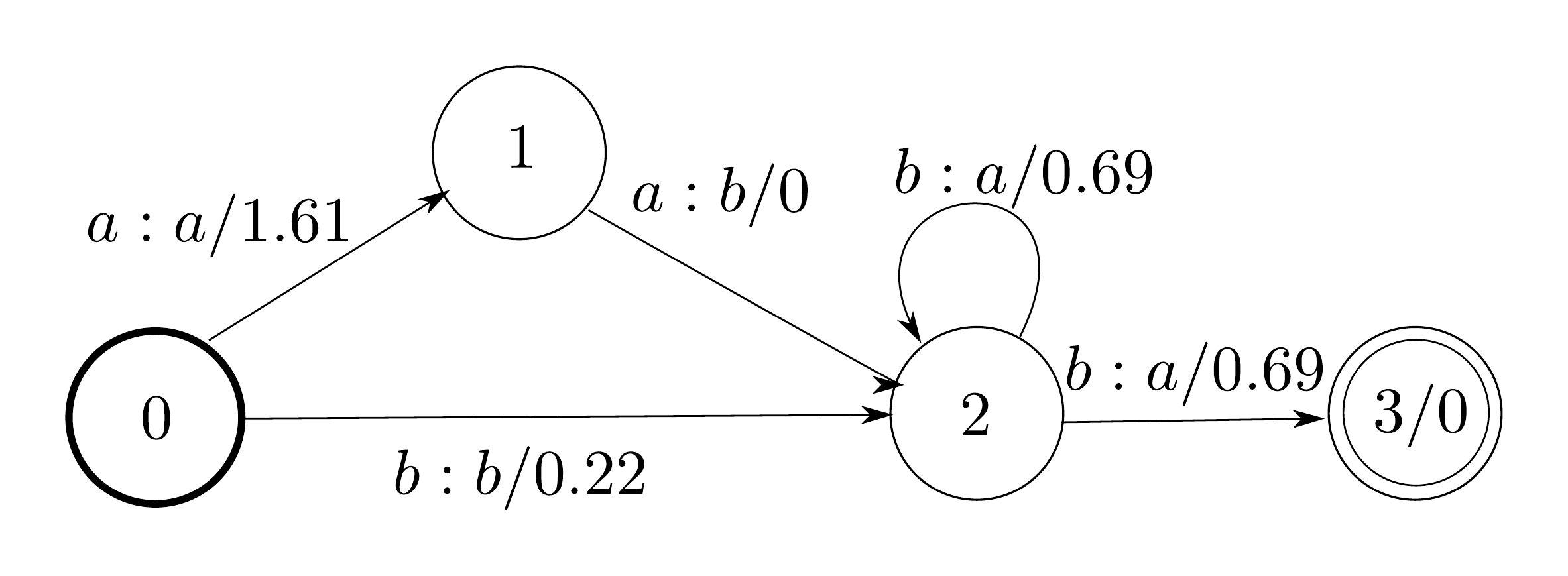}%
\label{fig_indv_graphs}}
\hfil
\subfloat[$T_2$]{\includegraphics[width=2.75in]{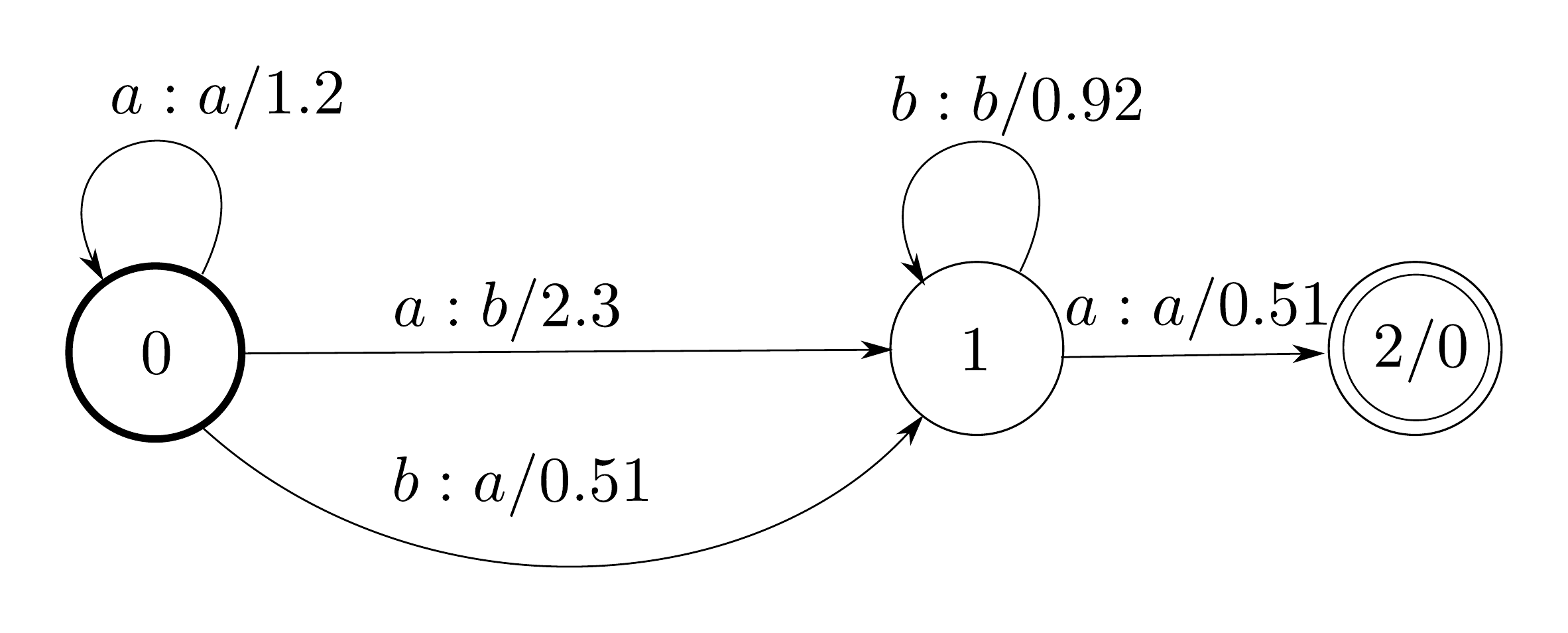}%
\label{fig_graph_product}}
\hfil
\subfloat[$T_1 \circ T_2$]{\includegraphics[width=2.85in]{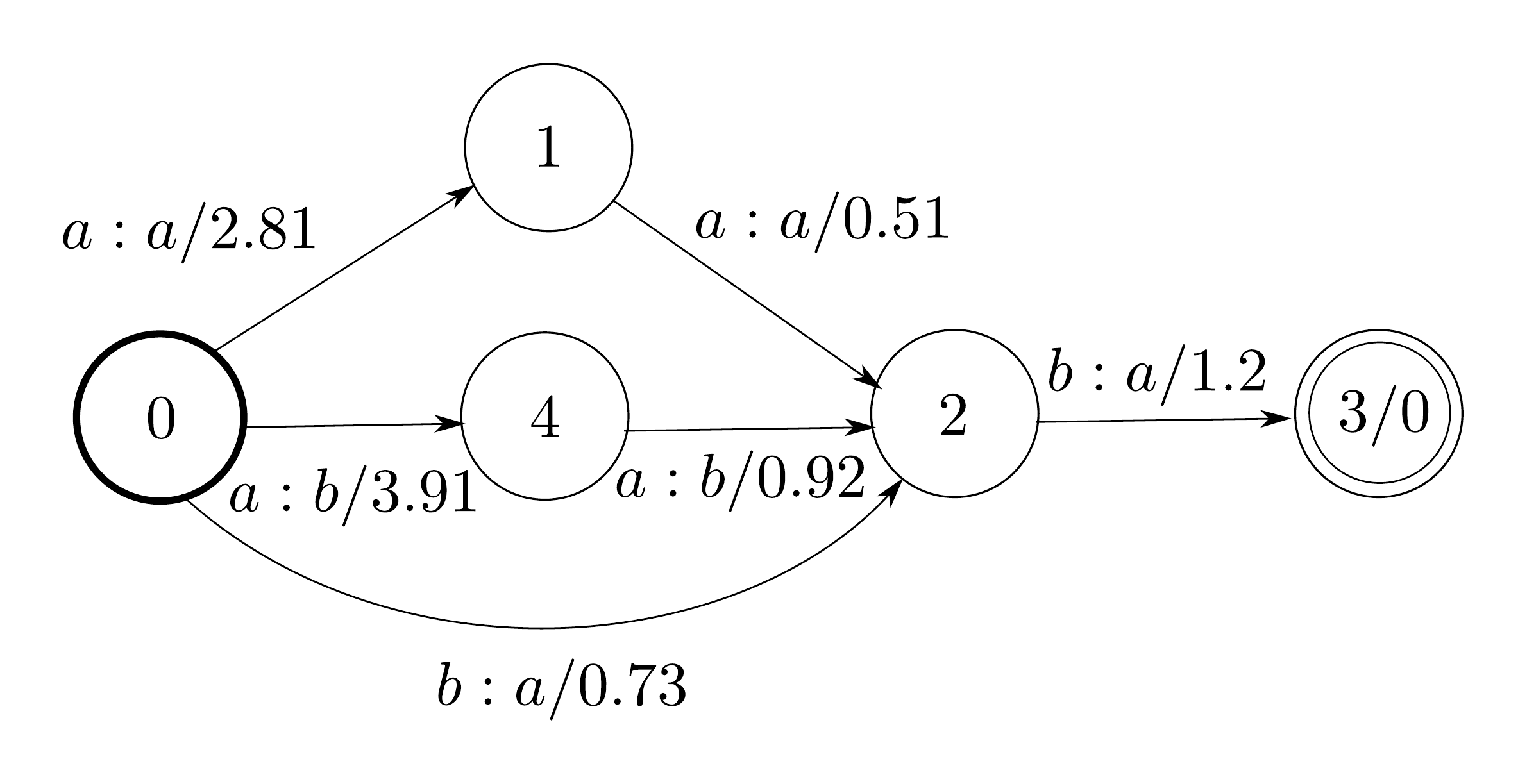}%
\label{fig_composition}}
\caption{Transducer composition.} 
\label{fig:transducer_composition}
\end{figure*}

\section{Relationship to Graph Kernels}
\label{sec:relation_to_graph}
We see an interesting relationship between \emph{random walk graph kernels} and rational kernels in \cite{Vishwanathan:2010:GK:1756006.1859891}. Informally, a random walk graph kernel gives us a measure of similarity between two graphs, by performing multiple random walks on each of them and counting how many of them are similar in some sense. The connection to rational kernels is, when we compose transducers: $T_1 \circ T_2(x,y)= \sum\limits_{z \in \Sigma^*}T_1(x,z)T_2(z,y)$, we are , in a way, counting similar paths by forcing paths from $T_1$ to have the same output sequence $z$, as the input sequence for paths on $T_2$. 

This relationship would be made more concrete in the following sections. Before we can describe random walk graph kernels, we quickly acquaint ourselves with the notation and terminology needed.

\begin{enumerate}
\item A graph $G$ consists of an ordered set of $n$ vertices $V = \{v_1,v_2,...,v_n\}$, and a set of directed edges $E \subset V \times V$. $v_i$ and $v_j$ are termed \emph{neighbors}, denoted by $v_i \sim v_j$, if $(v_i,v_j) \in E$. We do not allow self-loops, that is, $(v_i,v_i) \notin E$ for any $i$.
\item A walk of length $k$ on $G$ is a sequence of indices $i_0,i_1,...i_k$ such that $v_{i_{r-1}} \sim v_{i_r}$ for all $1 \leq r \leq k$. 
\item A graph is said to be undirected if $(v_i,v_j) \in E \iff (v_j,v_i) \in E$.
\item In a \emph{weighted} graph, for every edge $v_i \sim v_j$, we have a \emph{weight} $w_{ij} > 0$ associated with it. If $v_i \not\sim v_j$, $w_{ij}=0$. In an \emph{undirected weighted} graph, we have $w_{ij} = w_{ji}$.
\item The adjacency matrix $\widetilde{A}$ is a $n \times n$ matrix, where $\widetilde{A}_{ij}=w_{ji}$. Note that this definition is the transposed version of what other sources describe as the adjacency matrix, but we would use this in keeping with the notation of \cite{Vishwanathan:2010:GK:1756006.1859891}.
\item The \emph{normalized} adjacency matrix is defined as $A=\widetilde{A}D^{-1}$, where $D$ is a diagonal matrix with $D_{ii}=d_i=\sum\limits_{j}\widetilde{A}_{ij}$. $A$ has the property that each of its columns sums to one, and it can therefore serve as the \emph{transition matrix} for a stochastic process.
\item A random walk on $G$ is a process generating sequences of vertices $v_{i_1},v_{i_2},v_{i_3},...$ according to $P(i_{k+1}|i_1,..., i_k) = A_{i_{k+1},i_k}$,
that is, the probability at $v_{i_k}$ of picking $v_{i_{k+1}}$ next is proportional to the weight of the edge $(v_{i_k},v_{i_{k+1}})$. The $t^{th}$ power of $A$ thus describes t-length walks, that is, $(A^t)_{ij}$ is the probability of a transition from vertex $v_j$ to vertex $v_i$ via a walk of length $t$.
\item We can define an initial probability distribution of the vertices, $p_0$, as well as a stopping probability distribution over vertices, $q$. Consider the quantity $p_t = A^tp_0$, where $p_t \in \mathbb{R}^{n \times 1}$. The $j^{th}$ component of $p_t$ denotes the probability of finishing a t-length walk at vertex $v_j$. The scalar $q^Tp_t$ denotes the overall probability of stopping after $t$ steps. The quantities $p_0$ and $q$ will show up while defining the random-walk kernel. These are typically used to embed prior knowledge while designing the kernel.
\item Let $\mathcal{X}$ be a set of labels, including a special label $\zeta$. If a graph $G$ is edge-labeled then it may be associated with a \emph{label matrix} $X \in \mathcal{X}^{n \times n}$, where $X_{ij}$ is the label of the edge $v_j \sim v_i$. $X_{ij}=\zeta$ if $(v_j, v_i) \notin E$.
\item Let $\mathcal{H}$ be the \emph{Reproducing Kernel Hilbert Space} (RKHS) induced by the positive-definite and symmetric (PDS) kernel $\kappa: \mathcal{X} \times \mathcal{X} \rightarrow \mathbb{R}, \kappa(x_1,x_2)=\langle \phi(x_1), \phi(x_2)\rangle_\mathcal{H}$. Let $\phi : \mathcal{X} \rightarrow \mathcal{H}$ denote the corresponding feature map, which we assume maps $\zeta$ to the zero element of $\mathcal{H}$ . We use $\Phi(X)$ to denote the feature matrix of $G$, where  $[\Phi(X)]_{ij} = \phi(X_{ij})$.

\end{enumerate}

We also define the following operations.
\begin{definition}
\label{defn:kronecker}
Given real matrices $A \in \mathbb{R}^{n \times m}$ and $B \in \mathbb{R}^{p \times q}$, the Kronecker product $A \otimes B \in \mathbb{R}^{np \times mq}$ and column-stacking operator $vec(A) \in \mathbb{R}^{nm}$ are defined as:

\begin{align*}
&A \otimes B = \begin{bmatrix}
    A_{11}B       & A_{12}B  \dots & A_{1m}B \\
    \hdotsfor{3} \\
    A_{n1}B       & A_{n2}B  \dots & A_{nm}B
\end{bmatrix};
vec(A) = \begin{bmatrix}
A_{*1}\\
\vdots\\
A_{*m}
\end{bmatrix},\\
&\text{where} A_{*j} \text{ denotes the } j^{th} \text{ column of } A. 
\end{align*}
\end{definition}

\begin{definition}
Let $A \in \mathcal{X}^{n \times m}$ and $B \in \mathcal{X}^{p \times q}$. The Kronecker product $\Phi(A) \otimes \Phi(B) \in \mathbb{R}^{np \times mq}$ is defined as:
\begin{equation}
[\Phi(A)\otimes \Phi(B)]_{(i-1)p+k,( j-1)q+l} =  \langle \phi(A_{ij}),\phi(B_{kl}) \rangle_\mathcal{H}
\end{equation}
\end{definition}

\subsubsection{Random Walk Graph Kernels}
We are now ready to define random walk graph kernels. As mentioned before, the general idea is to perform random walks on graphs $G$ and $G'$ and count paths that are similar in some sense. Instead of performing random walks on two different graphs and matching paths, we define a \emph{direct product}  graph $G_\times$, which makes the procedure convenient.

Given two graphs $G(V,E)$ and $G'(V',E')$, their direct product $G_\times$ is a graph with vertex set:
\begin{equation*}
V_\times = {(v_i,v'_r) : v_i \in V, v'_r \in V'}, 
\end{equation*}
and edge set:
\begin{equation*}
E_\times = {((v_i,v'_r), (v_j,v'_s)) : (v_i,v_j) \in E \text{ and } (v'_r,v'_s) \in E'}
\end{equation*}
Observe that two nodes in $G_\times$ are neighbors only if the corresponding vertices are neighbors in $G$ and $G'$. Fig \ref{fig:graph_product} shows an example of a product graph.

\begin{figure*}[!t]
\centering
\subfloat[Individual graphs $G_1$ and $G_2$]{\includegraphics[width=2.25 in]{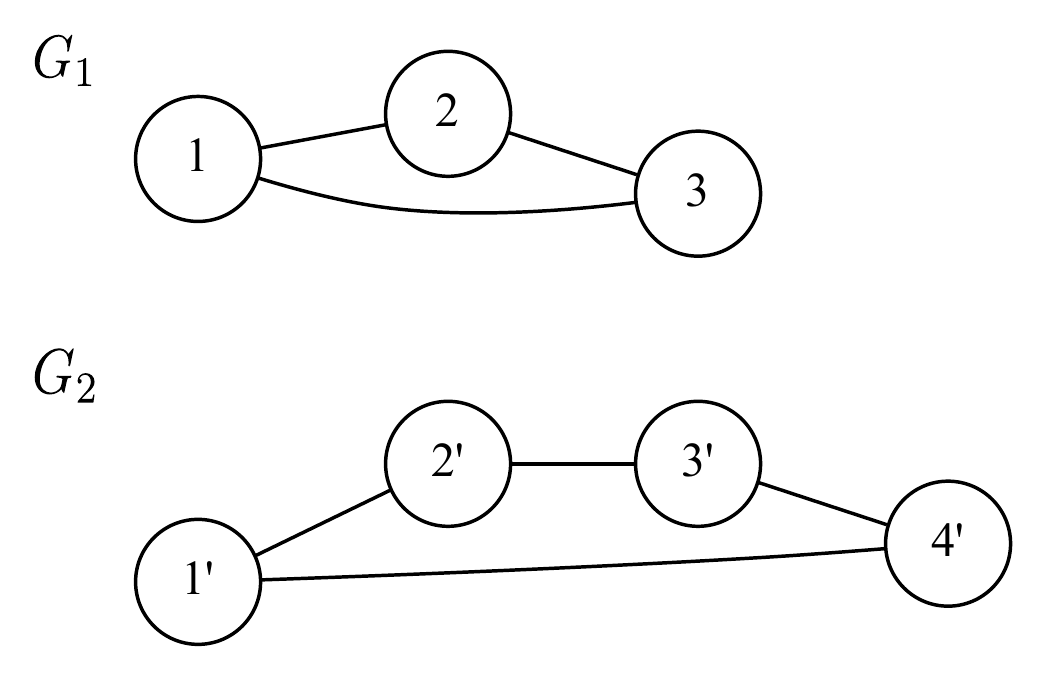}%
\label{fig_indv_graphs}}
%\hfil
\subfloat[Graph product, $G_1 \times G_2$]{\includegraphics[width=2.75in]{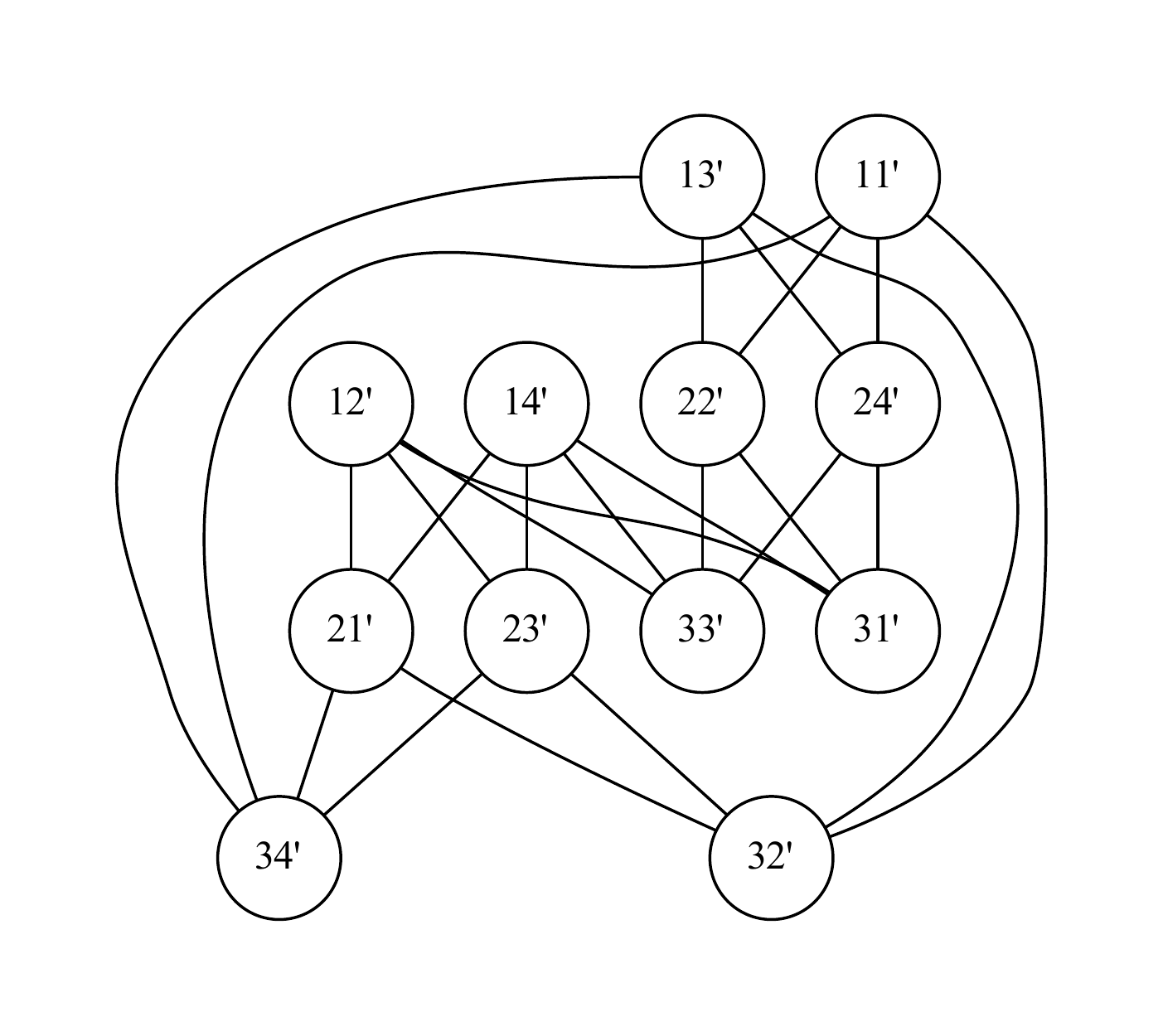}%
\label{fig_graph_product}}
\caption{The product of two graphs from (a) is shown in (b).} %$(1,2')\sim (2,1') \text{ in } G_1 \times G_2$}, since $(1,2) \text{ and } (2',1')$ are edges in $G_1\text{ and } G_2$ respctively.
\label{fig:graph_product}
\end{figure*}

If $\widetilde{A}$ and $\widetilde{A'}$ are the adjacency matrices for $G$ and $G'$ respectively, the adjacency matrix for $G_\times$ is $\widetilde{A_\times} = \widetilde{A} \otimes \widetilde{A'}$. Similarly, $A_\times = A \otimes A'$. We also define $p_\times = p \otimes p'$ and $q_\times = q \otimes q'$, which respectively give us the starting and stopping probabilities on $G_\times$.

Performing a random walk on $G_\times$ is equivalent to performing a random walk on the individual graphs. 

We associate the following weight matrix $W_\times \in \mathbb{R}^{nn' \times nn'}$ with $G_\times$:
\begin{equation}
W_\times = \Phi(X) \otimes \Phi(X')
\end{equation}
Here $X, X'$ are the label matrices for $G, G'$ respectively.

What does $W_\times$ give us? Consider $A_\times$ first. We observed above that performing a walk on $G_\times$ is equivalent to performing simultaneous walks on $G$ and $G'$. Thus, $A^t_\times$ represents t-length random walks done simultaneously on the graphs; specifically, the $((i-1)n' + r, ( j-1)n' + s)^{th}$ entry of $A^t_\times$ represents the probability of simultaneous length $t$ random walks on $G$ (starting from vertex $v_j$ and ending in vertex $v_i$) and $G'$ (starting from vertex $v'_s$ and ending in vertex $v'_r$). In the same way, $W^t_\times$ tells us how similar, based on the kernel function $\kappa$, are random $t$ length walks performed simultaneously on $G$ and $G'$. Using $p_\times$ and $q_\times$ we can compute $q_\times^T W_\times^t p_\times$, which is the expected similarity between simultaneous length $t$ random walks on $G$ and $G'$.

A specific representation for labels which we we would use in the next section is as follows: without loss of generality, assume the labels come from a finite set $\{1,2,...,d\}$, and we can let $\mathcal{H}$ be $\mathbb{R}^d$ endowed with the usual inner product. For each edge $(v_j,v_i) \in E$ we set $\phi(X_{ij}) = \mathbf{e_l} /d_i$ if the edge $(v_j,v_i)$ is labeled $l$; all other entries of $\Phi(X)$ are 0. Here, $\mathbf{e_l}$ is a vector of length $d$ with the $l^{th}$ entry set to $1$, and the rest set to $0$. Thus, $W_\times$ has a non-zero entry iff an edge exists in the direct product graph \emph{and} the corresponding edges in $G$ and $G'$ have the same label. Let $^l\!A$ denote the normalized adjacency matrix of the graph filtered by the label $l$, that is, $^l\!A_{ij} = A_{ij}$ if $X_{ij} = l$, and zero otherwise. It is easy to derive the following relationship:
\begin{equation}
W_\times = \sum\limits_{l=1}^{d} {}^l\!A \otimes {}^l\!A' \label{eqn:weight_mat_filtered_adj}
\end{equation}

One obvious way to define a kernel that computes the similarity between $G$ and $G'$ is to sum up $q_\times^T W_\times^t p_\times$ for all values of $t$. But, this has the potential shortcoming that the sum might not converge and the kernel would not be completely defined. This is addressed by modifying the expression only slightly and defining the kernel function $k$ as:
\begin{equation}
\label{eqn:graph_kernel}
k(G,G') = \sum\limits_{t=0}^\infty \mu(t) q_\times^T W_\times^t p_\times
\end{equation}
$\mu(t)$ provides a lot of flexibility during kernel design by allowing one to emphasize or de-emphasize paths of certain lengths, encode prior information, etc. Significantly, it can be shown that if $\mu(t)$ are chosen in a way that $k(G,G')$ converges, it is also PDS (see Theorem \ref{theorem:graph_kernel_PDS}, Appendix).

\subsubsection{Relationship with Rational Kernels}
We now show that rational kernels are random walk graph kernels in a specific setting. We  revisit (see eqn (\ref{eqn:kernel_automata_transducer})) %TODO: make sure this is mentioned earlier to merit a ``revisit''
the rational kernel definition for weighted automata $S,U$ and transducer $T$:
\begin{align}
\label{eqn:graph_kernels_automata}
k(S,U) &= \Psi\bigg(\bigoplus\limits_{\alpha,\beta} [\![ S ]\!] (\alpha) \otimes [\![ T ]\!] (\alpha, \beta) \otimes [\![ U ]\!] (\beta) \bigg) \nonumber \\
&=\sum\limits_{\alpha,\beta} \Psi([\![ S \circ T \circ U ]\!](\alpha,\beta))
\end{align}
Here, we use $(\mathbb{R},+,\cdot,0,1)$ as our semiring and $\Psi$ is the identity morphism. In a small deviation from our previous convention, for a transducer with $n$ states, we let $p \in \mathbb{R}^n$ denote \emph{initial weights}\footnote{similar to final weights, these are to be multiplied at the beginning of a path}, and $q \in \mathbb{R}^n$ denote final weights. We set the transducer $T$ to accept sequence pairs where the input sequence is the same as the output sequence. We further assume the weight assigned to such a sequence pair by $T$ is a function, represented by $\mu$, of sequence length i.e. $T(\alpha,\alpha)=\mu(|\alpha|) \geq 0$.
Rewriting eqn (\ref{eqn:graph_kernels_automata}):
\begin{align}
k(S,U) &= \sum\limits_{\alpha}\mu(|\alpha|)[\![ S \circ U]\!](\alpha) \nonumber \\
&= \sum\limits_{k}\mu(k)\bigg(\sum\limits_{a_1,a_2,...,a_k}[\![ S  \circ U ]\!](a_1,a_2,...,a_k) \bigg) \label{eqn:before_tensor}
\end{align}
To simplify further steps we represent the transition function the transition matrix
of a weighted automaton as a three-dimensional tensor in $\mathbb{R}^{n \times |\Sigma|\times n}$. We will use the shorthand $H_a$ to denote the $n\times n$ slice $H_{*a*}$ of the transition tensor, which represents all valid transitions on the symbol $a$. The transition tensors of $S$ and $U$ are represented by $H$ and $H'$ respectively. Rewriting eqn (\ref{eqn:before_tensor}):
\begin{align}
k(S,U) &= \sum\limits_k \mu(k) \bigg(\sum\limits_{a_1,a_2,..., a_k \in \Sigma^k}(q \otimes q')^T(H_{a_1} \otimes H'_{a_1})...(H_{a_k} \otimes H'_{a_k})(p \otimes p') \nonumber \\
&= \sum\limits_k \mu(k) (q \otimes q')^T \bigg(\sum\limits_{a_1,a_2,..., a_k \in \Sigma^k}(H_{a_1} \otimes H'_{a_1})...(H_{a_k} \otimes H'_{a_k}) \bigg) (p \otimes p') \nonumber \\
&= \sum\limits_k \mu(k) (q \otimes q')^T \bigg(\sum\limits_{a \in \Sigma}(H_a \otimes H'_a)\bigg)...\bigg(\sum\limits_{a \in \Sigma}(H_a \otimes H'_a)\bigg) (p \otimes p') \nonumber \\
&= \sum\limits_k \mu(k) (q \otimes q')^T \bigg(\sum\limits_{a \in \Sigma}(H_a \otimes H'_a)\bigg)^k (p \otimes p') \label{eqn:graph_rational_last_but_one}
\end{align} 
We observe that $H_a$ (resp. $H'_a$) is analogous to the label-filtered adjacency matrix $^a\!A$ (resp. $^a\!A'$) of a graph $G$ (resp. $G'$). The weight matrix of the direct product of $G$ and $G'$ is then given by $H_\times = \sum\limits_{a}H_a \otimes H'_a$ (see eqn (\ref{eqn:weight_mat_filtered_adj})). We set $p_\times = p \otimes p'$ and $q_\times = q \otimes q'$ and rewrite eqn (\ref{eqn:graph_rational_last_but_one}) as:
\begin{equation}
k(G, G') = \sum\limits_k \mu(k)q_\times^T H^k_\times p_\times
\end{equation}
Hence, the above rational kernel is equivalent to a random walk graph kernel with $W_\times = H_\times$.

\section{Other Applications}
This section looks at some other contexts where rational kernels have been used.
\subsection{Classifying Formal Languages using Rational Kernels}
\label{sec:ratk_classify_pt}
An extremely interesting relationship between automata theory and rational kernels is studied in \cite{Kontorovich:2008:KML:1411847.1411928}. It explores the notion of \emph{learning} a \emph{formal language} from positive and negative examples. A positive (negative) example is defined as a string contained (not contained) in the language. The paper proves that strings may be mapped to a high-dimensional feature space in a way that a hyperplane separating out the members of a \emph{Piecewise Testable} (PT) language, from ones not in the language, can be identified. A rational kernel can be used to obtain this mapping; and used with a SVM, a membership test for a PT language may be learned.

The proof uses the following sequence of logical steps:
\begin{enumerate}
\item We start with a convenient characterization of PT languages.
\item This characterization leads to a test of membership for the language using a particular kind of Boolean function known as a \emph{decision list}.
\item We point out that decision lists indicate separability.
\item The decision list may be encoded as a kernel, which in this case, is shown to be a rational kernel. 
\end{enumerate}
We now elaborate on these steps.
\subsubsection{Piecewise Testable Languages}
PT languages were first studied in \cite{Simon:1975:PTE:646589.697341} by Imre Simon. A language $L$ over the alphabet $\Sigma$ is said to be \emph{n-piecewise-testable} for $n \in \mathbb{N}$, if when $u, v \in \Sigma^*$ have the same subsequences of length at most $n$, then either $u,v \in L$ or $u,v \notin L$. Thus, membership to $L$ can be tested by comparing \emph{pieces}/subsequences.

We will use the symbol $\sqsubseteq$ to indicate the `subsequence' relationship i.e. $x \sqsubseteq y$ where $x,y \in \Sigma^*$, indicates the following: 
\begin{equation}
x \in \{A_0a_1A_1a_2... A_{l-1}a_lA_l| a_1a_2...a_l = x, a_i \in \Sigma, A_j \in \Sigma^*\}
\end{equation}
$x \sqsubseteq_n y$ denotes $x \sqsubseteq y$ such that $|x| \leq n$.
We write $u \sim_n v$ to denote the \emph{relation} that $u$ and $v$ have the same subsequences of length at most $n$. We can rewrite the above criteria for PT languages as:
\begin{equation}
\label{eqn:basic_criteria_PT}
u \sim_n v \implies (u \in L \iff v \in L)
\end{equation}
Note that $\sim_n$ has the following properties:
\begin{enumerate}
\item It is \emph{right congruent} i.e. $u \sim_n v \implies \forall z \in \Sigma^*, uz \sim_n vz$. This is obvious since the additional subsequences $y$ (of length at most $n$) produced due to concatenation with $z$ are either of the form (1) $y \sqsubseteq z, |y| \leq n$ or (2) $xy, \text{ where } x \sqsubseteq u \text{ or } x \sqsubseteq v, |x| \leq n, \text{ and } y \sqsubseteq z$. Since $u \text{ and } v$ originally had the same subsequences of length at most $n$, they yield the same subsequences in the latter case.
\item It \emph{refines} L i.e. $\forall u,v \in \Sigma^*, u \sim_n v \implies (u \in L \iff v \in L)$. This holds by definition for the PT language $L$.
\item It is an \emph{equivalence} relation. It is easy to see that it is \emph{reflexive}, \emph{symmetric} and \emph{transitive}.
\end{enumerate}
We define the \emph{shuffle ideal} of $x$, denoted by $\Psi(x)$, as the set of all strings $u \in \Sigma^*$, such that $x \sqsubseteq u$:
\begin{equation}
\Psi(x)=\{u \in \Sigma^*|x \sqsubseteq u\}
\end{equation} 

We now prove an alternate characterization of PT languages that is convenient for us. The proof presented here is borrowed from \cite{Kozen:1997:AC:549365}, \cite{Klima:2011:PTL:2645929.2646344}.

\begin{lemma}
\label{lemma:PT_boolean}
A language is PT iff it is a finite Boolean combination of shuffle ideals.
\end{lemma}

\begin{proof}
Lets define the relation $\sim_L$ as:
\begin{equation}
x \sim_L y \overset{def}{\iff} \forall z \in \Sigma^* (xz \in L \iff yz \in L)
\end{equation}

Note that $\sim_L$ is an equivalence relation. 

Clearly $\sim_n \subseteq \sim_L$, since if $x \sim_n y$ but $x \nsim_L y$, then $\exists z, xz \in L \text{ but } yz \notin L$. But $\sim_n$ is right congruent, so we have $\forall z \in \Sigma^*, xz \sim_n yz$. Since $\sim_n$ also refines $L$, $\forall z \in \Sigma^*$, either $xz, yz \in L$ or $xz, yz \notin L$. This is a contradiction.

Since $\sim_n \subseteq \sim_L$ and both are equivalence relations,  the equivalence classes induced by $\sim_L$, denoted by $L/\sim_L$, are composed of equivalence classes from $L/\sim_n$. Thus we need to only prove that each equivalence class in $L/\sim_n$ can be expressed as a finite Boolean combination of shuffle ideals; the language $L$ itself is a union of (some) classes from $L/\sim_L$, each of which is a union of classes from $L/\sim_n$.

Consider $v_{\sim_n} = \{w \in \Sigma^*| v \sim_n w\}$. 
\begin{equation} \label{eqn:partition_intersect}
v_{\sim_n} = \bigcap_{u \sqsubseteq_n v} \Psi(u)  \cap \bigcap_{u \not\sqsubseteq_n v} \Psi(u)^c
\end{equation} 
Here, $c$ denotes \emph{complement}:  $\Psi(u)^c =  \Sigma^* - \Psi(u)$.The first component in eqn (\ref{eqn:partition_intersect}) finds $w \in \Sigma^*$, such that $w$ at least has the same subsequences, of length at most $n$, as $v$. But $w$ may also have subsequences of length $\leq n$ that are not part of $v$; and to satisfy $\sim_n$ the set of subsequences of length $\leq n$ must \emph{exactly} be the same for $w$ and $v$. The second component weeds out $w$ whose subsequences are not part of $v$, by identifying such $w$ and only retaining their complement. Hence, $L$ can be written as a finite union of finite intersections of shuffle ideals.

We do not prove the other direction in the interest of brevity. Please refer \cite{Klima:2011:PTL:2645929.2646344} for details. 
\end{proof}

\subsubsection{PT languages as decision lists}

We start with the notion of \emph{decisive strings}, since they have an important role to play in the construction of decision lists for PT languages. 

We call $u \in \Sigma^*$ a \emph{decisive} string for $L$, if $\Psi(u) \subseteq L$ (\emph{positive-decisive}) or $\Psi(u) \subseteq L^c$ (\emph{negative-decisive}).
We prove:
\begin{lemma}
\label{lemma:decisive}
A PT language $L \subseteq \Sigma^*$ has at least one decisive string $u \in \Sigma^*$.
\end{lemma}
\begin{proof}
Any shuffle ideal $\Psi(u)$ trivially has $u$ as a positive decisive string. Since a PT language $L$ is a finite Boolean combination of shuffle ideals, all we need to prove is the existence of decisive strings is guaranteed under Boolean operations i.e. negation, intersection, union.
\begin{enumerate}
\item Negation: Assume $u$ is positive-decisive for $L$, $\Psi(u) \subseteq L$. Hence, we have a negative-decisive string $u$ for the negated language $L^c$.
\item Intersection: Let's assume $u_1$ is decisive for $L_1$ and $u_2$ is decisive for $L_2$. We want to prove that $L_1 \cap L_2$ has at least one decisive string. Consider the following cases:
\begin{enumerate} 
\item $u_1, u_2$ are positive-decisive. Then $\Psi(u_1) \cap \Psi(u_2) \subseteq L_1 \cap L_2$. Since $u_1u_2 \in \Psi(u_1) \cap \Psi(u_2)$, we know $\Psi(u_1) \cap \Psi(u_2) \neq \emptyset$. For any $u \in \Psi(u_1) \cap \Psi(u_2)$, we have $\Psi(u) \subseteq \Psi(u_1) \cap \Psi(u_2)$. By transitivity, we have $\Psi(u) \subseteq  L_1 \cap L_2$. Thus $u$ is positive-decisive for $L_1 \cap L_2$.
\item $u_1, u_2$ are negative-decisive. For any $u \in \Psi(u_1) \cup \Psi(u_2), \Psi(u) \subseteq (L_1 \cap L_2)^c$. This is because $\forall x \in  \Psi(u), u_1 \sqsubseteq x \text{ or } u_2 \sqsubseteq x$, and the intersection $L_1 \cap L_2$ cannot have any string with $u_1 \text{ or } u_2$ as a subsequence. Hence, $u$ is negative-decisive for $L_1 \cap L_2$.
\item $u_1$ is positive-decisive, $u_2$ is negative-decisive. Since $L_2$ does not contain any string that has  subsequence $u_2$, and for any $u \in \Psi(u_2)$ we have $u_2 \sqsubseteq u$, and by extension $\forall x \in \Psi(u), u_2 \sqsubseteq x$, it follows $u$ is negative-decisive for $L_1 \cap L_2$.
\end{enumerate}
\item Union: Again, let's assume $u_1, u_2$ are decisive for $L_1, L_2$ respectively. We want to prove that there is at least one decisive string for $L_1 \cup L_2$. Consider the following cases:
\begin{enumerate} 
\item $u_1, u_2$ are positive-decisive. For any string $u \in \Psi(u_1) \cup \Psi(u_2)$, we have $\Psi(u) \subseteq L_1 \cup L_2$, since $u$, and by extension any $x\in \Psi(u)$, has $u_1$ or $u_2$ as a subsequence and such strings are contained in $L_1$ or  $L_2$. Hence $u$ is positive-decisive for $L_1 \cup L_2$.
\item $u_1, u_2$ are negative-decisive. Note that since $u_1u_2 \in \Psi(u_1) \cap \Psi(u_2)$, we have $\Psi(u_1) \cap \Psi(u_2) \neq \emptyset$. Let $u \in \Psi(u_1) \cap \Psi(u_2)$. Clearly, $u$ is negative-decisive for $L_1 \cup L_2$, since $u_1 \in \Psi(u),u_2 \in \Psi(u)$, and sequences containing $u_1\text{ or } u_2$ are not part of $L_1 \cup L_2$.
\item $u_1$ is positive-decisive, $u_2$ is negative-decisive. For any $u \in \Psi(u_1)$, we have $u_1 \sqsubseteq u$, and therefore, $\Psi(u) \subseteq \Psi(u_1)$. Since $\Psi(u_1) \subseteq L_1$, we have $\Psi(u) \subseteq L_1 \cup L_2$. 
\end{enumerate}
\end{enumerate}  
\end{proof}
We refer to $u$ as \emph{minimally decisive} for $L$, if $u$ is decisive for $L$ and there is no $v \sqsubseteq u$, such that $v$ is decisive for $L$. We also extend the definition of decisiveness to say $u$ is decisive for $L$ \emph{modulo $V$}, where $V \subseteq \Sigma^*$, if $V \cap \Psi(u) \subseteq L$ or $V \cap \Psi(u) \subseteq L^c$. In the former case, we refer to $u$ as positive-decisive modulo $V$, and in the latter, as negative-decisive modulo $V$. When $V = \Sigma^*$, these definitions coincide with out normal definitions of decisiveness.

We state the following theorem which would be used in later results. However, we do not provide a proof in the interest of brevity - see \cite{HAINES196994} for details.
\begin{theorem}
\label{theorem:finite}
Let $\Sigma$ be a finite alphabet and $L\subseteq \Sigma^*$ a language containing no two distinct strings $x$ and $y$ such that $x \sqsubseteq y$. Then $L$ is finite.
\end{theorem}
We now state and prove the following results:
\begin{lemma}
\label{lemma:finite}
Let $L \subseteq \Sigma^*$ be a PT language and let $D \subseteq \Sigma^*$ be the set of all minimally decisive strings for $L$, then $D$ is a finite set.
\end{lemma}
\begin{proof}
Since $D$ contains minimally decisive strings, it is \emph{subsequence free} i.e. there are no strings $u,v \in D$, such that $u \sqsubseteq v$. Hence, by Theorem \ref{theorem:finite}, $D$ is a finite set.
\end{proof}
\begin{lemma}
\label{lemma:decisive_modulo}
If $V,L \subseteq \Sigma^*$ be PT languages, and $V \neq \emptyset$, then there exists at least one $u \in \Sigma^*$, that is decisive modulo $V$ for $L$. 
\end{lemma}
\begin{proof}
Lemma \ref{lemma:decisive} maybe modified to prove this case.  Since $V$ is a PT language it can be expressed as a finite boolean combination of shuffle ideals. Thus, $V \cap \Psi(s)$ is a  PT language by Lemma \ref{lemma:PT_boolean}, and has at least one decisive string $w$ by Lemma \ref{lemma:decisive}. If $w$ is positive decisive, we have $\Psi(w) \subseteq V \cap \Psi(u)$. Performing an intersection with $V$ on both sides, we get  $V \cap \Psi(w) \subseteq V \cap \Psi(u)$ (an additional $V$ in the RHS is redundant). This implies $V \cap \Psi(w) \subseteq \Psi(u)$, and hence we have a positive decisive string modulo $V$ for $\Psi(u)$. Alternatively, assume $w$ is negative decisive and $\Psi(w) \subseteq (V \cap \Psi(u))^c$. We again perform an intersection with $V$ on both sides, and obtain $V \cap \Psi(w) \subseteq V \cap (V \cap \Psi(u))^c$. This may be simplified as follows:
\begin{align*}
V \cap \Psi(w) &\subseteq V \cap (V \cap \Psi(u))^c \\
&\subseteq V \cap (V^c \cup \Psi(u)^c) \\
&\subseteq (V \cap V^c) \cup (V \cap \Psi(u)^c) \\
&\subseteq \emptyset \cup (V \cap \Psi(u)^c) \\
&\subseteq V \cap \Psi(u)^c\\
\implies V \cap \Psi(w) &\subseteq \Psi(u)^c
\end{align*}
Thus, we have a negative decisive string modulo $V$ for $\Psi(u)$. 

In either case, when $V$ is a PT language, there exists a decisive string modulo $V$ for $\Psi(u)$. This becomes the analogous starting point for the proof of Lemma \ref{lemma:decisive} - 
in the remaining proof,  we replace every instance of $\Psi(x)$ with $(V \cap \Psi(x))$.
%Note that if $\Psi(u) \subseteq L$, then $V \cap \Psi(u) \subseteq L$. Similarly, for the case when $\Psi(u) \subseteq L^c$, we have $V \cap \Psi(u) \subseteq L^c$. It seems if $u$ is decisive for $L$, it is also decisive modulo $V$; the key here really is to establish that $V \cap \Psi(u) \neq \emptyset$.  
%
%Since the shuffle ideal $\Psi(s)$, for some $s \in \Sigma^*$, is trivially a PT language, and $V$ is a PT language, from Lemma  \ref{lemma:decisive} we know that their Boolean combination $\Psi(s) \cap V$ has a decisive string $u$, implying $\Psi(u) \subseteq \Psi(s) \cap V$. It follows $V \cap \Psi(u) \subseteq \Psi(s) \cap V$. The proof in Lemma \ref{lemma:decisive} can be repeated by replacing every instance of $\Psi(x)$ with $(V \cap \Psi(x))$.
\end{proof}
\begin{lemma}
Let $L, V \subseteq \Sigma^*$ be two PT languages and let $D \subseteq \Sigma^*$ be the set of all minimally decisive strings for $L$ modulo $V$, then $D$ is a non-empty finite set.
\end{lemma}
\begin{proof}
This is an application of the extension discussed in Lemma \ref{lemma:decisive_modulo}  to the arguments in Lemma \ref{lemma:finite}.
\end{proof}

The above results enable us to prove that a PT language $L$ may be represented by a specific kind of Boolean function known as a \emph{decision list}. In general, a decision list is a function $f: \mathbb{R}^n \rightarrow \{0,1\}$. $f$ is evaluated based on a further set of functions $f_i: \mathbb{R}^n \rightarrow \{0,1\}$ and scalars $c_i \in \{0,1\}$, $1\leq i \leq r$ in the following manner: \\

\begin{algorithm}[H]
\label{algo:generic_dl}
 \eIf{$f_1(y)==1$}{
  then set $f(y) = c_1$\;
}{
	\eIf{$f_2(y)==1$}{
		then set $f(y) = c_2$\;
	}{
		\ldots \\
		\ldots \\
		\indent \eIf{$f_r(y)==1$} {
			then set $f(y) = c_r$\;
		}{
		set $f(y) = 0$\;
		}
	}
}
\caption{Evaluating the decision list $f(y)$}
\end{algorithm}
The decision list $f$ is succinctly expressed as $f = (f_1, c_1), (f_2, c_2), (f_3, c_3), ..., (f_r, c_r)$. For a detailed discussion around decision lists please refer  \cite{Anthony96thresholdfunctions}.
We now prove the following theorem about PT languages. This is a significant contribution of \cite{Kontorovich:2008:KML:1411847.1411928}.
\begin{theorem}
If $L\subseteq \Sigma^*$ is PT then $L$ is equivalent to some finite decision list $\Delta$ over shuffle ideals.
\end{theorem}
\begin{proof}
Consider languages $V_1, V_2, ...$, all of which are PT languages (Lemma \ref{lemma:dl_subproof_PT}), identified in the following manner:
\begin{enumerate}
\item $V_1 = \Sigma^*$.
\item For $V_i \neq \emptyset$, $V_{i+1}$ is determined thus: Let $D_i \subseteq V_i$ (Lemma \ref{lemma:dl_subproof_DinV}) be the set of minimal decisive strings for $L$ modulo $V_i$. Lemma \ref{} proves all strings contained in $D_i$ are either positive-decisive or negative-decisive modulo $V_i$ for $L$.  In the former case, we define the variable $\sigma_i = 1$, and in the latter case, $\sigma_i = 0$.

We calculate $V_{i+1}$ as
\begin{equation}
\label{eqn:dl_next_v}
V_{i+1} = V_i - \Psi(D_i), \text{ where } \Psi(D_i) = \bigcup_{s \in D_i} \Psi(s)
\end{equation}
\end{enumerate}
To be represented as a decision list we need to show the above process terminates. But before that we explain some of the implicit assumptions we have made.

\begin{lemma}
\label{lemma:dl_subproof_PT}
$V_1, V_2, ...$ are PT languages.
\end{lemma}
\begin{proof}
This is proved by induction. We start with $V_1=\Sigma^*$, which is a PT language (by eqn (\ref{eqn:basic_criteria_PT})). This forms our base case. Assume $V_i$ is a PT language. $\Psi(D_i) = \bigcup_{s \in D_i} \Psi(s)$ is also a PT language since its a finite boolean combination of shuffle ideals. Since the set difference operation can be rewritten as a boolean combination e.g. $A-B=A \cap B^c$, $V_{i+1}$, as given by eqn (\ref{eqn:dl_next_v}), is a PT language.
\end{proof}

\begin{lemma}
If $V_i \neq \emptyset$, then  $D_i \neq \emptyset$.
\end{lemma}
\begin{proof}
Since each $V_i$ is a PT language, as long as $V_i \neq \emptyset$, there is at least one decisive string for $L$ modulo $V_i$ by Lemma \ref{lemma:decisive_modulo}. Hence, $V_i \neq \emptyset \implies  D_i \neq \emptyset$.
\end{proof}

\begin{lemma}
\label{lemma:dl_subproof_DinV}
$\forall i, D_i \subseteq V_i$
\end{lemma}
\begin{proof}
From eqn (\ref{eqn:dl_next_v}), it is clear that $V_{i} \subset V_{i-1}$, $V_{i-1} \subset V_{i-2}$, and so on. Each successive $V_i$ is obtained by removing set $\Psi(D_{i-1})$ from $V_{i-1}$ in the previous step. If indeed there were a string $u \in D_{i}$, that is not in $V_{i}$, then $u \in \Psi(D_{i-1}) \cup \Psi(D_{i-2}) \cup \Psi(D_{i-3}) ... \cup \Psi(D_{1})$.  Lets assume $u \in \Psi(D_{i-k}), 0 < k \leq i-1$. But $\Psi(D_{i-k})$ has already been removed in the process to create set $V_i$. Hence $u$ cannot exist. 
\end{proof}

\begin{lemma}
\label{lemma:dl_subproof_only_posneg}
$D_i$ has either only positive decisive strings or negative decisive strings. 
\end{lemma}
\begin{proof}
We provide an informal proof by induction.
 
The claim is easy to prove for $D_1$. If $u \in D_1$ is positive-decisive and $v \in D_i$ is negative-decisive, then we have $\Psi(u) \subseteq L \text{ and } \Psi(v) \subseteq L^c$. Thus, $\Psi(u) \cap \Psi(v)=\emptyset$. But $uv \in \Psi(u) \cap \Psi(v)$, which is a contradiction; thus $D_1$ can only have positive decisive strings or only negative decisive strings. This serves as our base case.

In the case of a general $D_i$, note that strings in $D_i$ get included in this step because of the intersection with $V_i$ i.e. the modulo operation makes them decisive where they originally were not. Essentially, these decisive strings are created in this step. If this were not the case i.e. $V_i$ does not affect the decisive nature of these strings, then these would have been decisive in some earlier step too (if the intersection with $V_i$ don't affect them, intersection with $V_k, k < i$ cannot surely affect them since $V_i \subset V_k, k < i$). Hence, these would have been already eliminated by now, as part of $D_k$, for some $k <i$.

Let's assume we are at step $i+1$, and $D_i$ has only positive decisive strings modulo $V_i$ for $L$. The removal of $\Psi(D_i)$ could have only created negative-decisive strings modulo $V_{i+1}$ for $L$. This is because the only difference between $V_{i+1}$ and $V_{i}$ is a missing set of strings $\in L$. Some strings in $\Sigma^*$, whose shuffle ideals, in the previous step overlapped with both strings in and outside $L$, can now have their shuffle ideals entirely outside $L$. No positive-decisive strings could have been created because strings $\notin L$ were not removed.

Given our base case has only positive or negative decisive strings but not both, the above argument establishes the property for a general $D_i$ (modulo $V_i$). 
\end{proof}

We will now show that $V_{N+1} = \emptyset, \text{ for some } N>0$ i.e. the process terminates. Assume the contrary. Then the above process generates an infinte sequence $D_1, D_2, D_3, ...$. Construct an infinte sequence $X = (x_n)_{n \in \mathbb{N}} $, where $x_n \in D_n$. Since, $D_{n+1} \subseteq V_{n+1} \text{ and } V_{n+1} = V_n - \Psi(D_n)$, we have all $x_n$ in $X$ are necessarily distinct. Define a new sequence $Y= (y_n)_{n \in \mathbb{N}} $, where $y_1=x_1$, and $y_{n+1} = x_{\xi(n)}$. The function $\xi:\mathbb{N} \rightarrow \mathbb{N}$ is defined thus:
\begin{align*} 
\xi(n+1) &= \min_{k\in \mathbb{N}}\{y_1, y_2, ..., y_n, x_k\} \text{ is subsequence-free, if such a }k \text{ exists} \\
&= \infty \text{ otherwise }
\end{align*}

Observe that if $\xi(i) =\infty, i \in \mathbb{N}$, then $\xi(j) =\infty, \forall j > i$; this is because $\xi$ doesn't put any restriction on which $x_k$ to pick - if a suitable $k$ cannot be found in the infinite sequence $X$ in a particular step $i$, then it cannot be found in any subsequent step $j$ either. 

We also note that there indeed is a $i$ where $\xi(i) = \infty$; because if this were not the case, then $Y$ would have an infinite number of subsequence-free strings, which is not possible as per Theorem \ref{theorem:finite}.

Thus, for $n=1,2, 3,..., \text{ where }n \in \mathbb{N}$, $\xi(n)$ seems to evaluate to non-infinite natural numbers till a particular point, followed by only the value $\infty$ beyond this point. When $\xi(i) = \infty$ the first time, there is no suitable $x_k$ to be picked, which means every $x_k$ is a subsequence of some string in the set $\{y_1, y_2, ..., y_{i-1}\}$. But since this set is finite in size, the number of subsequences of strings from the set are also finite. Since, we assumed $X$ is infinite, with each $x_k$ distinct, this is a contradiction. Hence, we conclude the above process for determining $V_i$ terminates. 

We define a decision list $\Delta = (D_1, \sigma_1),(D_2, \sigma_2),...,(D_N, \sigma_N)$. For $x \in \Sigma^*$, we evaluate $\Delta(x)$ in the following manner:

\begin{algorithm}[H]
 \eIf{$x \in \Psi(D_1)$}{
  then set $\Delta(x) = \sigma_1$\;
}{
	\eIf{$x \in \Psi(D_2)$}{
		then set $\Delta(x) = \sigma_2$\;
	}{
		\ldots \\
		\ldots \\
		\indent \eIf{$x \in \Psi(D_N)$} {
			then set $\Delta(x) = \sigma_N$\;
		}{
		set $\Delta(x) = \sigma_N$\;
		}
	}
}
\caption{Evaluating $\Delta(x), x \in \Sigma^*$. The final \emph{else} clause is different from Algorithm \ref{algo:generic_dl}.}
\end{algorithm}
Observe that $\Delta$ acts as the \emph{characterestic function} for the PT language $L$ i.e. $ \forall x \in \Sigma^*, \Delta(x) = 1 \iff x \in L$. Why? It is easy to see this graphically. See Fig \ref{fig:DL}\footnote{not part of the original paper; this was added by us.}.
\begin{figure*}[!t]
\centering
\subfloat[$D_1, D_2$ are positive and negative decisive respectively. $\Psi(D_1), \Psi(D_2)$ are represented by dotted regions around $D_1,D_2$ - note that these can  overlap.]{\includegraphics[width=2.0 in]{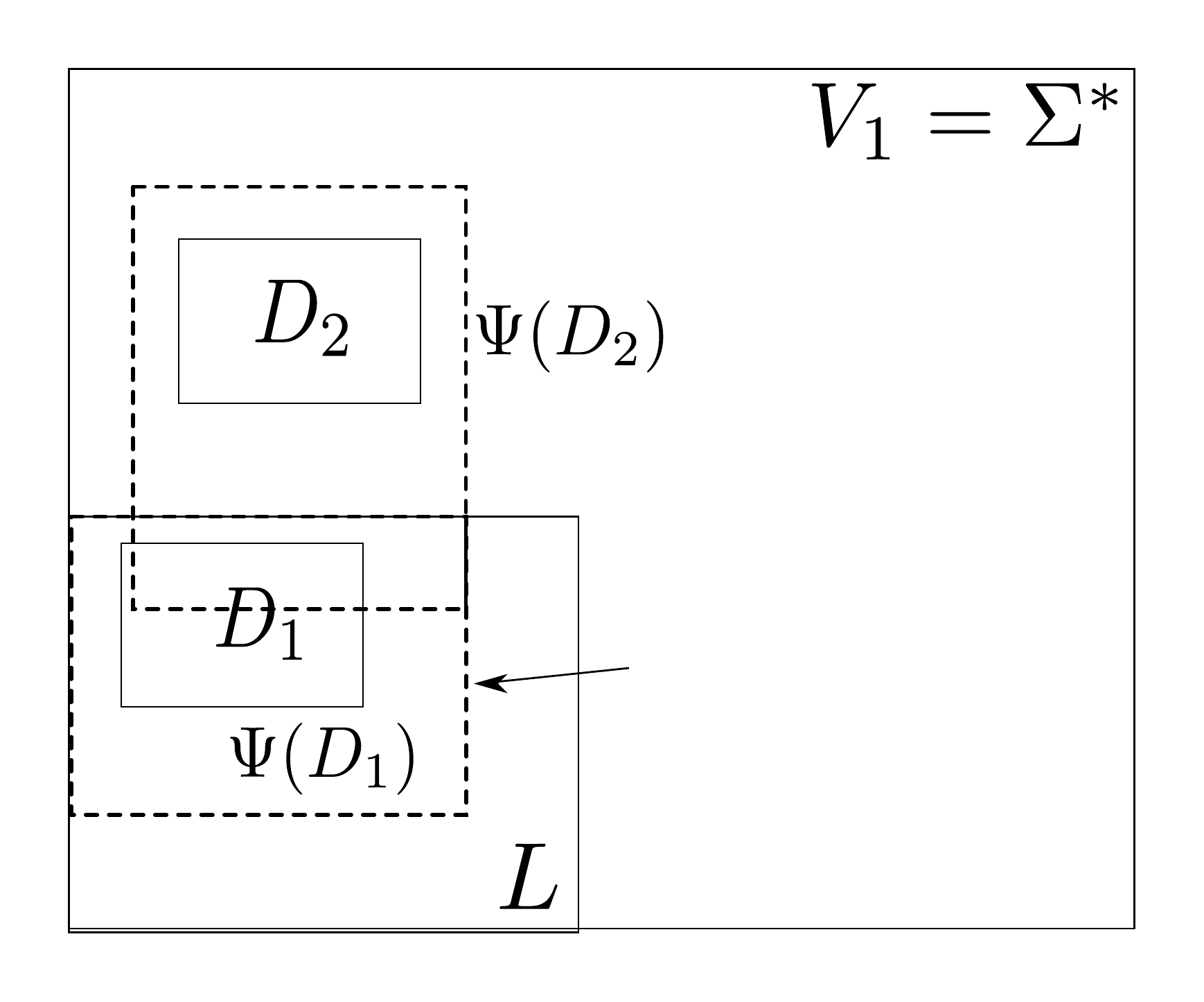}%
\label{fig_first_case}}
\hfil
\subfloat[$\Psi(D_1)$ has been tested for. In the new universe $V_2$, we ignore the grayed out region. The dotted region around $D_2$ in this universe is equivalent to $V_2 \cap \Psi(D_2)$.]{\includegraphics[width=2.0in]{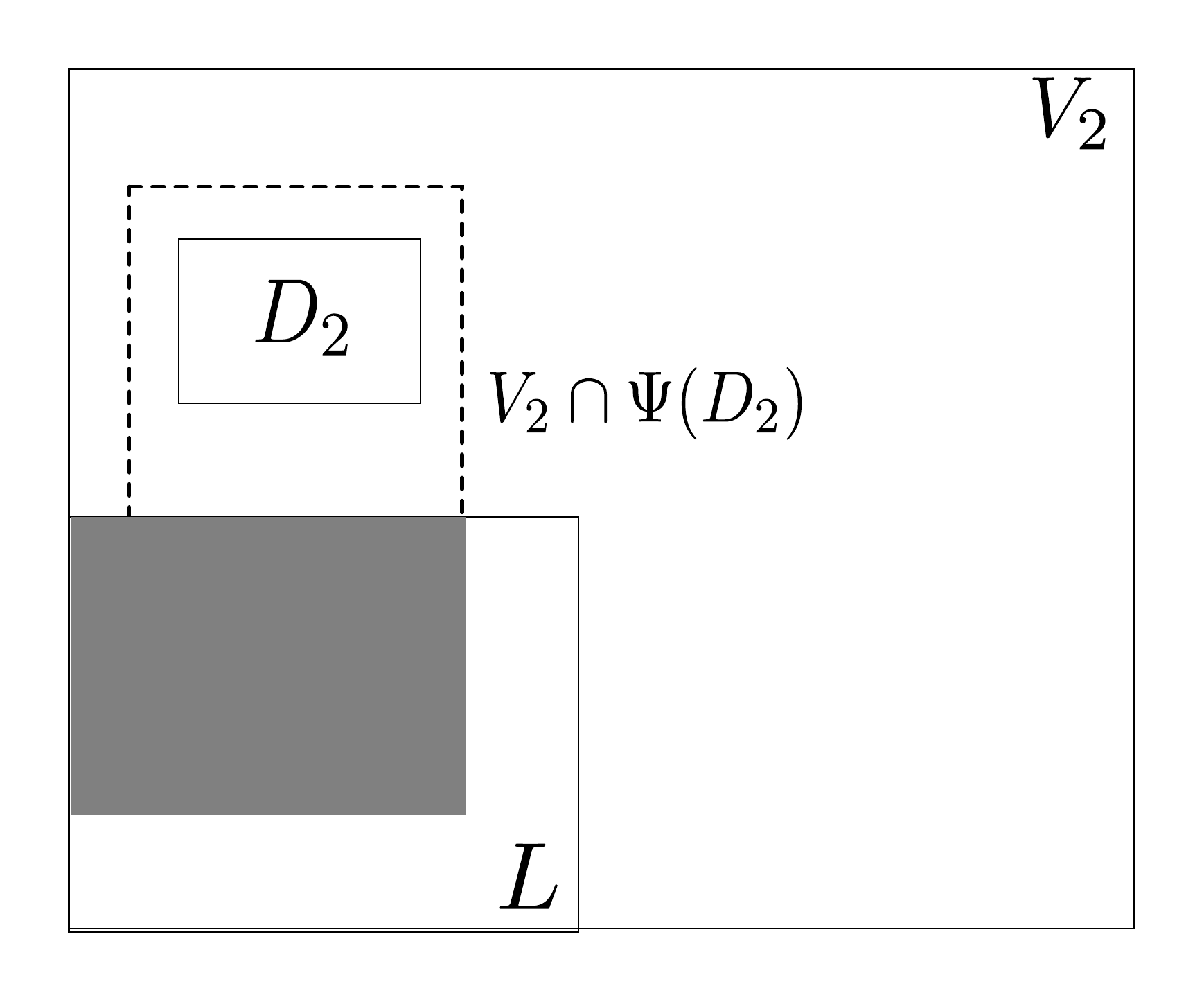}%
\label{fig_second_case}}
\hfil
\subfloat[$V_2 \cap \Psi(D_2)$ has been tested for. The new universe $V_3$ ignores the grayed out region. $D_3, D_4,...$ etc are not shown.]{\includegraphics[width=2.0in]{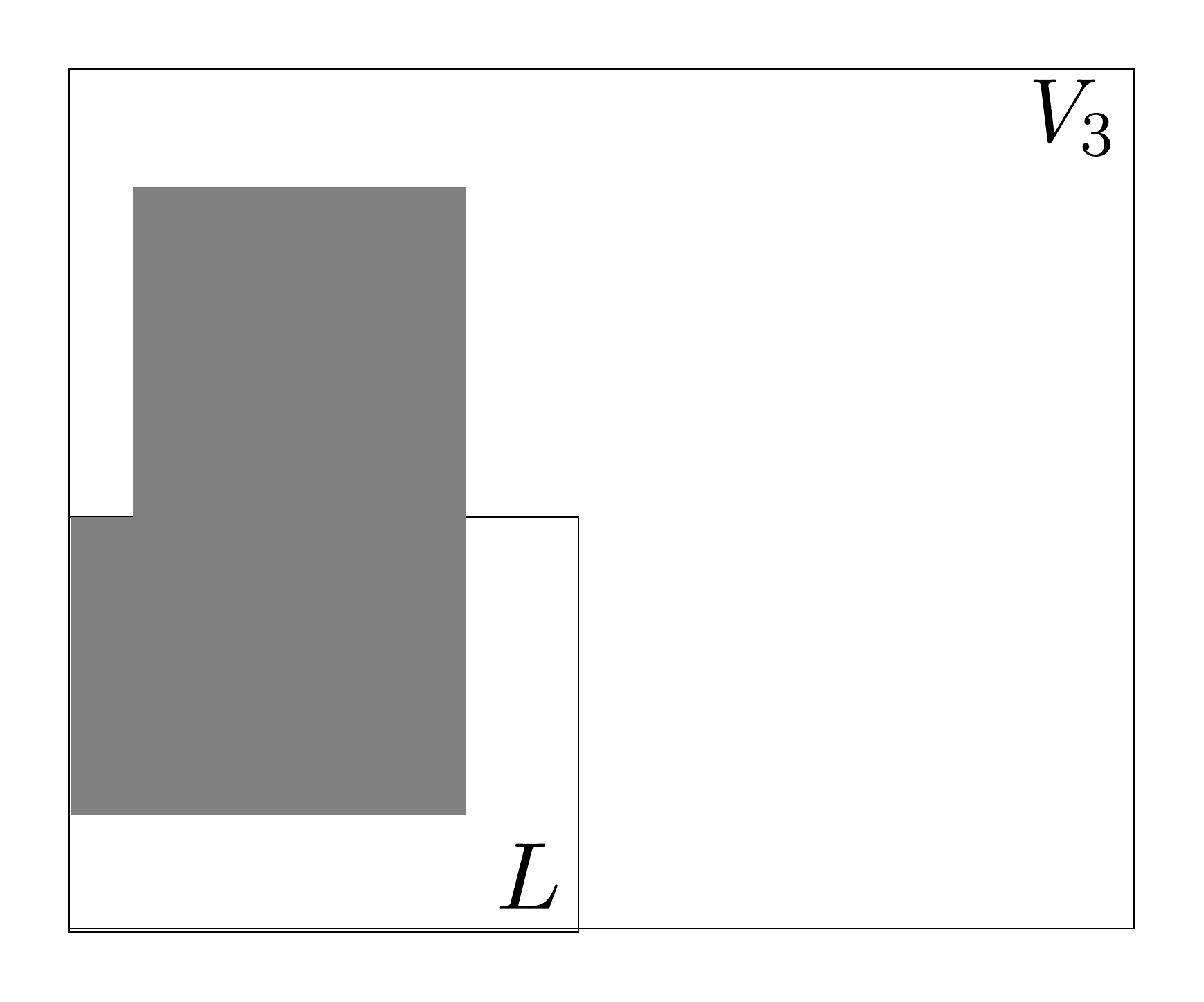}%
\label{fig_second_case}}
\caption{Decision list $\Delta$ to test membership to $L$.}
\label{fig:DL}
\end{figure*}
Each set $R_i=V_i \cap \Psi(D_i)$ lies either completely within $L$ or outside it; the way we calculate $V_i$ ensures this. At a step $i$, we check whether the given string $x \in \Sigma^*$ is a member of $R_i$ - if it is, we decide $x\in L$ or not depending on whether $R_i$ lies within $L$ i.e. whether it is positive or negative decisive. If it is not a member, we remove $R_i$ from consideration, and move on to $D_{i+1}$ and test for $R_{i+1}$. This is why the concept of ``decisive modulo'' is important - it limits the membership test to only the region under consideration. Not surprisingly, \cite{Anthony96thresholdfunctions} colloquially refers to the algorithm as a ``chopping procedure'' because of how we progressively ``chop'' off regions from the initial universe.
\end{proof}

%Essentially, we progressively test whether $x \in \Sigma^*$ belongs to certain specifc regions in $\Sigma^*$, which lie completely inside or outside $L$. These regions don't overlap and exhaust the space $\Sigma^*$.
\subsubsection{Linear Separability}
Quite remarkably, \cite{Kontorovich:2008:KML:1411847.1411928} also proves that with the right feature mapping $\phi(x), x \in \Sigma^*$, the decision list $\Delta$ might be represented as a weight vector $w$ such that the dot-product $\langle w \cdot \phi(x) \rangle > 0$ only when $x \in L$.

The feature mapping we use here is the \emph{subsequence feature mapping} $\phi: \Sigma^* \rightarrow \mathbb{R}^n$, defined as:
\begin{align*}
\phi(x) = &(y_u)_{u \in \Sigma^*}\\
\text{where }& y_u = 1, \text{ if } u \sqsubseteq x,\\
&y_u = 0, \text{ if } u \not\sqsubseteq x
\end{align*}
Let $\Delta= (D_1, \sigma_1), (D_2, \sigma_2), ..., (D_N, \sigma_N)$. We set the weight vector $w = 0$ to begin with, and update its coordinates once for each $D_n$, going in the order $n=N, N-1, ..., 1$:
\begin{align}
\forall u \in D_n,\; w_u &= +\bigg(|\sum\limits_{v\in V^-}w_v| + 1\bigg),\text{ if } \sigma_n=1, \label{eqn:sub_fv_pos}\\
&= -\bigg(|\sum\limits_{v\in V^+}w_v| + 1\bigg),\text{ otherwise } \label{eqn:sub_fv_neg}\\
\text{Here, } V^- = \{v\in \bigcup_{i=n+1}^{N}D_i:  w_v & < 0  \}, \; V^+ = \{v\in\bigcup_{i=n+1}^{N}D_i: w_v > 0  \} \nonumber
\end{align}
$w$, when constructed in this manner, is equivalent to $\Delta$. To see why this is, note the following:
\begin{enumerate}
\item $\forall i,j, i \neq j \text{ and } 1 \leq i,j \leq N, D_i \cap D_j = \emptyset$. This is by construction of $D_i$. As a result, the coordinates $w_u$ modified at an iteration $n$ are not modified during any other iteration.
\item Consider the coordinates $w_u$ modified for $D_i$. The only way these can contribute to $\langle w \cdot \phi(x) \rangle$ is when $\exists u \in D_i,u \sqsubseteq x$. In effect, a non-zero contribution by $D_i$ implies $x \in \Psi(D_i)$. 
\item Consider $D_i, D_j, D_k \text{ where } 1 \leq i < j < k \leq N$. Assume $\sigma_j = 1$ and $x \in V_j \cap \Psi(D_j)$.

\indent Since $x \notin V_i \cap \Psi(D_i)$ for $u \in D_i, u \not\sqsubseteq x$. Analogously, for $w$ we note that the coordinates $w_u$ updated while at $D_i$, result in a $0$ in $\langle w \cdot \phi(x) \rangle$, since $\phi(x)_u=0$.

For $v \in D_k$, it is possible that $v \sqsubseteq x$. Hence all such corresponding coordinates $w_v$ would contribute a non-zero value to the dot-product. Since our test for membership to $L$ is $\langle w \cdot \phi(x) \rangle > 0$, these contributions are fine as long as this inequality holds. This is definitely true when $\sigma_k = 1$, since by construction (eqn (\ref{eqn:sub_fv_pos})) $w_v > 0$. However, when $\sigma_k=-1$, we must ensure that the total negative value of such terms don't outweigh the sum of positive values in the dot-product. A simple way to achieve this is: 
\begin{enumerate}
\item Add all such negative values to obtain the sum  $\sum\limits_{v\in V^-}w_v$. Here $V^- = \{v\in \bigcup_{k > j}^{N}D_k:  w_v  < 0  \}$.
\item Add $1$ to the magnitude to get $t = \bigg(|\sum\limits_{v\in V^-}w_v| + 1\bigg)$.
\item Update all $w_u, u \in D_j$ with this value.
\end{enumerate}
Now, even if there is just one sequence $u \sqsubseteq x, u \in D_j$, its value in the dot-product would be greater than the sum of the negative values contributed by any number of subsequence matches $v, v \in D_k \text{ where } k > j \text{ and } \sigma_k=-1$. A similar argument applies when $\sigma_j=-1 \text{ and } \sigma_k=1$.
\end{enumerate}
\subsubsection{Using Rational Kernels} 
\label{sec:subsequence_kernels_PT_test}
We now show that the subsequence-feature mapping may be obtained by using a rational kernel.

\begin{figure*}[!t]
\centering
\subfloat[$T_0$, counts occurrences.]{\includegraphics[width=1.0 in]{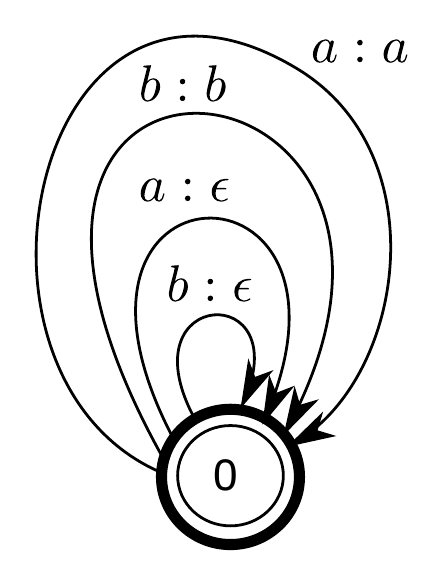}%
\label{fig_single_state}}
\hfil
\subfloat[$R$, paths to be rejected.]{\includegraphics[width=3.0in]{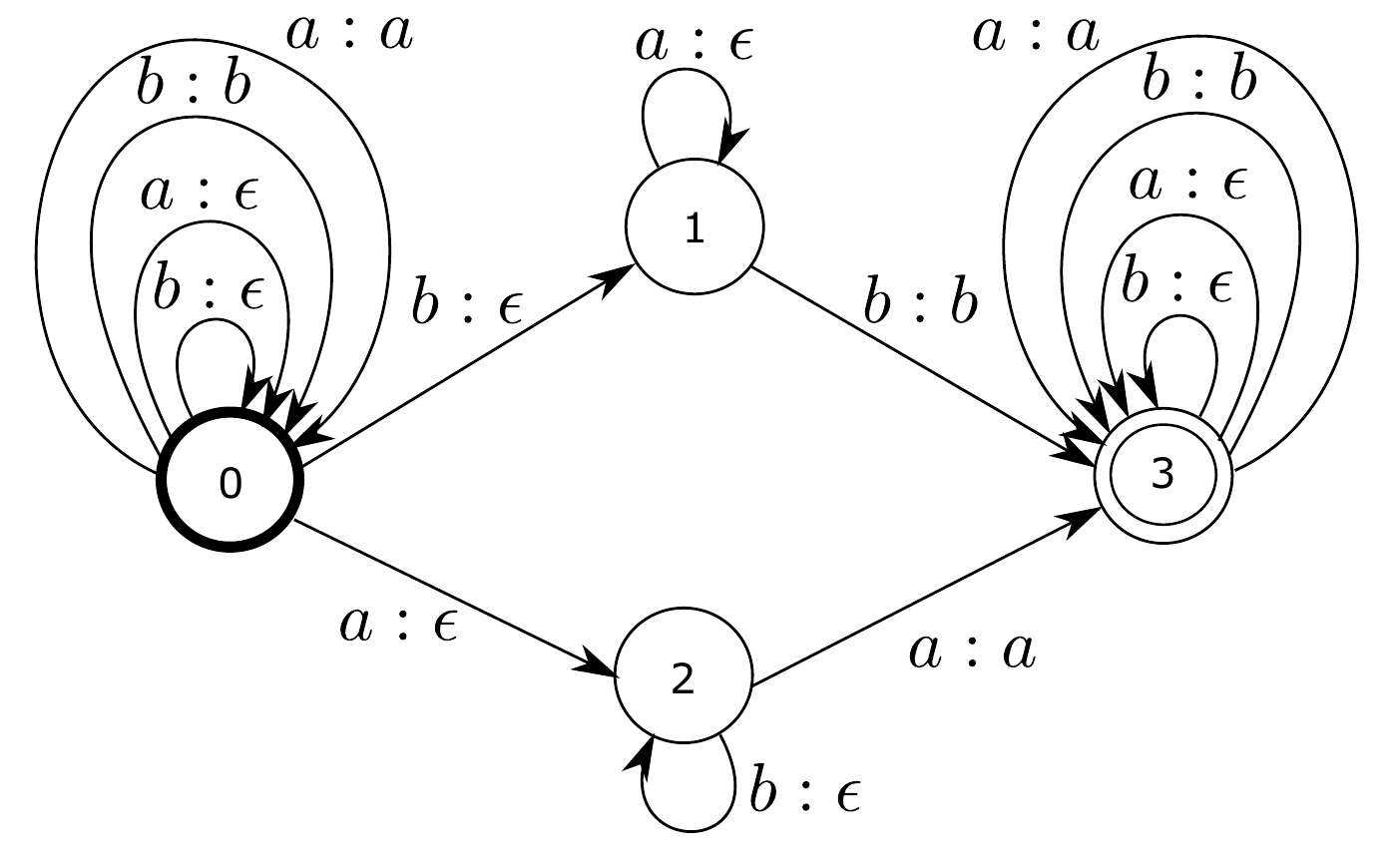}%
\label{fig_reject_paths}}
\hfil

\subfloat[$T = T_0 - R$.]{\includegraphics[width=3.0in]{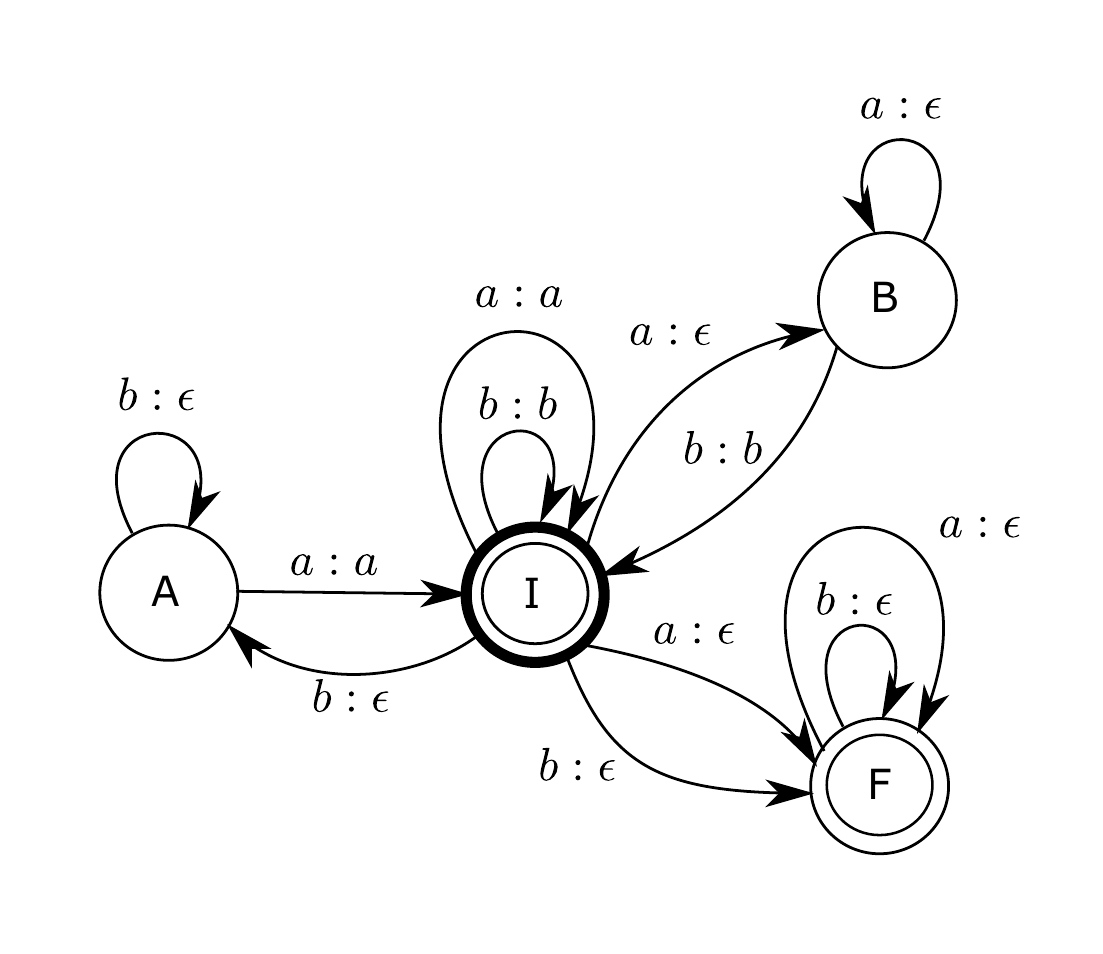}%
\label{fig_subseq_final}}
\caption{Determining a transducer $T$ to indicate presence of subsequences.}
\label{fig:transducer_PT}
\end{figure*}

We are interested in the following kernel function: 
\begin{equation}
\label{eqn:subsequence_kernel}
K(x,y)=\langle \phi(x), \phi(y) \rangle= \sum\limits_{u \in \Sigma^*} [\![ u \sqsubseteq x]\!] [\![  u \sqsubseteq y]\!]
\end{equation}
Here $[\![ P ]\!]$ represents the $\{0,1\}$ truth value of the predicate $P$. For a corresponding rational kernel we need:
\begin{enumerate}
\item The underlying transducer $T$ must match subsequences of all lengths. This might be contrasted with the behavior of the bigram counter where only bigram sequences are matched. Fig \ref{fig_single_state} shows transducer $T_0$ which possesses acceptable paths for all sequences as the output label.
\item \emph{Multiplicity} of subsequences must be ignored. For example, if $u \sqsubseteq x$, and $u$ occurs twice in $x$, we still want $\phi(x)_u=1$. $T_0$ does not satisfy this property since it \emph{counts} occurrences.
\end{enumerate}

It turns out that some simple modifications to $T_0$ can give us the transducer $T$ that meets our needs. We start by inspecting a particular structure in $T_0$ that captures multiplicity. Then we prove that eliminating just this structure gives us $T$.

Consider the section of a path corresponding to the regular expression on symbol pairs: $(a,\epsilon)(b, \epsilon)^* (a,a)$. This consumes $ab^*a$ in the input sequence, and generates $a$ as the output sequence. An equivalent regular expression is $(a,a)(b, \epsilon)^* (a,\epsilon)$. Having both kind of paths in $T$ would double count the occurrence of $a$ in the output, so we can safely drop one - we decide to drop the first expression. Similarly, we drop $(b,\epsilon)(a, \epsilon)^* (b,b)$ and retain $(b, b)(a, \epsilon)^* (b,\epsilon)$. Fig \ref{fig_reject_paths} shows transducer $R_0$ that only has the paths we want to drop. Considering $T_0$ and $R_0$ as \emph{automata} over an alphabet of symbol pairs $(\Sigma \cup \{\epsilon\} \times \Sigma \cup \{\epsilon\}) - \{\epsilon, \epsilon\}$, we can apply automata complementation to obtain $T=T_0-R_0$, as shown in Fig \ref{fig_subseq_final}. Note that $T$ accepts everything $T_0$ does since the expressions dropped were redundant.

We need to prove that if the pair $x,y$ is accepted by $T$, then there is exactly one accepting path. Let's assume we have two accepting paths $\pi_1$ and $\pi_2$ corresponding to $T(x,y)$, and $\pi$ is the the longest
prefix-path shared by $\pi_1$ and $\pi_2$. Consider the following:
\begin{enumerate}
\item $\pi$ cannot end in states $A$ or $B$. The input labels on the outgoing transitions from these states uniquely determine the transition. Thus, the next symbol (there is at least one symbol left since $T(x,y)$ has accepting paths and $A,B$ are not final states) after $\pi$ in $x$, which is identical for both $\pi_1$ and $\pi_2$, would use the same transition. This contradicts the claim that $\pi$ is the longest common prefix path.
\item $\pi$ cannot end in $F$ with symbols left to read for similar reasons - the next input symbol uniquely determines the transition.
\item $\pi$ cannot end in $I$ with non-empty symbols to read. Assume the next input symbol is $a$. If the next output symbol is also $a$, the only possible transition is the loop $(a,a)$ in which case $\pi$ cannot be the longest shared prefix. If the next output symbol is $b$, then the only possibility is going to state $B$ and returning to $I$, which again implies a longer prefix is possible. An analogous argument holds for the case when the next input symbol is $b$ - except that we now we go to state $A$ and return.
\end{enumerate}
What if $\pi$ has length $0$? Since we start at state $I$, arguments similar to 3) apply.
Therefore, there cannot be any more symbols to read when $\pi$ ends at $I$ or $F$. Hence, $\pi_1 = \pi_2$ and we have exactly one path for $T(x,y)$.

The composition $T \circ T^{-1}$ is equivalent to subsequence kernel:
\begin{align*}
T \circ T^{-1}(x,y) &= \sum\limits_{u \in \Sigma^*} T(x,u) T(y, u) \\
&=\sum\limits_{u \sqsubseteq x, u \sqsubseteq y} T(x,u) T(y, u) \\
&=\sum\limits_{u \sqsubseteq x, u \sqsubseteq y} 1 \cdot 1 \\
& = K(x,y)\;\;\;\;\;\;\text{ \emph{-- (ref. eqn (\ref{eqn:subsequence_kernel}))}}
\end{align*}
This proves that we have a rational kernel $T \circ T^{-1}$ that may be used with SVMs to separate out sequences belonging to a PT language $L$.

\subsection{Identifying Metabolic Pathways}
\label{sec:identify_metabolic}
The field of computational biology deals with a fair amount of sequence learning problems, and  as \cite{Roche-Lima2014} shows, rational kernels find use in this context.
We discuss the particular application described in \cite{Roche-Lima2014}, where they are used with \emph{pairwise SVMs}. We begin with a brief description of our objective, then provide an overview of pairwise SVMs (adopting notation and terminology from \cite{Brunner:2012:PSV:2503308.2503316}),  and finally outline their application to our problem.
\subsubsection{Objective}
Metabolic pathways are a series of chemical reactions where the output of one stage feeds as input into the next. Empirically determining such reactions is an active area of research. This paper specifically looks at enzyme-enzyme interactions. If an enzyme $A$ can create a product that can be used as a substrate by another enzyme $B$, we identify $A$ and $B$ with nodes in a graph, with an edge between them to denote they interact. Given a dataset of such known pairwise interactions, our objective is to model and predict whether any two given enzymes interact.

Each enzyme maybe represented by either its \emph{Enzyme Commission} (EC) number (ex. ``5.3.1.9'') or by its \emph{gene nomenclature} (ex. ``YAR071W'').

\subsubsection{Pairwise Kernels and SVMs}
In the case of classification by standard SVMs our training data is comprised of pairs $(x_i, y_i)$ where $x_i \in X$, the set of possible data instances, and $y_i \in C,\text{ the set of class labels}$. In the case of pairwise SVMs our training data is of the form $(x_i, x_j, y_{ij}), \text{ where } x_i, x_j \in X, y_{ij} \in \{+1,-1\}$. If $x_i \text{ and } x_j$ belong to the same class, $y_{ij}=+1$, else $y_{ij}=-1$. Thus, we need not know the label of each data point per se; the information whether or not a pair of points belongs to the same class is used.

The objective is to learn a decision function $f: X \times X \rightarrow \mathbb{R}$ such that $(a,b,+1) \implies f(a,b) > 0$ and $(a,b,-1) \implies f(a,b) < 0$.

The extension of standard SVMs to pairwise SVMs is actually quite straight forward: if $z_{ij}$ denotes the tuple $(x_i, x_j)$, then our training data is of the form $(z_{ij}, +1/-1)$ - this is identical to the case of standard SVMs, with the two classes $\{-1,+1\}$. The key is to redefine the various operations on $z_{ij}$ since its a tuple now, instead of a feature vector. This is done by modifying the kernel function to handle tuples; such kernels are known as \emph{pairwise kernels} $K: (X \times X) \times (X \times X) \rightarrow \mathbb{R}$. It is easy to see that our decision function may be expressed in the familiar form: 
\begin{align}
\label{eqn:dec_fn_pairwise}
f(a,b) =&  \sum\limits_{(x_i,x_j) \in S} \alpha_{ij}y_{ij}K((x_i, x_j),(a,b)) + \gamma \\ 
\text{ where } &\alpha_{ij} \geq 0, \nonumber\\
 &\gamma \in \mathbb{R}, \nonumber\\ 
 &S = \text{set of support vector \emph{pairs}} \nonumber
\end{align}
We want the following properties to be satisfied by pairwise kernels:
\begin{enumerate}
\item They should be \emph{positive definite} and \emph{symmetric} (PDS) i.e. $\forall a,b,c,d \in X, K((a,b), (c,d)) = K((c,d),(a,b))$, and the kernel matrix should be positive semidefinite. This is \emph{Mercer's condition}. It is easy to satisfy this property by ``composing'' pairwise kernels with standard PDS kernels. For ex, the following kernels, where $k: x \times x \rightarrow \mathbb{R}$ are assumed to be PDS, satisfy this property:
\begin{align*}
K_D((a,b),(c,d)) &= k(a,c) + k(b,d)\\ 
K_T((a,b),(c,d)) &= k(a,c) \cdot k(b,d)
\end{align*}
Because of the closure properties of kernels, $K_D$ and $K_T$ are PDS. They are known as the \emph{direct sum pairwise kernel} and the \emph{tensor pairwise kernel} respectively.
\item The decision function should be \emph{balanced}: $f(a,b) = f(b,a)$. Intuitively, this makes sense since switching the order of arguments shouldn't change the fact whether or not $a,b$ belong to the same class. Because of eqn (\ref{eqn:dec_fn_pairwise}), this property needs to be enforced by the kernel function. Essentially, we need $K((a,b),(c,d)) = K((a,b),(d,c))$. Note that $K_D$ and $K_T$ \emph{don't} satisfy this property. Some examples of balanced kernels, that are also PDS, from the literature:
\begin{align*}
K_{DL}((a,b),(c,d)) =& \frac{1}{2}(k(a,c) + k(a,d) +k(b,c)+k(b,d)) \\
K_{TL}((a,b),(c,d)) =& \frac{1}{2}(k(a,c)k(b,d) + k(a,d)k(b,c)) \\
K_{ML}((a,b),(c,d)) =& \frac{1}{4}(k(a,c) -k(a,d) -k(b,c) +k (b,d))^2 \\
K_{TM}((a,b),(c,d)) =&  K_{TL}((a,b)(c,d)) +K_{ML}((a,b),(c,d))
\end{align*}
$K_{DL}$ is known as the \emph{direct sum learning pairwise kernel}, $K_{TL}$ is the \emph{tensor learning pairwise kernel}, $K_{ML}$ is the \emph{metric learning pairwise kernel} and  $K_{TM}$ is known as the \emph{tensor metric learning pairwise kernel}.
\end{enumerate}
\subsubsection{Application}
For a pair of enzymes $x_i, x_j$, we include the triple $(x_i,x_j,+1)$ in the training data when they are known to interact, and $(x_i,x_j,-1)$, when they don't interact. Each enzyme is represented as a string - which can be its EC number or the gene nomenclature. We learn a pairwise SVM on this data. The pairwise kernels we use are the ones listed before - $K_{DL}, K_{TL}, K_{ML}, K_{TM}$ - with the composing kernels, $k$, being \emph{n-gram} rational kernels with $n=3$. Additonally, the authors use other kernels popular in the the bioinformatics domain. See \cite{Roche-Lima2014} for details.

\section{Large Scale SVM training using Rational Kernels}
\label{sec:large_scale}
The popularity of SVMs and the availability of increasingly large datasets has motivated lots of research around training SVMs faster. We present techniques from \cite{Allauzen2011} that specifically look at SVMs used with rational kernels. Since the work derives substantially from \cite{Hsieh:2008:DCD:1390156.1390208}, which discusses scalability in the case of linear kernels, we present that first.

\subsection{Large-scale training for linear kernels}
Given a set of instance-label pairs $(\mathbf{x_i},y_i), 1 \leq i \leq l, x_i \in \mathbb{R}^m, y_i \in \{-1,+1\}$	, SVM training solves the following optimization problem:
\begin{align}
&\min_{\mathbf{w}}\frac{1}{2}\mathbf{w}^T\mathbf{w} + C \sum\limits_{i=1}^l \xi_i \label{eqn:primal_linear} \\
&\text{ where } \xi_i = max(1-y_i \mathbf{w}^Tx_i,0) \nonumber
\end{align}
This formulation is known as the \emph{L1-SVM}\footnote{\cite{Hsieh:2008:DCD:1390156.1390208} also discusses the \emph{L2-SVM}, but we don't mention it since \cite{Allauzen2011} exclusively focuses on the L1-SVM.}. We often use a \emph{bias} term $b$, which can be included in the feature/weight vectors:
\begin{equation*}
\mathbf{x}^T_i \leftarrow [\mathbf{x}^T_i , 1],\;\;\mathbf{w}^T \leftarrow [\mathbf{w}^T , b]
\end{equation*}
Eqn (\ref{eqn:primal_linear}) is known as the \emph{primal} problem. One may solve the equivalent \emph{dual} problem:
\begin{align}
&\min_{\boldsymbol{\alpha}} f(\boldsymbol{\alpha}) = \frac{1}{2} \boldsymbol{\alpha}^TQ\boldsymbol{\alpha} - \boldsymbol{e}^T\boldsymbol{\alpha} \label{eqn:dual}\\
&\text{ subject to } 0 \leq \alpha_i \leq C, \forall i \nonumber \\
&\text{ where } Q_{ij} = y_iy_j \mathbf{x_i}^T\mathbf{x_j}\nonumber
\end{align}
We solve the dual optimization problem using \emph{coordinate descent}. We start with the initial value $\boldsymbol{\alpha}^0 \in \mathbb{R}^l$ and iteratively generate the vectors $\{\boldsymbol{\alpha}^k\}_{k=0}^\infty$. Each of these iterations is known as an \emph{outer iteration}. Every outer iteration consists of $l$ \emph{inner iterations}, where the coordinates $\alpha_i, 1 \leq i \leq l$, of $\boldsymbol{\alpha}$ are updated. We denote this by saying each outer iteration generates the vectors $\boldsymbol{\alpha}^{k,i} \in \mathbb{R}^l, i=1,2, ..., l+1$ such that $\boldsymbol{\alpha}^{k,1} = \boldsymbol{\alpha}^k$, $\boldsymbol{\alpha}^{k,l+1} = \boldsymbol{\alpha}^{k+1}$ and 
\begin{equation}
\boldsymbol{\alpha}^{k,i} = [\alpha_1^{k+1}, ..., \alpha_{i-1}^{k+1}, \alpha_i^k,..., \alpha_l^k]^T, \forall i =2,...,l
\end{equation}
In coordinate descent, we update each $\alpha_i$ individually to the best value possible. In moving from $\boldsymbol{\alpha}^{k,i}$ to $\boldsymbol{\alpha}^{k,i+1}$, we identify the incremental update $d$ in the following manner:
\begin{align}
&\min_d f(\boldsymbol{\alpha}^{k,i} + d \boldsymbol{e_i}), \text{ such that } 0 \leq \alpha_i^k + d\leq C \\
&\text{ where } \boldsymbol{e_i} = [0, ..., 0, 1, 0, ...,0]^T, \text{ the $i^{th}$ unit basis vector}. \nonumber
\end{align}
Substituting $\boldsymbol{\alpha}^{k,i} + d \boldsymbol{e_i}$ in eqn (\ref{eqn:dual}), we obtain:
\begin{align}
f(\boldsymbol{\alpha}^{k,i} + d \boldsymbol{e_i}) &= \frac{1}{2} (\boldsymbol{\alpha}^{k,i} + d \boldsymbol{e_i})^TQ(\boldsymbol{\alpha}^{k,i} + d \boldsymbol{e_i}) - \boldsymbol{e}^T(\boldsymbol{\alpha}^{k,i} + d \boldsymbol{e_i}) \nonumber \\
&=\frac{1}{2}[(\boldsymbol{\alpha}^{k,i})^TQ\boldsymbol{\alpha}^{k,i} + d \boldsymbol{\alpha}^{k,i}Q \boldsymbol{e_i} + d\boldsymbol{e_i}^TQ\boldsymbol{\alpha}^{k,i} + d^2 \boldsymbol{e_i}^TQ\boldsymbol{e_i}] - \boldsymbol{e}^T\boldsymbol{\alpha}^{k,i}-d\boldsymbol{e}^T\boldsymbol{e_i} \nonumber \\
&=\frac{1}{2}d^2Q_{ii} + \frac{d}{2}((\boldsymbol{\alpha}^{k,i})^TQ\boldsymbol{e}_i + \boldsymbol{e}_i^TQ\boldsymbol{\alpha}^{k,i}) -d + A \nonumber\\
&\text{ ...(where $A$ groups all terms that do not depend on $d$)} \nonumber \\
&=\frac{1}{2}d^2Q_{ii} + \frac{d}{2}((\boldsymbol{\alpha}^{k,i})^TQ_{*i} + Q_{i*}\boldsymbol{\alpha}^{k,i}) -d + A \nonumber\\
&=\frac{1}{2}d^2Q_{ii} + d((\boldsymbol{\alpha}^{k,i})^TQ_{*i})) -d + A \text{ \;\;...(since $Q$ is symmetric, and thus, $(\boldsymbol{\alpha}^{k,i})^TQ_{*i} = Q_{i*}\boldsymbol{\alpha}^{k,i}$)} \nonumber\\
&=\frac{1}{2}d^2Q_{ii} + d\nabla f(\boldsymbol{\alpha}^{k,i})\boldsymbol{e_i} + A \label{eqn:d_grad}
\end{align}
where $\nabla f(\boldsymbol{\alpha}^{k,i})\boldsymbol{e_i} = (Q\boldsymbol{\alpha}^{k,i})_i-1$ is the $i^{th}$ component of the gradient of $f$ at $\boldsymbol{\alpha}^{k,i}$. This is also written as $\nabla_i f (\boldsymbol{\alpha}^{k,i})$.

We define the \emph{projected gradient} $\nabla_i^Pf(\boldsymbol{\alpha})$ as:
\[
    \nabla_i^Pf(\boldsymbol{\alpha})= 
\begin{cases}
    \nabla_if(\boldsymbol{\alpha})& \text{ if } 0 < \alpha_i < C,\\
    \min(0,\nabla_if(\boldsymbol{\alpha}))& \text{ if } \alpha_i = 0,\\
    \max(0,\nabla_if(\boldsymbol{\alpha}))& \text{ if } \alpha_i = C.
\end{cases}
\]

It is easy to see that when $\nabla_i^Pf(\boldsymbol{\alpha}^{k,i}) = 0$, there is an optimum at $d=0$ i.e. $\alpha^{k}_i=\alpha^{k+1}_i$. Otherwise, we solve eqn (\ref{eqn:d_grad}) by taking the derivative of the RHS wrt $d$ and equating it to $0$. Accounting for the constraints on $\alpha_i$, the solution may be written as:
\begin{equation}
\alpha_i^{k+1} = \min \bigg(\max \bigg( \alpha_i^{k} - \frac{\nabla_i f(\boldsymbol{\alpha}^{k,i})}{Q_{ii}},0\bigg),C\bigg) \label{eqn:alpha_soln}
\end{equation}
Thus, we need to calculate $Q_{ii}$ and $\nabla_i f(\boldsymbol{\alpha}^{k,i})$ per update. $Q_{ii}$ does not change during the optimization procedure, so it can be stored in memory for all $i$ (space complexity $O(l)$). The gradient is given by:
\begin{equation}
\nabla_i f(\boldsymbol{\alpha}) = (Q\boldsymbol{\alpha})_i - 1 = \sum\limits_{j=1}^{l} Q_{ij} \alpha_{j} - 1 \label{eqn:vanilla_gradient}
\end{equation}
To calculate eqn (\ref{eqn:vanilla_gradient}), we need to calculate the $i^{th}$ row of $Q$. $Q$ has a space complexity of $O(l^2)$; we assume that $Q$ is large enough so that it cannot be stored in the memory. Let $n$ be the non-zero elements in a data instance on an average. Computing each product $\boldsymbol{x_i}^T\boldsymbol{x_j}$ needs $O(n)$ operations, and calculating the $i^{th}$ row thus takes $O(ln)$ operations. 

However, since our SVM is linear\footnote{implicit in the definition of $Q$ since we don't use a kernel function}, we can define:
\begin{equation}
\boldsymbol{w} = \sum\limits_{j=1}^{l} y_i \alpha_j \boldsymbol{x_j} \label{eqn:large_scale_linear_w}
\end{equation}
Note that $\boldsymbol{w}$ is a sum of vectors in $\mathbb{R}^m$. Now eqn (\ref{eqn:vanilla_gradient}) may be written as:
\begin{equation}
\nabla_i f(\boldsymbol{\alpha}) = y_i \boldsymbol{w}^T\boldsymbol{x_i}-1 \label{eqn:w_gradient}
\end{equation}
If we already have $\boldsymbol{w}$, computing eqn (\ref{eqn:w_gradient}) takes only $O(n)$ operations. We can avoid computing $\boldsymbol{w}$ from scratch in each iteration by maintaining it throughout the procedure. $\boldsymbol{w}$ has a space complexity of $O(m)$; hence storing it in memory is feasible. We start with setting $\boldsymbol{\alpha}^0=\boldsymbol{0}$, so $\boldsymbol{w}=0$. Every time some $\overline{\alpha_i}$ is updated to $\alpha_i$, we can make the following $O(n)$ update to $\boldsymbol{w}$:
\begin{equation}
\bm{w} \leftarrow \bm{w} + (\alpha_i -\overline{\alpha_i}) y_i \boldsymbol{x_i} \label{eqn:update_w}
\end{equation}
Because the update requires $\boldsymbol{x_i}$, we need to store these vectors in the memory too. The total time complexity, from eqns (\ref{eqn:w_gradient}) and (\ref{eqn:update_w}), is $O(n)$. This is much better than the complexity of $O(ln)$ we saw for eqn (\ref{eqn:vanilla_gradient}). The space complexity is $O(m) + O(ml) = (ml)$ (the former for $\boldsymbol{w}$ and the latter for all $\boldsymbol{x_i}$), which is also better than storing $Q$, that has a space complexity of $O(l^2)$. Of course, we assume $l \gg m$. When the procedure terminates we directly obtain the value of $\boldsymbol{w}$ as defined in the primal.

We need to account for the boundary case $Q_{ii}=0$ in eqn (\ref{eqn:alpha_soln}). This implies $\mathbf{x_i}^T\mathbf{x_i} =0$, and hence, $\mathbf{x_i}=0$. From eqn (\ref{eqn:w_gradient}), we have $\nabla_i f(\boldsymbol{\alpha}) = -1$. Substituting $Q_{ii}=0, \nabla_i f(\boldsymbol{\alpha})=-1$ in eqn (\ref{eqn:d_grad}), we see that $d$ can be infinitely large to minimize $f$. But $\alpha_i \leq C$, so we set $\alpha_i^{k+1}=C$. This case can be included in eqn (\ref{eqn:alpha_soln}) by defining $\frac{1}{Q_{ii}} = \infty$ when $Q_{ii}=0$. Algorithm \ref{algo:dual_coordinate_descent_linear} presents a summary of the steps discussed. 

\begin{algorithm}[H]
\label{algo:dual_coordinate_descent_linear}
%\SetAlgoLined
\DontPrintSemicolon
%\KwData{this text}
%\KwResult{how to write algorithm with \LaTeX2e }
\KwData{Given $\boldsymbol{\alpha}$ and corresponding $\boldsymbol{w}=\sum\limits_iy_i\alpha_i \boldsymbol{x_i}$}
\KwResult{Find optimal  $\boldsymbol{\alpha}$}\;
\While{$\boldsymbol{\alpha}$ is not optimal}{
%read current\;
\For{$i\leftarrow 1$ \KwTo $l$}{
	$G = y_i\boldsymbol{w}^T \boldsymbol{x_i} - 1$\;

    $G'= 
	\begin{cases}
    G& \text{ if } 0 < \alpha_i < C,\\
    \min(0,G)& \text{ if } \alpha_i = 0,\\
    \max(0,G)& \text{ if } \alpha_i = C.
	\end{cases}
	$\;
	\If{$|G'|\neq0$}{
		$\overline{\alpha_i} \leftarrow \alpha_i$\;
		$\alpha_i \leftarrow \min(\max(\alpha_i - G'/Q_{ii}, 0), C)$\;
		$\boldsymbol{w} \leftarrow \boldsymbol{w} + (\alpha_i- \overline{\alpha_i})y_i\boldsymbol{x_i}$
	}

}
}
\caption{Dual coordinate descent for linear SVMs}
\end{algorithm}

\subsection{Large-scale training for rational kernels}
We now extend algorithm \ref{algo:dual_coordinate_descent_linear} to rational kernels. Some notation we use in this section:
\begin{enumerate}
\item $\mathit{\Pi_2}$ denotes the \emph{output projection} of a transducer $T$. $\mathit{\Pi_2(T)}$ gives us a \emph{weighted automaton} $A$ that drops the input symbols from transitions on $T$; the output symbols and weights are retained. Fig \ref{fig:output_projection} shows an example of output projection.

\begin{figure*}[!t]
\centering
\subfloat[$T$]{\includegraphics[width=1.75 in]{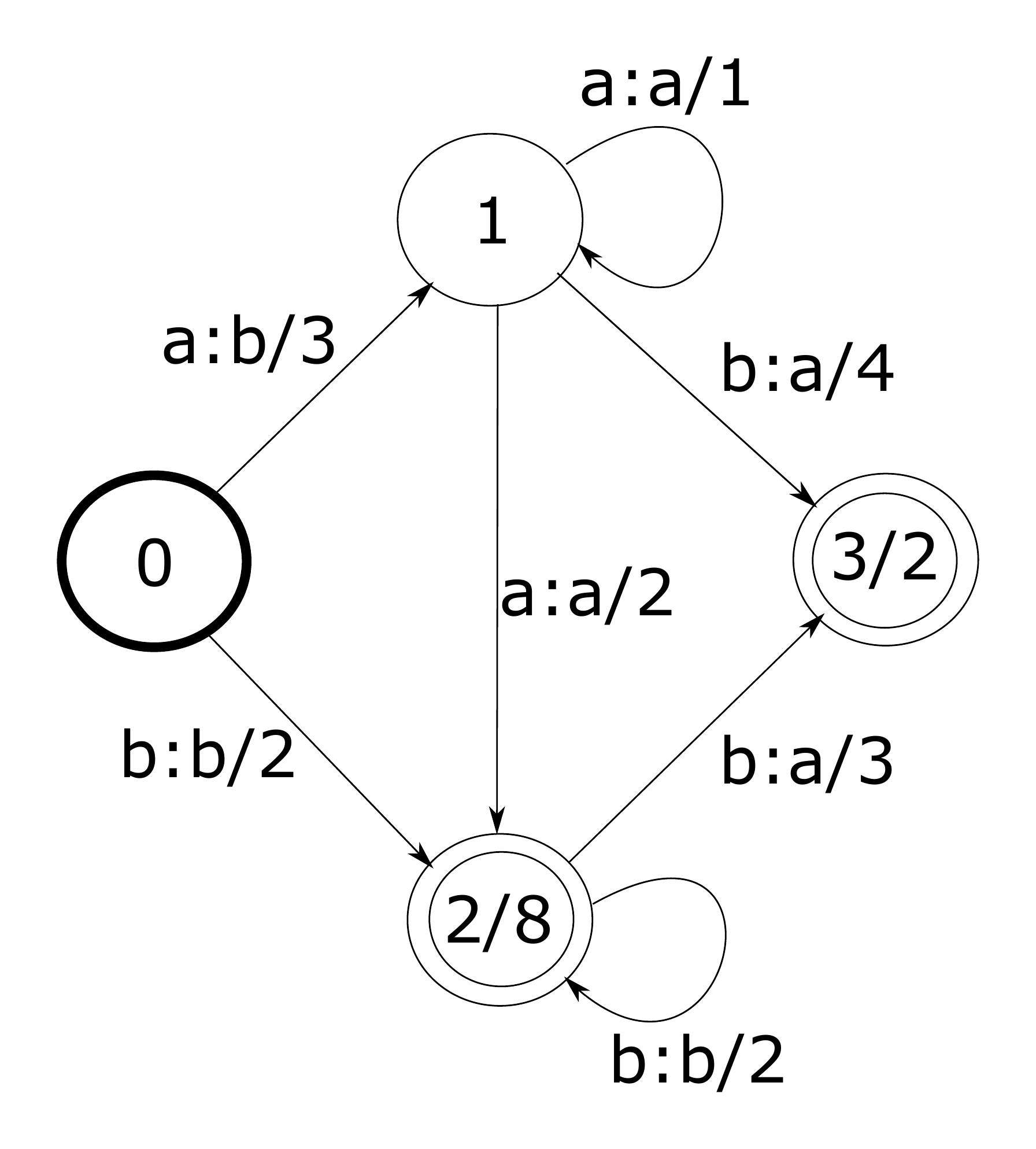}%
\label{fig:transducer_to_be_projected}}
\hfil
\subfloat[$A=\mathit{\Pi_2(T)}$]{\includegraphics[width=1.75in]{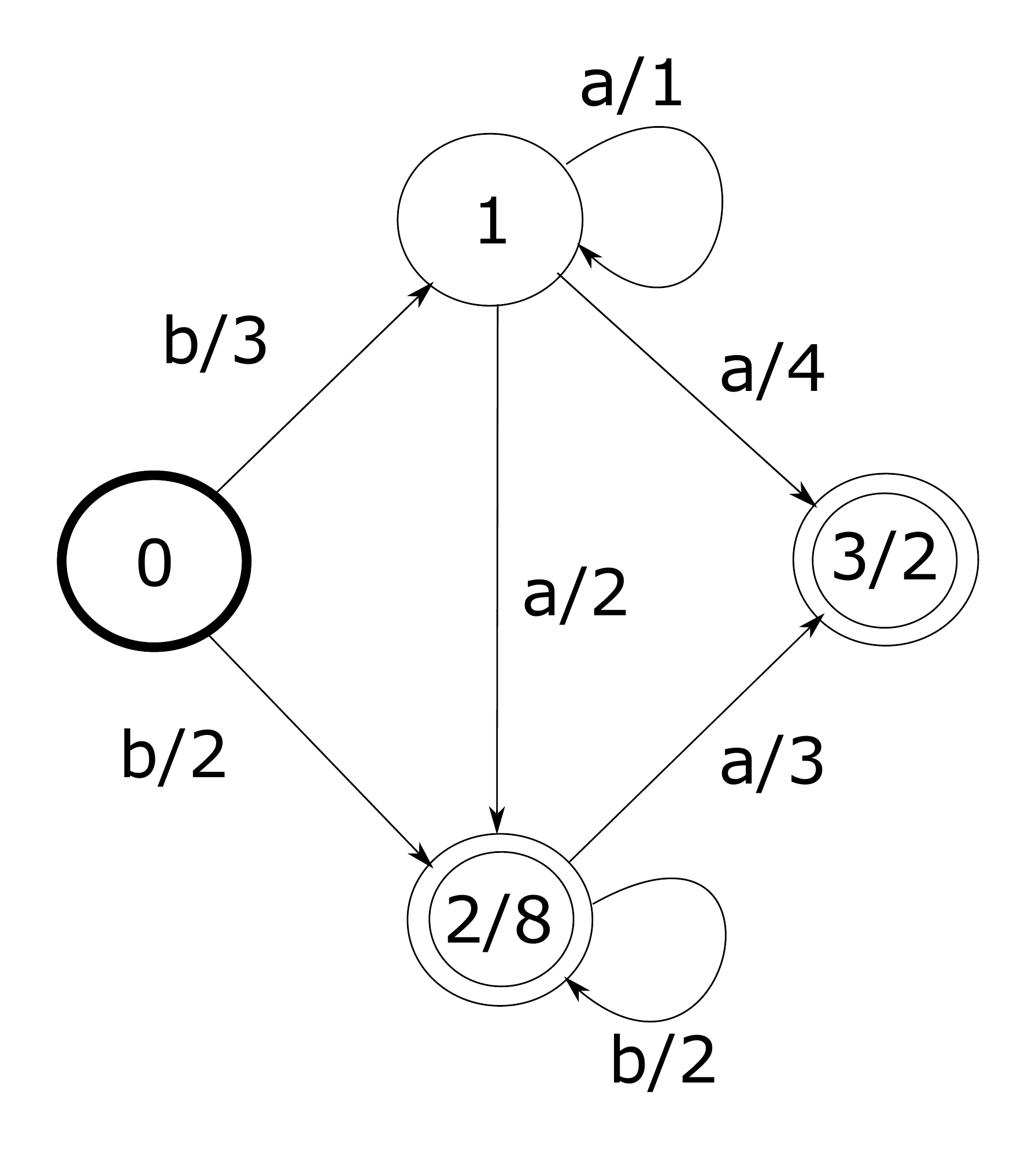}%
\label{fig:projected_automaton}}
\caption{$A$ is the output projection of $T$.\\$T(aab, baa)= A(baa)= 3\times1\times4\times2+3\times2\times3\times2$}
\label{fig:output_projection}
\end{figure*}

\item The linear operator $D$ denotes the sum of \emph{all} accepting paths for a transducer $T$.
\item We use a \emph{generic} form of the matrix $Q$ (eqn (\ref{eqn:dual})) here: $Q_{ij} = y_i y_j k(\boldsymbol{x_i},\boldsymbol{x_j})$ where $k$ is a kernel function. Let $\boldsymbol{U}$ be the weighted transducer associated with the kernel $k$.
\item We will assume the weighted transducers we use are do not admit any $\epsilon$-transitions i.e. transitions where both the input and output symbol\footnote{\cite{Allauzen2011} insists that there shouldn't be $\epsilon$-cycles on the input symbol, but the inverse of such a transducer, which violates this condition, is acyclic too.} are $\epsilon$ . This property holds for all rational kernels we have seen so far.
\item As before, $T(x,y)$ denotes the weight $T$ associates with the sequence pair $(x,y)$. 
\end{enumerate}
We exclusively look at transducers defined on the \emph{real semiring} $(\mathbb{R}_+, +, \times, 0, 1)$ here.

The first extension to Algorithm \ref{algo:dual_coordinate_descent_linear} is quite straightforward. We now have sequences instead of the vectors $\boldsymbol{x_i}$. We replace vectors $\boldsymbol{x_i}$ with weighted automata $\boldsymbol{X_i}$ that only accept the sequence $\boldsymbol{x_i}$ and returns a weight of $1$ for such a path. $\boldsymbol{W}$ denotes a weighted automaton that represents $\boldsymbol{w} = \sum\limits_{j=1}^{l}y_j\alpha_j\boldsymbol{x_j}$ i.e. it only accepts sequences $\boldsymbol{x_j}, j=1,2...,l$ and returns the weight $y_j\alpha_j$ for a sequence $\boldsymbol{x_j}$ it accepts. Note that $D(\boldsymbol{X_j}) =1, D(\boldsymbol{W}) =\boldsymbol{w}$.

Using the linearity property of $D$, we have:
\begin{align}
(Q\boldsymbol{\alpha})_i &= \sum\limits_{j=1}^{l} y_iy_j k(\boldsymbol{x_i},\boldsymbol{x_j}) \alpha_j =\sum\limits_{j=1}^{l} y_iy_j D(\boldsymbol{X_i} \circ \boldsymbol{U} \circ \boldsymbol{X_j} ) \nonumber \\
&=	D(y_i \boldsymbol{X_i} \circ \boldsymbol{U} \circ \sum_{j=1}^{l} y_j\alpha_j \boldsymbol{X_j}) \label{eqn:large_scale_first_move_in}\\
& = D(y_i \boldsymbol{X_i} \circ \boldsymbol{U} \circ \boldsymbol{W})
\end{align}

Step (\ref{eqn:large_scale_first_move_in}) can be thought of as replacing the addition of paths over different automata $\boldsymbol{X_j}$ in the previous step with using a single automata with many accepting paths, one each for $\boldsymbol{X_j}, 1\leq j \leq l$.

Given a kernel, since $\boldsymbol{U}$ is constant, the complexity of the composition is $O(|\boldsymbol{X_i}||\boldsymbol{W}|)$. To calculate $D$, paths can be found on this transducer using a shortest distance algorithm. This complexity is linear in the size of  the transducer. The size of the composed transducer is $O(|\boldsymbol{X_i}||\boldsymbol{W}|)$ (we again consider $|\boldsymbol{U}|$ as a constant), hence the complexity to compute $D$ is also $O(|\boldsymbol{X_i}||\boldsymbol{W}|)$. Hence, the overall time complexity for calculating $(Q\boldsymbol{\alpha})_i$ is $O(|\boldsymbol{X_i}||\boldsymbol{W}|)$. We can maintain $\boldsymbol{W}$ in an efficient datastructure where the update takes the form:
\begin{equation}
\boldsymbol{W}\leftarrow \boldsymbol{W} + \Delta(\alpha_i)y_i\boldsymbol{X_i}
\end{equation}

Although the above arguments seem similar to the ones we have seen for the linear kernel, it turns out how much we save on time-complexity is extremely implementation-dependent. Superficially, we seem to have avoided computing the individual kernel values $k(\boldsymbol{X_i}, \boldsymbol{X_j}), 1 \leq j \leq l$, which have a time complexity of $O(|\boldsymbol{X_i}||\boldsymbol{X_1}|) + O(|\boldsymbol{X_i}||\boldsymbol{X_2}|) + ... + O(|\boldsymbol{X_i}||\boldsymbol{X_l}|)$. Whether this is better depends on how $\boldsymbol{W}$ is implemented. If $|\boldsymbol{W}| \sim \sum_{j=1}^{l}|\boldsymbol{X_j}|$, which would be the case if we naively build $\boldsymbol{W}$ with a separate path for each of $\boldsymbol{X_j}$, this is not an improvement. However, if $\boldsymbol{W}$ is a minimal DFA, we could do much better in terms of time complexity. In the case of linear kernels, the difference in time complexities was explicitly clear.

Fortunately, $\cite{Allauzen2011}$ shows that there is more we can do with rational kernels.

Before we proceed, we note the following:
\begin{enumerate}
\item 
Consider the weighted automaton $X$ and transducer $T$. We think of $X$ as a transducer with identical input and output symbols for each transition. Hence, $X$ only possesses accepting paths of the form $(u,u), \forall u \in L_X$, where $L_X$ is the language accepted by $X$.
\begin{align*}
&X \circ T(u,v) =X(u,u)T(u,v) + \sum\limits_{z \neq u}X(u,z)T(u,v) \\
&\text{Since }X(u,z) = 0\text{ when } u \neq z, \text{ and } X(u,u) = 0\text{ when } u \neq L_X,\\\
&X \circ T(u,v) = \begin{cases}
				X(u)T(u,v)\text{, if } u \in L_X \\
				0 \text{, if } u \notin L_X 
			\end{cases}
\end{align*}
We may think of composition with a weighted automaton as applying a \emph{filter} on the transducer: the set of accepting paths in the transducer $T$, post composition/filtering, is a subset of the original set of accepting paths, such that each path has $u \in L_X$ as its input sequence. 
\item Composition of automata works differently than composition of two transducers, or that of an automaton and a transducer. Technically, this is known as \emph{intersection}, and was mentioned in Section \ref{sec:ratk_theory_algo}. We denote this operation by $\diamond$. If $A=(Q_A, \Sigma, \delta_A, s_A, F_A), B=(Q_B, \Sigma, \delta_B, s_B, F_B)$, then the composed automaton is $A \diamond B =(Q_A \times Q_B, \Sigma, \delta, s_A\cdot s_B, F_A \times F_B)$, where $\delta$ is given by:
\begin{equation*}
\delta = \begin{cases}
		\delta_A(q_A, e)\cdot \delta_B(q_B, e) \text{ if } \delta_A(q_A, e) \text{ and } \delta_B(q_B, e) \text{ are defined, where } q_A \in Q_A, q_B \in Q_B  \\
		\text{undefined otherwise}
\end{cases} 
\end{equation*}
Only sequences accepted by both $A$ and $B$ are accepted by $A \diamond B$.
\end{enumerate}
We will also assume that the weighted transducer $\boldsymbol{U}$ can be expressed as $\boldsymbol{T} \circ \boldsymbol{T^{-1}}$, which is true for all PDS sequence kernels seen in practice. Then,
\begin{align}
(Q \boldsymbol{\alpha})_i &= D(y_i \boldsymbol{X_i} \circ (\boldsymbol{T} \circ \boldsymbol{T^{-1}}) \circ \boldsymbol{W}) \label{eqn:large_rational_derivation_1}\\
&= D((y_i \boldsymbol{X_i} \circ \boldsymbol{T}) \circ (\boldsymbol{T^{-1}} \circ \boldsymbol{W})) \label{eqn:large_rational_derivation_2} \\
&= D((y_i \boldsymbol{X_i} \circ \boldsymbol{T}) \circ (\boldsymbol{W} \circ \boldsymbol{T})^{-1}) \label{eqn:large_rational_derivation_3} \\
&= D(\mathit{\Pi_2}(y_i \boldsymbol{X_i} \circ \boldsymbol{T}) \diamond \mathit{\Pi_2}(\boldsymbol{W} \circ \boldsymbol{T})) \label{eqn:large_rational_derivation_4} \\
&=D(\boldsymbol{\Phi'}_i \diamond \boldsymbol{W'}) \label{eqn:large_rational_derivation_5} \\
\text{ where } & \boldsymbol{\Phi'}_i=\mathit{\Pi_2}(y_i \boldsymbol{X_i} \circ \boldsymbol{T}), \boldsymbol{W'} = \mathit{\Pi_2}(\boldsymbol{W} \circ \boldsymbol{T}) \nonumber
\end{align}

Eqn (\ref{eqn:large_rational_derivation_2}) uses the \emph{associativity} of $\circ$. 

Going from eqn (\ref{eqn:large_rational_derivation_2}) to eqn (\ref{eqn:large_rational_derivation_3}), we have $(\boldsymbol{T^{-1}} \circ \boldsymbol{W}) = (\boldsymbol{W} \circ \boldsymbol{T})^{-1}$. To see why, consider the filtering effect of $\boldsymbol{W}$. $(\boldsymbol{T^{-1}} \circ \boldsymbol{W})(u,v)$ has accepting paths from $\boldsymbol{T^{-1}}$ such that $v \in L_{\boldsymbol{W}}$. Equivalently, these are paths in  $\boldsymbol{T}(v,u)$ with $v \in L_{\boldsymbol{W}}$; such paths can be produced by $\boldsymbol{W} \circ \boldsymbol{T}$. Accounting for the order of accepted sequences, $(u,v)$, we have $(\boldsymbol{T^{-1}} \circ \boldsymbol{W}) = (\boldsymbol{W} \circ \boldsymbol{T})^{-1}$. 

Eqn (\ref{eqn:large_rational_derivation_4}) is slightly tricky. In the previous eqn, we consider the composition of two transducers $A=\boldsymbol{X_i} \circ \boldsymbol{T} ,B=(\boldsymbol{W} \circ \boldsymbol{T})^{-1}$ (we ignore $y_i$ for now since it determines weights, not accepting paths). Because of the filtering effects of $\boldsymbol{X_i}$ and $\boldsymbol{W}$, the only cases with non-zero path weights are:
\begin{align*}
A \circ B(u,v) &= \sum\limits_{z \in \Sigma^*} \boldsymbol{T}(u,z)  \boldsymbol{T}^{-1}(z,v) \text{ where } u \in L_{\boldsymbol{X_i}}, v \in L_{\boldsymbol{W}} \\
&=\sum\limits_{z \in \Sigma^*} \boldsymbol{T}(u,z) \boldsymbol{T}(v,z) %\text{ where } u \in L_{\boldsymbol{X_i}}, v \in L_{\boldsymbol{W}} 
\end{align*}
Thus, we are interested in accepting paths in $T$, that have output sequences produced by \emph{both} the following kinds of input: (1) the input is filtered by $\boldsymbol{X_i}$ (2) the input is filtered by $\boldsymbol{W}$. In eqn (\ref{eqn:large_rational_derivation_4}), we identify these by filtering the inputs on the transducers first, and then using the output projection $\mathit{\Pi_2}$ automata, with the operator $\diamond$, to screen these common output sequences.

%
%\begin{figure*}[!t]
%\centering
%\subfloat[]{\includegraphics[width=3.0 in]{auto_tran_compose.pdf}}%
%\caption{$y_i \boldsymbol{X_i} \circ \boldsymbol{T}$ where $y_1=1, x_i = ababa$ and $\boldsymbol{T}$ is a bigram counter.}
%\label{fig:compose_auto_trans}
%\end{figure*}
%
%
%
%\begin{figure*}[!t]
%\centering
%\subfloat[$x_i = ababa,y_i=1$. $Q_{ii}=8.$]{\includegraphics[width=2.0 in]{large_scale_str_1.pdf}%
%\label{fig_str1}}
%\hfil
%\subfloat[$x_i=abaab,y_i=1$. $Q_{ii}=6.$]{\includegraphics[width=2.0in]{large_scale_str_2.pdf}%
%\label{fig_str2}}
%\hfil
%\subfloat[$x_i=abbab, y_i=-1$. $Q_{ii}=6.$]{\includegraphics[width=2.0in]{large_scale_str_3.pdf}%
%\label{fig_str3}}
%\caption{$\boldsymbol{\Phi_i}'$ for various sequences and labels.}
%\label{fig:single_transducers}
%\end{figure*}
%
%
%
%Fig \ref{fig:compose_auto_trans} shows an example of $y_i \boldsymbol{X_i} \circ \boldsymbol{T}$ where $y_1=1, x_i = ababa$ and $\boldsymbol{T}$ is a bigram counter.

$\boldsymbol{\Phi'}_i$ can be precomputed for $i=1,2,..,l$, and we can maintain $\boldsymbol{W'}$ using the update rule:
\begin{equation}
\boldsymbol{W'} \leftarrow \boldsymbol{W'} + \Delta(\alpha_i)\boldsymbol{\Phi'}_i \label{eqn:update_W'}
\end{equation}
The gradient (eqn (\ref{eqn:vanilla_gradient})) can be written as
\begin{equation}
\nabla_i f(\boldsymbol{\alpha}) = (Q\boldsymbol{\alpha})_i - 1 = D(\boldsymbol{\Phi'}_i \diamond \boldsymbol{W'})-1 \label{eqn:grad_update_rational}
\end{equation}

%The update rule for $D(\boldsymbol{\Phi'}_i \diamond \boldsymbol{W'})$ can be written as:
%\begin{equation}
%D(\boldsymbol{\Phi'}_i \diamond \boldsymbol{W'}) \leftarrow D(\boldsymbol{\Phi'}_i \diamond \boldsymbol{W'}) + D(\boldsymbol{\Phi'}_i \diamond \Delta\boldsymbol{W'})
%\end{equation}

This is similar to the case of linear kernels (eqn (\ref{eqn:w_gradient})), and we may modify Algorithm \ref{algo:dual_coordinate_descent_linear} for the case of rational kernels to obtain Algorithm \ref{algo:dual_coordinate_descent_rational}.\\

\begin{algorithm}[H]
\label{algo:dual_coordinate_descent_rational}
%\SetAlgoLined
\DontPrintSemicolon
%\KwData{this text}
%\KwResult{how to write algorithm with \LaTeX2e }
\KwData{Given $\boldsymbol{\alpha}$ and corresponding $\boldsymbol{W'}$}%=\sum\limits_iy_i\alpha_i \boldsymbol{x_i}$}
\KwResult{Find optimal  $\boldsymbol{\alpha}$}\;
\While{$\boldsymbol{\alpha}$ is not optimal}{
%read current\;
\For{$i\leftarrow 1$ \KwTo $l$}{
	$G  = D(\boldsymbol{\Phi'}_i \diamond \boldsymbol{W'})-1$\;

    $G'= 
	\begin{cases}
    G& \text{ if } 0 < \alpha_i < C,\\
    \min(0,G)& \text{ if } \alpha_i = 0,\\
    \max(0,G)& \text{ if } \alpha_i = C.
	\end{cases}
	$\;
	\If{$|G'|\neq0$}{
		$\overline{\alpha_i} \leftarrow \alpha_i$\;
		$\alpha_i \leftarrow \min(\max(\alpha_i - G'/Q_{ii}, 0), C)$\;
		$\boldsymbol{W'} \leftarrow \boldsymbol{W'} + (\alpha_i- \overline{\alpha_i})\boldsymbol{\Phi'}_i$
	}

}
}
\caption{Dual coordinate descent for SVMs using rational kernels}
\end{algorithm}

As we have seen before, time complexities in the case of rational kernels are depend on how we represent $\boldsymbol{W'}$. We will assume $\boldsymbol{\Phi_i}$s and hence, $\boldsymbol{W'}$, are acyclic; this is true for all rational kernels used in practice.  Given an acyclic weighted automaton $A$, we denote by $s(A)$ the maximal length
of an accepting path in $A$ and by $n(A)$ the number of accepting paths in $A$.

A simple way to represent $\boldsymbol{W'}=\sum_{i=1}^{l}\alpha_i\boldsymbol{\Phi'}_i$ is to have an initial state, with $l$ outgoing $\epsilon-$transitions, where the $i^{th}$ edge is weighted $\alpha_i$ and directed at the start state of the transducer $\boldsymbol{\Phi}_i$. Here $|\boldsymbol{W'}| = l + \sum_{i=1}^{l}|\boldsymbol{\Phi'}_i|$. Accommodating updates to $\alpha$ in eqn (\ref{eqn:update_W'}) are $O(1)$, since this involves only changing the weight on one of the $\epsilon-$transitions. However, calculating the gradient using eqn (\ref{eqn:grad_update_rational}) is $O(|\boldsymbol{\Phi'}_i|\sum_{j=1}^{l}|\boldsymbol{\Phi'}_j|)$.

We may represent $\boldsymbol{W'}$ as a \emph{weighted trie}. A weighted trie is a rooted tree where each edge is labeled and each node is weighted. To calculate the composition $\boldsymbol{\Phi'}_i \diamond \boldsymbol{W'}$, by finding sequences accepted by both $\boldsymbol{\Phi'}$ and $\boldsymbol{W'}$, and hence the gradient by eqn (\ref{eqn:grad_update_rational}), we need to look-up paths in the trie corresponding to each accepting path of $n(\boldsymbol{\Phi'}_i)$, and match $s(\boldsymbol{\Phi'}_i)$ nodes - for a total complexity of $O(n(\boldsymbol{\Phi'}_i))(s(\boldsymbol{\Phi'}_i))$. Updates to eqn (\ref{eqn:update_W'}) are slower - at most $n(\boldsymbol{\Phi'}_i)$ node weights need to be updated. Hence, the update complexity is $O(n(\boldsymbol{\Phi'}_i))$. %Interestingly, in the case of representation by trie, the time complexities don't depend on $|\boldsymbol{W'}|$.

There is a lot of benefit to representing $\boldsymbol{W'}$ by a minimal automaton - the number of states can be much lower than the trie representation. The time complexity of the gradient computation is $O(|\boldsymbol{\Phi'}_i \diamond \boldsymbol{W'}|)$. Unfortunately, it is hard to determine the size of the minimal automaton in terms of the component $\boldsymbol{\Phi'}_i$s, hence the expression for the complexity cannot be further simplified. Typically, we expect this complexity to be much better than what we have seen for the trie representation. For the same reason it is hard to determine the update cost for eqn (\ref{eqn:update_W'}).

Table \ref{tab:large_scale_complexities} summarizes the complexities for various representations.

\begin{table}[!t]
\centering
\begin{tabular}{|c|c|c|c|}
\hline
&\multicolumn{2}{c|}{Time Complexity}&\\
\cline{2-3}
 Representation of $\mathbf{W'}$ &gradient&update&Space Complexity\\
\hline \hline
naive ($\mathbf{W'_n}$)&$O(|\boldsymbol{\Phi'_i}|\sum\limits_{i=1}^{l}|\boldsymbol{\Phi'_i}|)$&$O(1)$&$O(l)$\\
%\hline
trie ($\mathbf{W'_t}$)&$O(n(\boldsymbol{\Phi'_i})l(\boldsymbol{\Phi'_i}))$&$O(n(\boldsymbol{\Phi'_i}))$&$O(|\boldsymbol{W'_t}|)$\\
%\hline
min. automaton ($\mathbf{W'_m}$)&$O(|\boldsymbol{\Phi'_m} \circ \boldsymbol{W'_m}|)$&NA&$O(|\boldsymbol{W'_m|})$\\
\hline
\end{tabular}
\caption{Time and space complexities for different representations of $\boldsymbol{W'}$}

\label{tab:large_scale_complexities}
\end{table}

Table \ref{tab:large_scale_runtimes} shows the runtimes for a classification task on the \emph{Reuters} dataset. Different n-gram kernels were used with $n=4,5,6,7$. The ``New Algo.'' column displays runtimes for the  trie-based representation. We notice that we do significantly better than the standard SMO training algorithm. A comparison of the sizes of   $\boldsymbol{W'}$ for a trie based representation and a minimal automaton are shown. The savings in space complexity is significant in all cases.

\begin{table}[!t]
\centering
\begin{tabular}{|c|c|c|c|c|c|}
\hline
& \multicolumn{2}{c|}{Runtimes} & \multicolumn{3}{c|}{$|\boldsymbol{W'}|$}\\
\cline{2-6}
 Kernel&SMO-like&New Algo. with Trie&Trie&min. aut.& \% reduction\\
\hline \hline
4-gram & $618m43s$ & $16m30s$ & $242,570$ & $106,640$ & $56$\%\\
%\hline
5-gram& $>2000m$ & $23m17s$ & $787,514$ & $237,783$ & $69.8$\% \\
%\hline
6-gram&$>2000m$ & $31m22s$ & $1,852,634$ & $441,242$ & $76.2$\%\\
%\hline
7-gram&$>2000m$ & $37m23s$ & $3,570,741$ & $727,743$ & $79.6$\% \\
\hline
\end{tabular}
\caption{Runtimes and sizes, when training on Reuters dataset.}

\label{tab:large_scale_runtimes}
\end{table}

\section{Learning Rational Kernels for classification}
\label{sec:learning_kernels}
All rational kernels we have seen so far embody some standard notion of similarity between sequences. An interesting question is, can we play around with the transition weights to obtain a kernel that is better in some sense? Since each transition has an associated weight, we have many ``degrees of freedom'' to set these weights - how do decide on an optimal set of weights?

\cite{4685446} explores the possibility of \emph{learning} these weights given data. The objective is to obtain better performance with certain kernel-based methods. We focus on using SVMs in this section. The discussion is structured thus: we define the kernel matrix corresponding to rational kernels in way that is convenient for further analysis. We then formulate the kernel learning problem using this representation. Finally, we provide a feasible way to learn an optimal kernel matrix.

\subsubsection{The kernel matrix} \cite{4685446} looks only at \emph{count-based rational kernels} i.e. kernels that count the number of occurrences of specific pattern in a string. The bigram-counter we have seen before belongs to this family. In general, a count-based kernel uses a transducer as in Fig \ref{fig:count_based_kernels}. The portion ``$A:A/1$'' counts the subsequences we are interested in. A bigram counter is shown for comparison.

\begin{figure*}[!t]
\centering
\subfloat[General count based rational kernel]{\includegraphics[width=1.5 in]{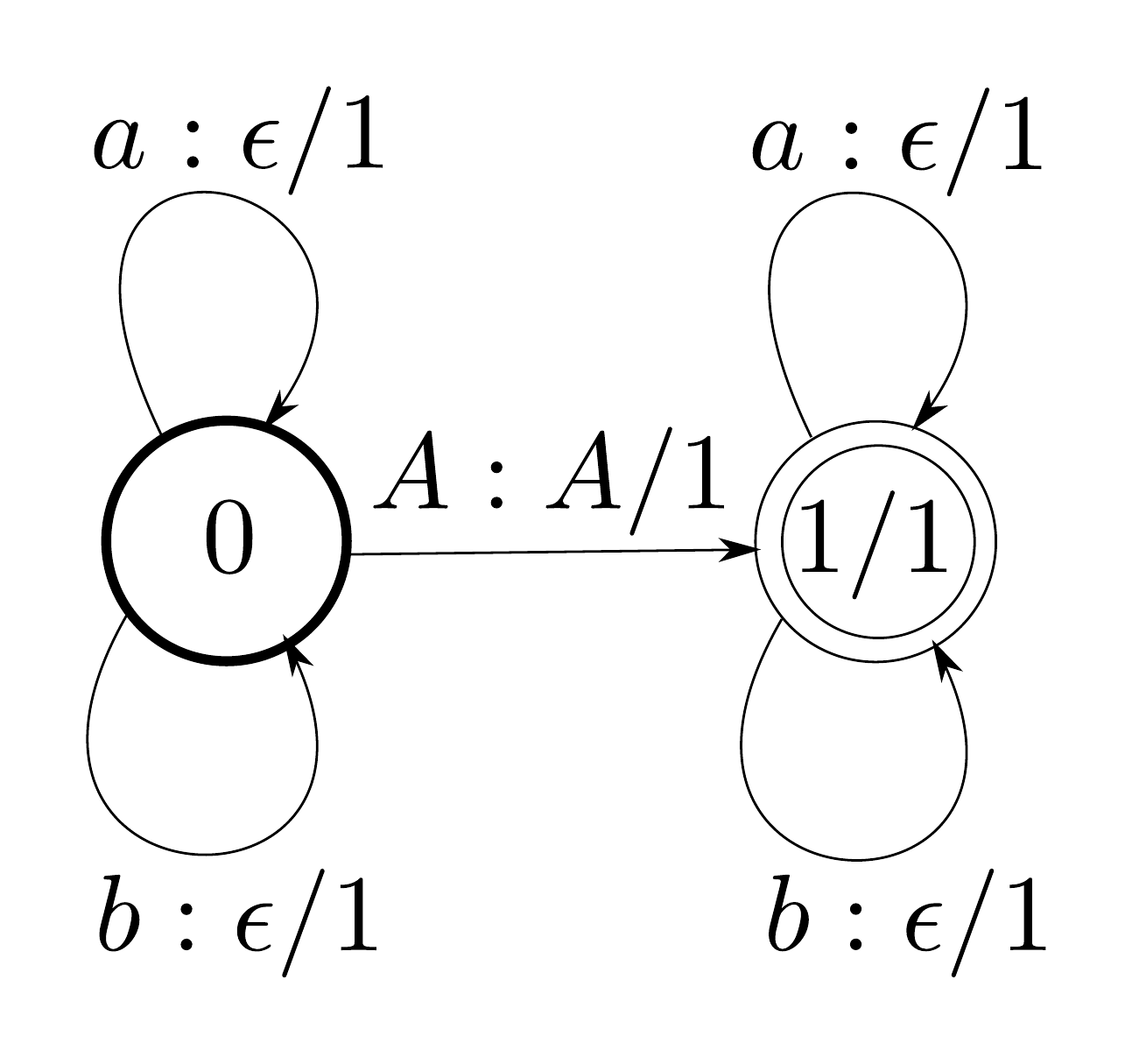}%
\label{fig:count_based_kernel}}
\hfil
\subfloat[A bigram counter. $A:A/1$ is the dashed box.]{\includegraphics[width=2.5in]{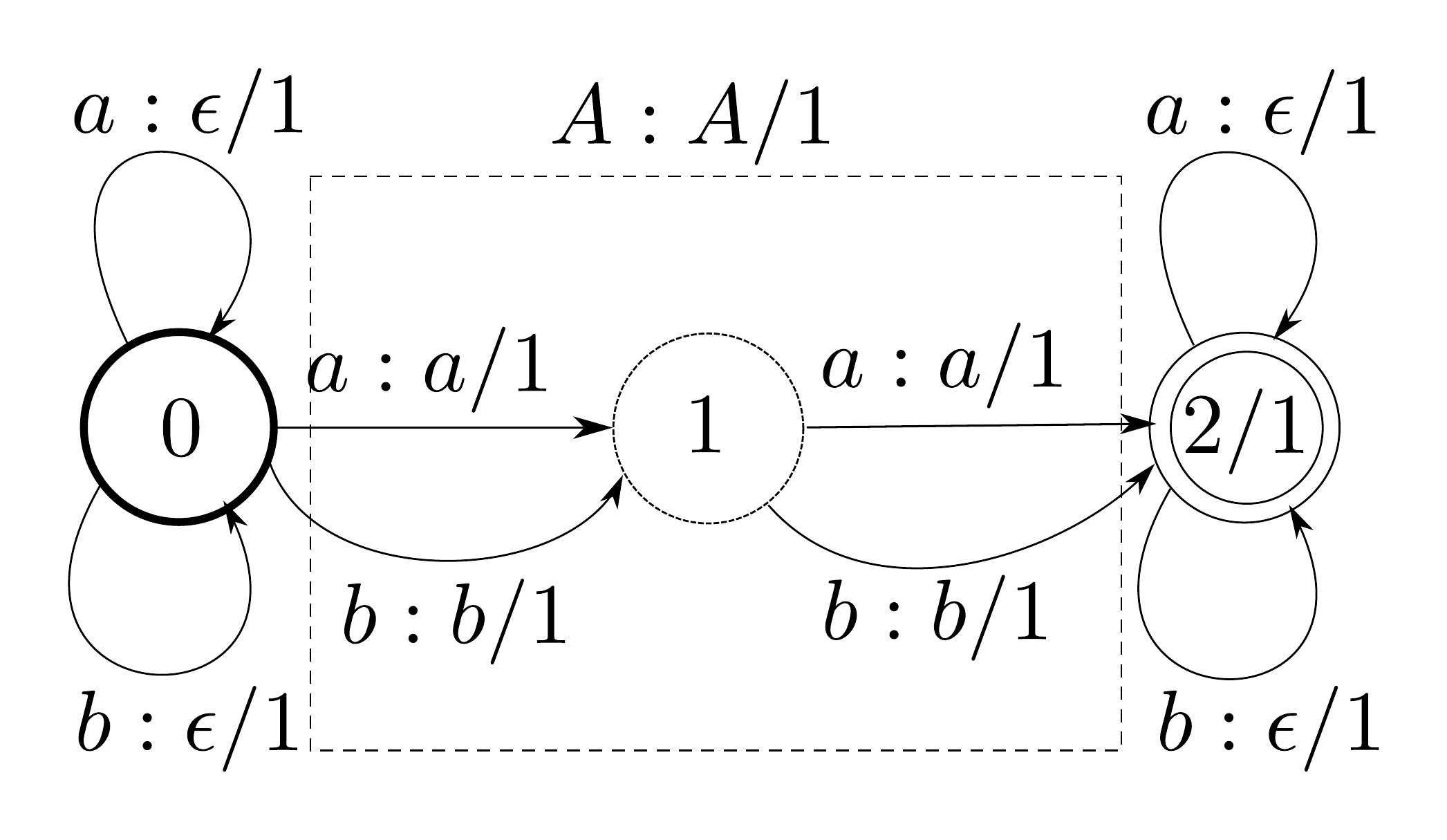}%
\label{fig:bigram_as_count_based}}
\caption{(a) shows a general count-based kernel. The bigram counter, shown in (b), is a specific instance.}
\label{fig:count_based_kernels}
\end{figure*}

We assume we can write the kernel as $k(x_i, x_j) = T\circ T^{-1}(x_i, x_j)$, for instances in our data set $x_i, x_j, 1 \leq i,j \leq l$. This is true for all PDS rational kernels we have seen. We have:
\begin{align}
T\circ T^{-1}(x_i, x_j) &= \sum\limits_{k=1}^{p} T(x_i, z_k) T(x_j, z_k) \nonumber \\
&= \sum\limits_{k=1}^{p} w_k^2 \left|x_i \right|_k \left|x_j\right|_k \label{eqn:count_kernel}
\end{align}
where $|x_i|_k$ denotes the number of occurrences of the sequence $z_k$ in $x_i$ (similarly for $x_j$), and  $w_k$ is the weight associated by $T$ to each occurrence of $z_k$. We assume we are interested in a total of $p$ subsequences to be matched, hence $k \in [1,p]$.

Let $\boldsymbol{X} \in \mathbb{R}^{l \times p}$ denote the matrix defined by $\boldsymbol{X}_{ik} = |x_i|_k$ for $i \in [1, l] \text{ and } k \in [1, p]$, and let $\boldsymbol{X}_k, k \in [1, p]$, denote the $k^{th}$ column of $\boldsymbol{X}$. Using eqn (\ref{eqn:count_kernel}), we can write the kernel matrix as:
\begin{align}
T \circ T^{-1} &= \sum\limits_{k=1}^{p} \mu_k \boldsymbol{X}_k \boldsymbol{X}_k^T \\
&\text{where } \mu_k = w^2_k \nonumber
\end{align}

\subsubsection{Problem formulation}
We use a \textbf{different notation} to denote the labels: $\boldsymbol{Y}\in \mathbb{R}^{l \times l}$ is a diagonal matrix where $\boldsymbol{Y}_{ii} = y_i,$ the label for $x_i$. We denote the column vector of labels with $\boldsymbol{y}$. By $\boldsymbol{C}$, we denote the vector $[C,C, ..., C]^T$, where $C$ is a constant. We will work with a modified form of the SVM dual optimization problem. First, here's the original problem (this is the negated form of eqn (\ref{eqn:dual})):

\begin{align}
&\max_{\boldsymbol{\alpha}} f(\boldsymbol{\alpha}) = 2 \boldsymbol{\alpha}^T \boldsymbol{1} - \sum\limits_{k=1}^{p} \mu_k \boldsymbol{\alpha}^T \boldsymbol{Y}^T \boldsymbol{X}_k \boldsymbol{X}_k^T \boldsymbol{Y} \boldsymbol{\alpha} \label{eqn:kernel_learning_naive_dual}\\
&\text{subject to} \nonumber \\ 
& \;\;\;\;\;\boldsymbol{0} \leq \boldsymbol{\alpha} \leq \boldsymbol{C},
\;\;\;\; \boldsymbol{\alpha}^T \boldsymbol{y} = \boldsymbol{0} \nonumber \\
& \;\;\;\;\;\boldsymbol{\mu} \geq \boldsymbol{0}, \text{ where } \boldsymbol{\mu}^T = [\mu_1, \mu_2, ..., \mu_p] 
\end{align}

Instead of considering this as only a function of $\boldsymbol{\alpha}$, we want to optimize the above wrt the kernel too. The kernel values are decided by $\boldsymbol{\mu}$, hence, we could simply impose an additional minimization wrt $\boldsymbol{\mu}$ in eqn (\ref{eqn:kernel_learning_naive_dual}). Unfortunately, this does not lead to a standard optimization problem.%To this end, we impose an additional constraint to convert the problem into \emph{Quadratic Programming}(QP) problem.  

However, adding a specific constraint helps. For the kind of \emph{thresholding classifier}\footnote{such a classifier decides the label based on whether the decision function is greater than a predetermined value.} we use in SVMs, \cite{Lanckriet:2004:LKM:1005332.1005334} shows the upper-bound of the generalization error is directly proportional to the \emph{trace} of the kernel matrix (see Section 5 in \cite{Lanckriet:2004:LKM:1005332.1005334}). We fix the trace of the kernel matrix in our problem; we would see later that this allows us to present the optimization as a \emph{Quadratic Programming(QP)} problem.

%The trace of the kernel matrix $T \circ T^{-1}$ is given by $\sum\limits_{k=1}^{p}\mu_k$

Our optimization problem now becomes:
\begin{align}
&\min_{\boldsymbol{\mu} \in \mathcal{M}} \max_{\boldsymbol{\alpha} \in \mathcal{A}} f(\boldsymbol{\mu}, \boldsymbol{\alpha}) = 2 \boldsymbol{\alpha}^T \boldsymbol{1} - \sum\limits_{k=1}^{p} \mu_k \boldsymbol{\alpha}^T \boldsymbol{Y}^T \boldsymbol{X}_k \boldsymbol{X}_k^T \boldsymbol{Y} \boldsymbol{\alpha} \label{eqn:kernel_learning_final_opt} \\
&\text{Here,} \nonumber \\ 
%& \;\;\;\;\;\boldsymbol{0} \leq \boldsymbol{\alpha} \leq \boldsymbol{C},
%\;\;\;\; \boldsymbol{\alpha}^T \boldsymbol{y} = \boldsymbol{0} \nonumber \\
%& \;\;\;\;\;\boldsymbol{\mu} \geq \boldsymbol{0}, \;\;\;\; \sum\limits_{k=1}^{p} \mu_k \left|\left| \boldsymbol{X}_k \right|\right|^2 =  \Lambda \nonumber \\
&\;\;\;\;\;\mathcal{M} = \{\boldsymbol{\mu}: \boldsymbol{\mu} \geq \boldsymbol{0}, \sum\limits_{k=1}^{p}\mu_k \left|\left| \boldsymbol{X}_k \right|\right|^2 = \Lambda \}, \text{ where $\Lambda$ is the trace} \nonumber \\
&\;\;\;\;\;\mathcal{A} = \{ \boldsymbol{0} \leq \boldsymbol{\alpha} \leq \boldsymbol{C}, \boldsymbol{\alpha}^T\boldsymbol{y}=0\} \nonumber 
\end{align}

%Here $\sum_{k=1}^{p} \mu_k \left|\left| \boldsymbol{X}_k \right|\right|^2 $ is the trace. 

\subsubsection{Learning the kernel}
We note the following properties:
\begin{enumerate}
\item $\mathcal{M}$ is convex and compact.
\item $\mathcal{A}$ is convex and compact.
\item $f(\boldsymbol{\mu}, \boldsymbol{\alpha})$ is a convex function in $\boldsymbol{\mu}$.
\item $f(\boldsymbol{\mu}, \boldsymbol{\alpha})$ is a concave function in $\boldsymbol{\alpha}$.
\end{enumerate}
See Lemmas  \ref{lemma:m_conv}, \ref{lemma:a_conv}, \ref{lemma:mu_conv}, \ref{lemma:alpha_conv} in the Appendix section for proofs\footnote{the original paper, \cite{4685446}, only mentions these properties; the proofs have been added by us.}.

These properties allow us to apply an extension of \emph{von Neumann's minimax theorem} (proposed and proved in \cite{Fan1953}) to our problem:
\begin{equation} 
\min_{\boldsymbol{\mu} \in \mathcal{M}} \max_{\boldsymbol{\alpha} \in \mathcal{A}} f(\boldsymbol{\mu}, \boldsymbol{\alpha}) = \max_{\boldsymbol{\alpha} \in \mathcal{A}}  \min_{\boldsymbol{\mu} \in \mathcal{M}} f(\boldsymbol{\mu}, \boldsymbol{\alpha})
\end{equation}
Since the term $2 \boldsymbol{\alpha}^T\boldsymbol{1}$ does not depend on $\boldsymbol{\mu}$, this can be further written as:
\begin{equation}
\max_{\boldsymbol{\alpha} \in \mathcal{A}}  \min_{\boldsymbol{\mu} \in \mathcal{M}} f(\boldsymbol{\mu}, \boldsymbol{\alpha}) = \max_{\boldsymbol{\alpha} \in \mathcal{A}} \bigg( 2 \boldsymbol{\alpha}^T\boldsymbol{1} - \max_{\boldsymbol{\mu} \in \mathcal{M}} \sum\limits_{k=1}^{p} \mu_k(\boldsymbol{\alpha}^T \boldsymbol{Y}^T\boldsymbol{X}_k)^2 \bigg)
\end{equation}
Since each term in the summation (last term above) is non-negative, the optimal $\boldsymbol{\mu}$ can be obtained by setting $\mu_k = 0$ for all but the summand with the highest value. Considering the constraint $\sum_{k=1}^{p}\mu_k \left|\left| \boldsymbol{X}_k \right|\right|^2 = \Lambda$, we write the optimization problem as:
\begin{align}
&\max_{\boldsymbol{\alpha} \in \mathcal{A}} \bigg( 2 \boldsymbol{\alpha}^T\boldsymbol{1} - \Lambda \max_{k \in [1,p]} \bigg( \frac{\boldsymbol{\alpha}^T \boldsymbol{Y}^T\boldsymbol{X}_k}{\left|\left| \boldsymbol{X}_k \right|\right|} \bigg)^2\bigg) \nonumber \\
&=\max_{\boldsymbol{\alpha} \in \mathcal{A}} \bigg( 2 \boldsymbol{\alpha}^T\boldsymbol{1} - \Lambda \max_{k \in [1,p]} ( \boldsymbol{\alpha}^T \boldsymbol{u}'_k )^2\bigg) \text{\;\; where $\boldsymbol{u}'_k=\frac{\boldsymbol{Y}^T \boldsymbol{X}_k}{\left|\left| \boldsymbol{X}_k \right|\right|}$} \nonumber \\
&= \max_{\boldsymbol{\alpha} \in \mathcal{A}} \min_{k \in [1,p]}  \bigg( 2 \boldsymbol{\alpha}^T\boldsymbol{1} -\Lambda ( \boldsymbol{\alpha}^T \boldsymbol{u}'_k )^2\bigg) 
\end{align}
Introducing a new variable $t$ we rephrase our optimization problem as follows:
\begin{align}
&\min_{\boldsymbol{\alpha},t} \;\;\;\; -2 \boldsymbol{\alpha}^T\boldsymbol{1} +\Lambda t^2 \\
&\text{subject to}\nonumber \\
& \;\;\;\; \boldsymbol{0} \leq  \boldsymbol{\alpha} \leq \boldsymbol{C}, \boldsymbol{\alpha}^T\boldsymbol{y} = 0 \nonumber \\
& \;\;\;\; -t \leq \boldsymbol{\alpha}^T\boldsymbol{u}'_k \leq t, \forall k \in [1, p] \nonumber   
\end{align}
Since the second constraint wrt $t$ applies to all $k \in [1,p]$, minimizing $t$ across them makes it fit the maximum value of  $\boldsymbol{\alpha}^T\boldsymbol{u}'_k$. It is used here because it helps us pose the problem entirely as a minimization problem.

Let $\boldsymbol{U}' \in \mathbb{R}^{l\times p}$ be the matrix whose $k^{th}$ column is $\boldsymbol{u}'_k$ and introduce the Lagrange  variables $\boldsymbol{\beta}, \boldsymbol{\beta}' \in \mathbb{R}^{p \times 1}, \boldsymbol{\eta}, \boldsymbol{\eta}' \in \mathbb{R}^{l \times 1}$ and $\delta \in \mathbb{R}$ to write the Lagrangian:
\begin{equation}
\mathcal{L}(\boldsymbol{\alpha}, t, \boldsymbol{\beta}, \boldsymbol{\beta}', \boldsymbol{\eta}, \boldsymbol{\eta}', \delta) = -2 \boldsymbol{\alpha}^T \boldsymbol{1} + \Lambda t^2 - \boldsymbol{\eta}^T \boldsymbol{\alpha} + \boldsymbol{\eta}'(\boldsymbol{\alpha}-\boldsymbol{C}) + \delta \boldsymbol{\alpha}^T \boldsymbol{y} - \boldsymbol{\beta}^T(\boldsymbol{U}'^T \boldsymbol{\alpha} + t \boldsymbol{1}) + \boldsymbol{\beta}'^T(\boldsymbol{U}'^T \boldsymbol{\alpha} - t \boldsymbol{1}) \label{eqn:lagrangian}
\end{equation}
Differentiating wrt primal variables $t \text{ and } \boldsymbol{\alpha}$, we have:
\begin{align}
&\nabla_{t} \mathcal{L} =  2t \Lambda - (\boldsymbol{\beta} + \boldsymbol{\beta})^T\boldsymbol{1} = 0 \label{eqn:primal_kernel_learning_t} \\
&\nabla_{\boldsymbol{\alpha}} \mathcal{L} =  -2 \boldsymbol{1} - \boldsymbol{\eta} + \boldsymbol{\eta}' + \delta\boldsymbol{y} + \boldsymbol{U'}(\boldsymbol{\beta}-\boldsymbol{\beta}') = 0 \label{eqn:primal_kernel_learning_alpha} 
\end{align}
Substituting values from eqn (\ref{eqn:primal_kernel_learning_t}) and eqn (\ref{eqn:primal_kernel_learning_alpha}) into the primal eqn (\ref{eqn:lagrangian}), we have the following dual optimization problem:
\begin{align}
&\max_{\boldsymbol{\beta}, \boldsymbol{\beta}', \boldsymbol{\eta}, \boldsymbol{\eta}', \delta} \frac{-1}{4\Lambda} (\boldsymbol{\beta}'+\boldsymbol{\beta})^T(\boldsymbol{1}\boldsymbol{1}^T) (\boldsymbol{\beta}'+\boldsymbol{\beta}) - \boldsymbol{\eta}'^T	\boldsymbol{C} \label{eqn:final_qp_dual} \\
&\text{subject to}\nonumber \\
& \;\;\;\; \boldsymbol{U}'(\boldsymbol{\beta}'-\boldsymbol{\beta}) + (\boldsymbol{\eta}'-\boldsymbol{\eta}) + \delta\boldsymbol{y}-2 \boldsymbol{1} = 0 \nonumber \\
& \;\;\;\; \boldsymbol{\beta}, \boldsymbol{\beta}', \boldsymbol{\eta}, \boldsymbol{\eta}'\geq \boldsymbol{0}, \delta \geq 0 \nonumber
\end{align}
This is a QP problem and standard solvers can be used to obtain a solution. 

Table \ref{tab:kernel_learning_results} shows some results from \cite{4685446} on multiple classification tasks. A task is identified by the name of the dataset used. The last column shows the ratio of error seen with a bigram kernel learned according to eqn (\ref{eqn:final_qp_dual}) to the standard bigram kernel. Results have been averaged over 10 trials and the standard deviation has been reported. We see upto $\sim 15\%$ error reduction.

\begin{table}[!t]
\centering
\begin{tabular}{|l|c|c|}
\hline
Dataset &\# bigrams&Normalized Error\\
\hline \hline
acq&1500&0.9161 $\pm$ 0.0633\\
crude&1200&0.8448 $\pm$ 0.0828\\
earn&900&0.9196 $\pm$ 0.0712\\
grain&1200&0.9707 $\pm$ 0.0294\\
money-fx&1500&0.9682 $\pm$ 0.0396\\
\hline
\end{tabular}
\caption{Performance of kernels that have been learned. All error rates are normalized by the baseline error rate with standard deviation shown over 10 trials.}
\label{tab:kernel_learning_results}
\end{table}

%\section{Conclusion}
\appendix
\section{Graph Kernels*}

%\begin{lemma}
%\label{lemma:r_squared_nds}
%The squared distance $(x, x') \longmapsto \lVert x - x'\rVert^2$ in $\mathbb{R}^N$ defines a NDS kernel. 
%\begin{proof}
%Let $\boldsymbol{c} \in \mathbb{R}^{m \times 1}$ with $\sum_{i=1}^{m} c_i = 0$. Then, for any ${x_1, ... , x_m} \subseteq X $, we can write
%\begin{align*}
%\sum\limits_{i,j=1}^m c_ic_j\lVert\boldsymbol{x}_i - \boldsymbol{x}_j\rVert^2 &= \sum\limits_{i,j=1}^m c_ic_j (\lVert\boldsymbol{x}_i\rVert^2 + \lVert\boldsymbol{x}_j\rVert^2 - 2\boldsymbol{x}_i \cdot \boldsymbol{x}_j) \nonumber \\
%&=\sum\limits_{i,j=1}^m c_ic_j (\lVert\boldsymbol{x}_i\rVert^2 + \lVert\boldsymbol{x}_j\rVert^2) - \sum\limits_{i=1}^m 2c_i\boldsymbol{x}_i \cdot \sum\limits_{j=1}^m 2c_j\boldsymbol{x}_j \nonumber \\
%&=\sum\limits_{i,j=1}^m c_ic_j (\lVert\boldsymbol{x}_i\rVert^2 + \lVert\boldsymbol{x}_j\rVert^2) -2\big\lVert  \sum\limits_{i=1}^m  c_i\boldsymbol{x}_i\big\rVert^2 \nonumber \\
%& \leq \sum\limits_{i,j=1}^m c_ic_j (\lVert\boldsymbol{x}_i\rVert^2 + \lVert\boldsymbol{x}_j\rVert^2) \nonumber \\
%&= \bigg( \sum\limits_{j=1}^m  c_j \bigg) \bigg(\sum\limits_{i=1}^m c_i\lVert\boldsymbol{x}_i\rVert^2\bigg) + \bigg( \sum\limits_{i=1}^m  c_j \bigg) \bigg(\sum\limits_{i=1}^m c_i\lVert\boldsymbol{x}_j\rVert^2\bigg) = 0
%\end{align*}
%\end{proof}
%\end{lemma}
%

%For convenience we reproduce equation \ref{eqn:graph_kernel} here:
%\begin{equation}
%k(G,G') = \sum\limits_{t=0}^\infty \mu(t) q_\times^T W_\times^t p_\times
%\end{equation}

\begin{lemma}
Identities:
\begin{align}
&vec(ABC) = (C^T \otimes A)vec(B) \label{eqn:kronecker_vec_identities_1} \\
&(A \otimes B)(C \otimes D) = AC \otimes BD \label{eqn:kronecker_vec_identities_2}
\end{align}
\begin{proof}
We don't prove these identities here,  see \cite{opac-b1116623} for details.
\end{proof}
\end{lemma}

\begin{lemma}
\label{lemma:pds_graph_kernel_1}
$\forall k \in \mathbb{N}: W_\times^t p_\times = vec[\Phi(X')^t p'(\Phi(X)^T p)^T]$
\begin{proof}
By induction over $t$. Base case: $t = 0$. Using eqn (\ref{eqn:kronecker_vec_identities_1}) we find
\begin{equation*}
W_\times^0 p_\times = p_\times = (p \otimes p')vec(1) = vec(p' 1 p^T) = vec[\Phi(X')^0 p' (\Phi(X)^0 p)^T]
\end{equation*}
We use the inductive assumption for $t: W_\times^t p_\times = vec[\Phi(X')^t p'(\Phi(X)^T p)^T]$. 

For $t+1$:
\begin{align*}
W^{t+1}_\times p_\times = W_\times W_\times^t p_\times &= (\Phi(X) \otimes \Phi(X')) vec[\Phi(X')^t p'(\Phi(X)^t p)^T] \\
&= vec[\Phi(X')\Phi(X')^t p'(\Phi(X)^t p)^T \Phi(X)^T] \text{\emph{\;\;\;\;...(using eqn (\ref{eqn:kronecker_vec_identities_1}))}} \\
&= vec[\Phi(X')^{t+1} p'(\Phi(X)^{t+1} p)^T] 
\end{align*}
\end{proof}
\end{lemma}

\begin{theorem}
\label{theorem:graph_kernel_PDS}
If the coefficients $\mu(t)$ are such that eqn (\ref{eqn:graph_kernel}) converges, then eqn (\ref{eqn:graph_kernel}) defines a valid PDS kernel.
\begin{proof}
Using lemma \ref{lemma:pds_graph_kernel_1} we can write:
\begin{align*}
q^T_\times W_\times^t p_\times &= (q \otimes q')^T vec[\Phi(X')^t p'(\Phi(X)^t p)^T]\\
&= vec[q'^T\Phi(X')^t p'(\Phi(X)^t p)^T q] \text{\emph{\;\;\;\;...(using eqn (\ref{eqn:kronecker_vec_identities_1}))}}\\
&= \underbrace{(q^T\Phi(X)^t p)^T}_{\rho_t(G)^T} \underbrace{(q'^T\Phi(X')^t p')}_{\rho_t(G')}
\end{align*}
Each individual term of the kernel is the product $\rho_t(G)^T \rho_t(G')$ for some function $\rho_t$, and is therefore a valid PDS kernel. The sum i.e. $k(G,G')$ is also PDS given closure properties of PDS kernels.
\end{proof}
\end{theorem}

\begin{lemma}
\label{lemma:m_conv}
$\mathcal{M}$ is convex and compact.
\begin{proof}
$\mathcal{M}$ is a $(p+1)$-dimensional simplex defined by the points $\left|\left| \boldsymbol{X}_k\right|\right|, k=1,2,...,p$. A simplex is a convex set; hence, $\mathcal{M}$ is convex.

Since $\mathcal{M} \subset \mathbb{R}^p$, $\mathcal{M}$ is compact iff it is \emph{closed} and \emph{bounded}, by the \emph{Heine-Borel theorem}. A simplex is both closed and bounded, hence $\mathcal{M}$ is compact.

\end{proof}
\end{lemma}

%
%\begin{lemma}
%\label{lemma:m_is_convex}
%$\mathcal{M}$ is compact.
%\begin{proof}
%Since, $\boldsymbol{\mu} \in \mathbb{R}^p$, compactness is equivalent to being \emph{closed} and \emph{bounded} by the \emph{Heine-Borel theorem}. 
%
%Since $\boldsymbol{\mu}$ is a simplex, by definition it is closed. Also, since the region defined by $\boldsymbol{\mu}$ can be completely contained within a ball with norm, it is bounded.
%\end{proof}
%\end{lemma}

\begin{lemma}
\label{lemma:a_conv}
$\mathcal{A}$ is convex and compact.
\begin{proof}
We prove convexity first. 
Let $\boldsymbol{a}, \boldsymbol{b} \in \mathcal{A}$. We have,
\begin{align*}
\boldsymbol{0} \leq \boldsymbol{a} \leq \boldsymbol{C}, \;\;\boldsymbol{a}^T\boldsymbol{y}=0 \\
\boldsymbol{0} \leq \boldsymbol{b} \leq \boldsymbol{C}, \;\;\boldsymbol{b}^T\boldsymbol{y}=0
\end{align*}
Clearly,
\begin{align*}
\lambda\boldsymbol{a} + (1-\lambda)\boldsymbol{b} &= \lambda \boldsymbol{a}^T\boldsymbol{y} + (1-\lambda) \boldsymbol{b}^T\boldsymbol{y} \\
&= \lambda \cdot  0 + (1-\lambda) \cdot 0 = 0
\end{align*}
Also, $\boldsymbol{0} \leq \lambda\boldsymbol{a} + (1-\lambda)\boldsymbol{b} \leq \boldsymbol{C},\text{ since } \boldsymbol{0} \leq \boldsymbol{a,b} \leq \boldsymbol{C} \text{ and } 0 \leq \lambda \leq 1$. Hence, $\boldsymbol{a}, \boldsymbol{b} \in \mathcal{A} \implies \lambda\boldsymbol{a} + (1-\lambda)\boldsymbol{b} \in \mathcal{A}, \text{ for } 0 \leq \lambda \leq 1$. This proves $\mathcal{A}$ is convex.

Since $\boldsymbol{\alpha} \subset \mathbb{R}^l$, as in Lemma \ref{lemma:m_conv}, to prove compactness we need to only show $\mathcal{A}$ is closed and bounded.
That $\mathcal{A}$ is bounded is obvious, since the points in $\mathcal{A}$ is contained with the hypercube with edge length $C$. To see that $\mathcal{A}$ is closed note that $\boldsymbol{\alpha}^T\boldsymbol{y}=0$ is a hyperplane, which is a closed set.% the limit points of the set are either on the hypercube with edge length $C$, or on the hyperplane defined by $\boldsymbol{\alpha}^T\boldsymbol{y}$, and these points are included within $\mathcal{A}$.
\end{proof}
\end{lemma}

%
%\begin{lemma}
%$\mathcal{A}$ is compact.
%\begin{proof}
%Since $\boldsymbol{\alpha} \in \mathbb{R}^l$, as in Lemma \ref{lemma:m_is_convex}, it suffices to prove $\mathcal{A}$ is \emph{closed} and \emph{bounded}.
%\end{proof}
%\end{lemma}

\begin{lemma}
\label{lemma:mu_conv}
$f(\boldsymbol{\mu}, \boldsymbol{\alpha})$ is a convex function in $\boldsymbol{\mu}$.
\begin{proof}
To prove $f(\boldsymbol{\mu}, \boldsymbol{\alpha})$ is convex in $\boldsymbol{\mu}$, we need to prove:
\begin{equation*}
\forall \boldsymbol{\mu}_1, \boldsymbol{\mu}_2 \in \mathcal{M}, \forall t \in [0,1], \;\;\;\; f(t\boldsymbol{\mu}_1 + (1-t)\boldsymbol{\mu}_2, \boldsymbol{\alpha}) \leq t f(\boldsymbol{\mu}_1, \boldsymbol{\alpha}) + (1-t)f(\boldsymbol{\mu}_2, \boldsymbol{\alpha})
\end{equation*}
Expanding the RHS:
\begin{align*}
f(t\boldsymbol{\mu}_1 + (1-t)\boldsymbol{\mu}_2, \boldsymbol{\alpha}) &= 2 \boldsymbol{\alpha}^T \boldsymbol{1} - \sum\limits_{k=1}^{p} (t \mu_{1k} + (1-t)\mu_{2k}) \boldsymbol{\alpha}^T \boldsymbol{Y}^T \boldsymbol{X}_k \boldsymbol{X}_k^T \boldsymbol{Y} \boldsymbol{\alpha} \\
&=2 \boldsymbol{\alpha}^T \boldsymbol{1} - \sum\limits_{k=1}^{p} t \mu_{1k} \boldsymbol{\alpha}^T \boldsymbol{Y}^T \boldsymbol{X}_k \boldsymbol{X}_k^T \boldsymbol{Y} \boldsymbol{\alpha} - \sum\limits_{k=1}^{p} (1-t)\mu_{2k} \boldsymbol{\alpha}^T \boldsymbol{Y}^T \boldsymbol{X}_k \boldsymbol{X}_k^T \boldsymbol{Y} \boldsymbol{\alpha} \\
&=2 (t+(1-t))\boldsymbol{\alpha}^T \boldsymbol{1} - \sum\limits_{k=1}^{p} t \mu_{1k} \boldsymbol{\alpha}^T \boldsymbol{Y}^T \boldsymbol{X}_k \boldsymbol{X}_k^T \boldsymbol{Y} \boldsymbol{\alpha} - \sum\limits_{k=1}^{p} (1-t)\mu_{2k} \boldsymbol{\alpha}^T \boldsymbol{Y}^T \boldsymbol{X}_k \boldsymbol{X}_k^T \boldsymbol{Y} \boldsymbol{\alpha} \\
&=\bigg(2 t\boldsymbol{\alpha}^T \boldsymbol{1} - \sum\limits_{k=1}^{p} t \mu_{1k} \boldsymbol{\alpha}^T \boldsymbol{Y}^T \boldsymbol{X}_k \boldsymbol{X}_k^T \boldsymbol{Y} \boldsymbol{\alpha}\bigg) + \bigg(2 (t-1) \boldsymbol{\alpha}^T \boldsymbol{1} - \sum\limits_{k=1}^{p} (1-t)\mu_{2k} \boldsymbol{\alpha}^T \boldsymbol{Y}^T \boldsymbol{X}_k \boldsymbol{X}_k^T \boldsymbol{Y} \boldsymbol{\alpha} \bigg)\\
&=t\bigg(2 \boldsymbol{\alpha}^T \boldsymbol{1} - \sum\limits_{k=1}^{p}  \mu_{1k} \boldsymbol{\alpha}^T \boldsymbol{Y}^T \boldsymbol{X}_k \boldsymbol{X}_k^T \boldsymbol{Y} \boldsymbol{\alpha}\bigg) + (t-1)\bigg(2  \boldsymbol{\alpha}^T \boldsymbol{1} - \sum\limits_{k=1}^{p} \mu_{2k} \boldsymbol{\alpha}^T \boldsymbol{Y}^T \boldsymbol{X}_k \boldsymbol{X}_k^T \boldsymbol{Y} \boldsymbol{\alpha} \bigg)\\
&= tf(\boldsymbol{\mu}_1, \boldsymbol{\alpha}) + (1-t)f(\boldsymbol{\mu}_2, \boldsymbol{\alpha})
\end{align*}
Hence proved.
\end{proof}
\end{lemma}

\begin{lemma}
\label{lemma:alpha_conv}	
$f(\boldsymbol{\mu}, \boldsymbol{\alpha})$ is a concave function in $\boldsymbol{\alpha}$.
\begin{proof}
To prove $f(\boldsymbol{\mu}, \boldsymbol{\alpha})$ is concave in $\boldsymbol{\alpha}$, we need to prove $-f(\boldsymbol{\mu}, \boldsymbol{\alpha})$ is convex in $\boldsymbol{\alpha}$, i.e.:
\begin{equation*}
\forall \boldsymbol{\alpha}_1, \boldsymbol{\alpha}_2 \in \mathcal{A}, \forall t \in [0,1], \;\;\;\; -f( \boldsymbol{\mu}, t\boldsymbol{\alpha}_1 + (1-t)\boldsymbol{\alpha}_2) \leq t (-f(\boldsymbol{\mu}, \boldsymbol{\alpha}_1)) + (1-t)(-f(\boldsymbol{\mu}, \boldsymbol{\alpha}_2))
\end{equation*}
Since $\sum_{k=1}^{p} \mu_k \boldsymbol{Y}^T \boldsymbol{X}_k \boldsymbol{X}_k^T$ is positive-definite and symmetric, and $\boldsymbol{Y}$ is invertible (as it has non-zero diagonal entries), we note that $ \sum_{k=1}^{p} \mu_k \boldsymbol{Y}^T \boldsymbol{X}_k \boldsymbol{X}_k^T \boldsymbol{Y}$ is positive-definite and symmetric (property 4 in the Section ``Further Properties'', \cite{wiki:PDS_properties}). We denote this quantity by $\boldsymbol{P}$.

LHS:
\begin{align}
-f( \boldsymbol{\mu}, t\boldsymbol{\alpha}_1 + (1-t)\boldsymbol{\alpha}_2) &=-2( t\boldsymbol{\alpha}_1 + (1-t)\boldsymbol{\alpha}_2)^T \boldsymbol{1} + ( t\boldsymbol{\alpha}_1 + (1-t)\boldsymbol{\alpha}_2)^T \boldsymbol{P} ( t\boldsymbol{\alpha}_1 + (1-t)\boldsymbol{\alpha}_2) \nonumber \\
&=-2t\boldsymbol{\alpha}_1^T \boldsymbol{1} -2(1-t)\boldsymbol{\alpha}_2^T \boldsymbol{1} + ( t\boldsymbol{\alpha}_1 + (1-t)\boldsymbol{\alpha}_2)^T \boldsymbol{P} ( t\boldsymbol{\alpha}_1 + (1-t)\boldsymbol{\alpha}_2) \nonumber \\
&=-2t\boldsymbol{\alpha}_1^T \boldsymbol{1} -2(1-t)\boldsymbol{\alpha}_2^T \boldsymbol{1} + t^2 \boldsymbol{\alpha}_1 ^T \boldsymbol{P} \boldsymbol{\alpha}_1 +(1-t)^2 \boldsymbol{\alpha}_2 ^T \boldsymbol{P} \boldsymbol{\alpha}_2 + t(1-t) (\boldsymbol{\alpha}_1 ^T \boldsymbol{P} \boldsymbol{\alpha}_2 +  \boldsymbol{\alpha}_2 ^T \boldsymbol{P} \boldsymbol{\alpha}_1)\nonumber
\end{align}

RHS:
\begin{align}
t (-f(\boldsymbol{\mu}, \boldsymbol{\alpha}_1)) + (1-t)(-f(\boldsymbol{\mu}, \boldsymbol{\alpha}_2)) &=  -2t\boldsymbol{\alpha}_1^T \boldsymbol{1} +t \boldsymbol{\alpha}_1 ^T \boldsymbol{P} \boldsymbol{\alpha}_1 -2(1-t)\boldsymbol{\alpha}_2^T \boldsymbol{1} + (1-t) \boldsymbol{\alpha}_2 ^T \boldsymbol{P} \boldsymbol{\alpha}_2 \nonumber
\end{align}
The terms deriving from $2\boldsymbol{\alpha}^T\boldsymbol{1}$ cancel out on both sides. We are left with determining the relationship between the remaining terms:
\begin{align*}
t^2 \boldsymbol{\alpha}_1 ^T \boldsymbol{P} \boldsymbol{\alpha}_1 +(1-t)^2 \boldsymbol{\alpha}_2 ^T \boldsymbol{P} \boldsymbol{\alpha}_2 + t(1-t) (\boldsymbol{\alpha}_1 ^T \boldsymbol{P} \boldsymbol{\alpha}_2 +  \boldsymbol{\alpha}_2 ^T \boldsymbol{P} \boldsymbol{\alpha}_1) &\lesseqgtr t \boldsymbol{\alpha}_1 ^T \boldsymbol{P} \boldsymbol{\alpha}_1  + (1-t) \boldsymbol{\alpha}_2 ^T \boldsymbol{P} \boldsymbol{\alpha}_2 \nonumber \\
\implies t(1-t) (\boldsymbol{\alpha}_1 ^T \boldsymbol{P} \boldsymbol{\alpha}_2 +  \boldsymbol{\alpha}_2 ^T \boldsymbol{P} \boldsymbol{\alpha}_1) &\lesseqgtr t(1-t) \boldsymbol{\alpha}_1 ^T \boldsymbol{P} \boldsymbol{\alpha}_1  + t(1-t) \boldsymbol{\alpha}_2 ^T \boldsymbol{P} \boldsymbol{\alpha}_2 \nonumber
\end{align*}
Since $t(1-t) \geq 0$, we cancel it out from both sides without changing the relationship.
\begin{align*}
&\boldsymbol{\alpha}_1 ^T \boldsymbol{P} \boldsymbol{\alpha}_2 +  \boldsymbol{\alpha}_2 ^T \boldsymbol{P} \boldsymbol{\alpha}_1 \lesseqgtr  \boldsymbol{\alpha}_1 ^T \boldsymbol{P} \boldsymbol{\alpha}_1  +  \boldsymbol{\alpha}_2 ^T \boldsymbol{P} \boldsymbol{\alpha}_2 \nonumber \\
\implies &\boldsymbol{\alpha}_1 ^T \boldsymbol{P} (\boldsymbol{\alpha}_2 - \boldsymbol{\alpha}_1) + \boldsymbol{\alpha}_2 ^T \boldsymbol{P} (\boldsymbol{\alpha}_1- \boldsymbol{\alpha}_2) \lesseqgtr 0 \nonumber\\
\implies & (\boldsymbol{\alpha}_2 ^T \boldsymbol{P}-\boldsymbol{\alpha}_1 ^T \boldsymbol{P})(\boldsymbol{\alpha}_1- \boldsymbol{\alpha}_2) \lesseqgtr  0 \\
\implies & (\boldsymbol{\alpha}_2 ^T -\boldsymbol{\alpha}_1 ^T )\boldsymbol{P} (\boldsymbol{\alpha}_1- \boldsymbol{\alpha}_2) \lesseqgtr  0 \\
\implies & -(\boldsymbol{\alpha}_1 -\boldsymbol{\alpha}_2  )^T\boldsymbol{P} (\boldsymbol{\alpha}_1- \boldsymbol{\alpha}_2) \lesseqgtr  0 \\
\implies & -(\boldsymbol{\alpha}_1 -\boldsymbol{\alpha}_2  )^T\boldsymbol{P} (\boldsymbol{\alpha}_1- \boldsymbol{\alpha}_2) \boldsymbol{<}  0 \text{\;\;\;\;...(since $\boldsymbol{P}$ is positive definite)} \\
\end{align*}
Thus, LHS $<$ RHS. Hence proved.
\end{proof}
\end{lemma}

\label{app:graph_kernels}
\bibliographystyle{ieeetr}
\bibliography{ref}

% that's all folks
\end{document}